\documentclass[twoside,11pt]{article}

\usepackage{jmlr2e}

\usepackage{bbm}
\usepackage{dsfont}
\usepackage{algorithm}

\usepackage{mathtools}
\usepackage{multirow}
\usepackage{amsmath} 

\newtheorem{prop}{Proposition}

\usepackage[table]{xcolor}

\usepackage{natbib}

\usepackage{jmlr2e}
\usepackage{xcolor}
\usepackage{bbm}
\usepackage{dsfont}
\usepackage{algorithm}
\usepackage{bm} 
\usepackage{mathtools}
\usepackage{multirow}
\usepackage{amsmath} 

\usepackage{makecell}

\def\m{\mathfrak{m}}
\def\real{\mathbb{R}}

\def\S{\mathcal{S}}
\def\A{\mathcal{A}}
\def\sgn{\mathrm{sgn}}
\def\v{\mathbf{v}}
\def\x{\mathbf{x}}
\def\p{\mathbf{p}}
\def\a{\mathbf{a}}
\def\tf{\tilde{f}}

\newcommand{\ff}[1] {\textcolor{black}{#1}}

\def\m{\mathfrak{m}}
\def\real{\mathbb{R}}

\def\S{\mathcal{S}}
\def\A{\mathcal{A}}
\def\sgn{\mathrm{sgn}}
\def\v{\mathbf{v}}
\def\x{\mathbf{x}}
\def\p{\mathbf{p}}
\def\a{\mathbf{a}}
\def\tf{\tilde{f}}

\ShortHeadings{Quasi-Equivalence of Width and Depth of Neural Networks}{Fenglei Fan, Rongjie Lai, and Ge Wang}
\firstpageno{1}

\begin{document}

\title{Quasi-Equivalence of Width and Depth of Neural Networks}

\author{\name Feng-Lei Fan \email fanf2@rpi.edu \\
       \addr Department of Biomedical Engineering\\
       Rensselaer Polytechnic Institute\\
       Troy, NY 12180, USA
       \AND
       \name Rongjie Lai \email lair@rpi.edu \\
       \addr Department of Mathematics\\
       Rensselaer Polytechnic Institute\\
       Troy, NY 12180, USA
       \AND      
       \name Ge Wang \email wangg6@rpi.edu \\
       \addr Department of Biomedical Engineering\\
       Rensselaer Polytechnic Institute\\
       Troy, NY 12180, USA}

\editor{XXXX}

\maketitle

\begin{abstract}
While classic studies proved that wide networks allow universal approximation, recent research and successes of deep learning demonstrate the power of deep networks. Based on a symmetric consideration, we investigate if the design of artificial neural networks should have a directional preference, and what the mechanism of interaction is between the width and depth of a network. Inspired by the De Morgan law, we address this fundamental question by establishing a quasi-equivalence between the width and depth of ReLU networks in two aspects. First, we formulate two transforms for mapping an arbitrary ReLU network to a wide network and a deep network respectively for either regression or classification so that the essentially same capability of the original network can be implemented. Then, we replace the mainstream artificial neuron type with a quadratic counterpart, and utilize the factorization and continued fraction representations of the same polynomial function to construct a wide network and a deep network, respectively. Based on our findings, a deep network has a wide equivalent, and vice versa, subject to an arbitrarily small error.
\end{abstract}


\section{Introduction}

Over the past years, deep learning \citep{b1, b2} has become the mainstream approach of machine learning and achieved the state-of-the-art performance in many important tasks \citep{b3, kumar2016ask, b5,b6}. One of the key reasons that accounts for the success of deep learning is the increased depth, which allows a hierarchical representation of features. There are a number of papers dedicated to explaining why deep networks are better than shallow ones. Encouraging progress has been made along this direction. The idea to show the superiority of deep networks is basically to find a special family of functions that are very hard to be approximated by a shallow network but easy to be approximated by a deep network, or that a deep network can express complicated functions that a wide network could not \citep{b7, cohen2016expressive, b10, eldan2016power, montufar2014number, b14}. For example, in \cite{eldan2016power}, a special class of radial functions was constructed so that a one-hidden-layer network needs to use an exponential number of neurons to obtain a good approximation, but a two-hidden-layer network only requires a polynomial number of neurons for the same purpose. With the number of linear regions as the complexity measure, \cite{montufar2014number} showed that the number of linear regions grows exponentially with the depth of a network but only polynomially with the width of a network. In \cite{b14}, a topological measure was utilized to characterize the complexity of functions. Then, it was shown that deep networks can represent more complex functions than what the shallow counterparts express. \ff{Besides, width-bounded but depth-unbounded universal approximators were also developed \citep{lu2017expressive,lin2018resnet,fan2018sparse} in analogy to the depth-bounded but width-unbounded universal approximators \citep{b15, b16}}.

\ff{Recently, the effects of width are discussed by more and more studies \citep{cheng2016wide, b18, b19}. 
Since width and depth are the most basic topology measures of a neural network, exploring the roles of width and depth in neural networks is a problem of strong interest and importance. Currently, there exist both width-bounded and depth-bounded universal approximators. Since both width-bounded and depth-bounded networks can represent any function, they can represent each other as well, which suggests the width-depth equivalence of neural networks. Nevertheless, how a neural network learns a mapping is quite different from the way used in proving the universal approximation. Moreover, the core of the width-depth conversion is to employ a network to learn another network instead of any function. Therefore, the width-depth conversion based on universal approximation falls short to capture the relationship between width and depth.} 

\ff{Specifically, we argue that the width-depth conversion via universal approximation is simplistic, inefficient, and lack of insight: 1) (Simplistic) As mentioned earlier, the way used in enabling universal approximation is to divide a target function into many functions over tiny hypercubes. In practice, a network usually does not do so, \textit{e.g.}, a ReLU network divides the space into polytopes. 2) (Inefficient) The published universal approximation analyses do not consider the character of a target function, and just inefficiently divide the function into many pieces over tiny hypercubes. However, in the width-depth transformation, the problem is how to express a wide (or deep) network using a deep (or wide) network. Thus, the properties of networks should be used to make a much more efficient transformation. 3) (Lack of insight) Universal approximation theoretically guarantees the representation capability of neural networks but it does not give any new insight in terms of difference in consumed computational resources and a practical procedure for the transformation from a wide (or deep) network to a deep (or wide) counterpart. As a small summary, utilizing universal approximation theorems to do width-depth transformation do not accurately characterize the width-depth conversion. We still do not fully resolve essential questions on width and depth, \textit{e.g.}, how to non-trivially do transformation between a wide and a deep network, and what the mechanism of interaction is between the width and depth of a network. }

\ff{To fill this gap, inspired by the De Morgan law, here we demonstrate from two perspectives that the width and depth of neural networks are quasi-equivalent. The first perspective leverages that a ReLU network is a piecewise linear function over polytopes, while second perspective utilizes the nested structure of deep networks and the parallel structure of wide networks.} Specifically, in the first perspective, we revisit the De Morgan law:
\begin{equation}
 A_1 \lor A_2 \cdots \lor A_n = \neg\Big( (\neg A_1) \land (\neg A_2) \cdots  \land  (\neg A_n) \Big),   
\end{equation}
where $A_i$ is a propositional rule (\textit{e.g.}, IF $input \in [a_i
, b_i]^m$, THEN $input$ belongs to some class), and such rules are disjoint. A neural network can be linked to a rule-based system such as a collection of propositional rules. Straightforwardly, we can construct either a deep network to realize a union of propositional rules (left side) or a wide network that realizes the complement of the intersection of those rules after complement (right side). As a result, the constructed deep and wide networks are equivalent to each other. Furthermore, we elaborate the quasi-equivalence of general regression and classification networks by constructing two transforms mapping an arbitrary ReLU network to a wide network and a deep network, respectively, thereby verifying a general quasi-equivalence of the width and depth of ReLU networks. Our constructive scheme is largely based on the fact that a ReLU network partitions the space into polytopes~\citep{chu2018exact}. This enables us to have a simplicial complex in the space and then to establish a quasi-equivalence of networks using the essential building blocks, fan-shaped (more generally, hyper-cone-shape) functions, in the form of modularized ReLU networks. 

\begin{table}[htb]
 \centering
\caption{Network structures and complexities through transformation of regression and classification networks. $D$ is the input dimension, and $M$ is the complexity measure of a function class represented by ReLU networks.}
 \scalebox{0.8}{\begin{tabular}{ |c|c|c|c| c|  }
      \hline
      &Network & Width & Depth \\
   \hline
    \multirow{2}{12em}{Transform Regression Networks (\textbf{Theorem \ref{thm:maind=2}})} & Wide & $D(D+1) (2^{D}-1)M $ &
    $D$ \\
    \cline{2-4}
    & Deep  & $(D+1)D^2$ & $(D+1)M $ \\
   \hline
    \multirow{2}{12em}{Transform Classification Networks (\textbf{Theorem \ref{thm:networklaw}})} & Wide  & $D(D+1) (2^{D}-1)M$ & $D$ \\   
    \cline{2-4}
    & Deep  & $(D+1)D^2$ & $(D+1)M $ \\
   \hline
 \end{tabular}}
 \label{tab:summary1}
\end{table}

In the second perspective, we replace the mainstream artificial neuron type with a quadratic counterpart and extend our first perspective by utilizing the factorization and continued fraction representations of the same univariate polynomial to construct wide and deep networks, respectively. Specifically, a univariate polynomial function can be expressed as follows:
\begin{equation}
\begin{aligned}
   &  \prod_j^N (r_j x^2 + s_j x + t_j)   = \sum_{i=0}^{2N} a_{i} x^i = \sum_{k=0}^{N} a_{2k} x^{2k} + \sum_{k=0}^{N-1} a_{2k+1} x^{2k+1}  \\
     = & \cfrac{b_0}{1-\cfrac{b_1 x^2}{1+b_1 x^2- \cfrac{b_2 x^2}{1+b_2 x^2- \cfrac{b_3 x^2}{1-\cdots}}}} +  \cfrac{c_0x}{1-\cfrac{c_1 x^2}{1+c_1 x^2- \cfrac{c_2 x^2}{1+c_2 x^2- \cfrac{c_3 x^2}{1-\cdots}}}}, 
\end{aligned}    
    \label{ExtendedDeMorgan}
\end{equation}
where $a_i\neq 0$, and $r_j, s_j, t_j$ and $b_l, c_l$ are related. 
\ff{Previously, inspired by neuronal diversity, our group designed the quadratic neuron \citep{b2} that replaces the inner product in a conventional neuron with a quadratic function. Due to the merits of the idea, the network based on quadratic neurons have been increasingly studied and applied \citep{bu2021quadratic, ji2021prediction, mantini2021cqnn, xu2022quadralib}}. Here, we can construct a wide quadratic network and a deep quadratic network to implement the left side and right side of Eq. \eqref{ExtendedDeMorgan}, respectively. This establishes the equivalence between wide and deep quadratic networks. Finally, we generalize such an equivalence into a multivariate setting based on Kolmogorov-Arnold theorem \citep{b29}.

Our main contribution is the establishment of the width-depth quasi-equivalence of neural networks. We summarize our main results on ReLU networks (the first perspective) and quadratic networks (the second perspective) in Tables \ref{tab:summary1} and \ref{tab:summary2}, respectively. \ff{Specifically, Table 1 lists the width and depth of the wide and deep networks constructed in our first perspective, where the complexity measure $M$ is the minimum number of simplices needed to cover the polytopes formed by a ReLU network. Table 2 shows the width and depth of the wide and deep quadratic networks constructed in our second perspective, where the complexity measures $K_1$ and $K_2$ are the degrees of polynomials to represent a function of interest. Clearly, given a complexity measure, the width of the constructed wide network is greater than the depth, while the depth of the constructed deep network is greater than the width. }

\begin{table}[htb]
 \centering
\caption{Network structures and complexities through the extension of the De Morgan law. $D$ is the input dimension, and $K_1$ and $K_2$ are the complexity measure of a function class.}
 \scalebox{0.8}{\begin{tabular}{ |c|c|c|c| c|  }
      \hline
      &Network & Width & Depth \\
   \hline
    \multirow{2}{15em}{Extension of the De Morgan Law (\textbf{Theorem \ref{thm:RealDuality}})} & Wide  & $\max\{K_1,K_2\}$&
    $\log(K_1 K_2)$ \\
    \cline{2-4}
    & Deep  & $2(2D+1)$ & $K_1 + K_2$ \\
   \hline
 \end{tabular}}
 \label{tab:summary2}
\end{table}

To put our contributions in perspective, we would like to mention relevant studies. \ff{\cite{jacot2018neural} proposed the theory of the neural tangent kernel (NTK), which provides a useful lens to understand a network when the width of a network goes to infinity.} \cite{b22} analyzed the effect of width and depth on the quality of local minima. They showed that the quality of local minima improves toward the global minima as depth and width become larger. \cite{levine2020limits} revealed the width-depth interplay in a self-attention network. We discuss the width of neural networks as related to NTK in the \textbf{Supplementary Information I.} To the best of our knowledge, our study is the first work to reveal the width-depth quasi-equivalence of neural networks.

\section{Quasi-Equivalence by De Morgan's Law}

\subsection{Preliminaries}

For convenience, we use $\sigma(x) = \max\{0,x\}$ to denote the ReLU function. 
We mainly discuss ReLU networks in this work, thus all networks in the rest of this paper are referred as ReLU networks unless otherwise specified. At the same time, we focus on the fully-connected ReLU networks.

\begin{definition}
A regression network is a network with continuous outputs, while a classification network is a network that produces categorical outputs (for example,  $\{0,1,\cdots,9\}$ for digit recognition). \ff{The classification network is obtained by thresholding a ReLU network in the last layer}. In this study, we investigate a classification network with binary labels without loss of generality.
\end{definition}

\begin{definition}[Width and depth of a feedforward network \citep{arora2016understanding}] 

\ff{For any number of hidden layers $k \in \mathbb{N}$, input and output dimensions  $w_{0}, w_{k+1} \in \mathbb{N}$, a $\mathbb{R}^{w_{0}} \rightarrow \mathbb{R}^{w_{k+1}}$ feedforward network is given by specifying a sequence of $k$ natural numbers  $w_{1}, w_{2}, \ldots, w_{k}$ representing widths of the hidden layers, a set of $k$ affine transformations $T_{i}: \mathbb{R}^{w_{i-1}} \rightarrow \mathbb{R}^{w_{i}}$  for  $i=1, \ldots, k$ and a linear transformation  $T_{k+1}: \mathbb{R}^{w_{k}} \rightarrow   \mathbb{R}^{w_{k+1}}$ corresponding to weights of the hidden layers. The function $f: \mathbb{R}^{n_{1}} \rightarrow \mathbb{R}^{n_{2}}$ computed or represented by this network is
\begin{equation}
f=T_{k+1} \circ \sigma \circ T_{k} \circ \cdots \circ T_{2} \circ \sigma \circ T_{1}, 
\end{equation}
where $\circ$ denotes function composition, and $\sigma$ is an activation function. The depth of a ReLU DNN is defined as $k+1$. The width of a ReLU DNN is $\max \left\{w_{1}, \ldots, w_{k}\right\}$.} 
\end{definition} 

\begin{definition}[Width and depth of a shortcut network] 
\ff{Given a shortcut network $\mathbf{\Pi}$, we delete a minimum number of links such that the resultant network $\mathbf{\Pi}'$ is a feedforward network without any isolated neuron. Then, we define the width and depth of $\mathbf{\Pi}$ as the width and depth of $\mathbf{\Pi}'$, respectively.}
\end{definition}

Over the past several years, increasingly diversified network architectures, such as randomly wired networks \citep{xie2019exploring}, networks with stochastic structures \citep{deng2020understanding}, etc. are used as backbones for deep learning. Our definitions for width and depth are applicable to many unusual network configurations. Moreover, they are also natural extensions of the conventional width and depth definitions and can make sense for common networks.  

\begin{definition}[Simplicial complex]
\label{def:simplex}
A $D$-simplex $S$ is a $D$-dimensional convex hull provided by convex combinations of  $D+1$ affinely independent vectors $\{\v_i\}_{i=0}^D \subset \real^D$. In other words, $\displaystyle S = \left\{ \sum_{i=0}^D \xi_i \v_i ~|~ \xi_i \geq 0, \sum_{i=0}^D \xi_i = 1 \right \}$.  The convex hull of any subset of $\{\v_i\}_{i=0}^D$ is called a {\it face} of $S$. A simplicial complex $\displaystyle \mathcal{S} = \bigcup_\alpha S_\alpha$ is composed of a set of simplices $\{S_\alpha\}$ satisfying: 1) every face of a simplex from $\S$ is also in $\S$; 2) the non-empty intersection of any two simplices $\displaystyle S _{1},S _{2}\in \S$ is a face of both $S_1$ and $S_2$. 
\end{definition}

\begin{prop} Suppose that $f(\mathbf{x})$ is a function represented by a ReLU network, then $f$ is a piecewise linear function that splits the space into polytopes, \ff{where each polytope is convex and associated with a linear function}.
\label{prop:ReLUpolytopes}
\end{prop}  

\begin{proof}
Inspired by the idea in \citep{chu2018exact}, the proof here is for all ReLU networks, including networks using shortcuts. Let a vector $\mathbf{C} = \{c_1,...,c_N\}$ denote the firing states of all neurons in the network, where $N$ is the total number of neurons, and $c_i \in \{0,1\}$. $c_i=0$ means that the $i$-th neuron is not fired and vice versa. The firing state of every neuron is determined by the input $\mathbf{x}$, then we denote the set of instances that share the same collective neuron firing state $\mathbf{C_h}$ as the polytope $P_h$: $P_h=\{\mathbf{x}~|~\mathbf{C}(\mathbf{x})=\mathbf{C_h}\}$.  For $\mathbf{x} \in P_h$, the output of the neuron $i$ is a linear function, denoted as $n^{(i)}(\mathbf{x})$. Then, the firing state of the neuron $i$ is controlled by a linear inequality ($n^{(i)}(\mathbf{x})>0$ or $n^{(i)}(\mathbf{x})\leq 0$). In total, $\mathbf{C}(\mathbf{x})=\mathbf{C_h}$ is equivalent to a set of $N$ linear inequality constraints, indicating that $P_h$ is a convex polytope.
\end{proof}


\begin{definition}
We define the complexity of the function represented by a ReLU network as the minimum number of simplices: $M$ that are needed to cover each and every polytope to support the function of the ReLU network. 
\end{definition}

\begin{figure*}[htb!]
\center{\includegraphics[width=0.5\linewidth] {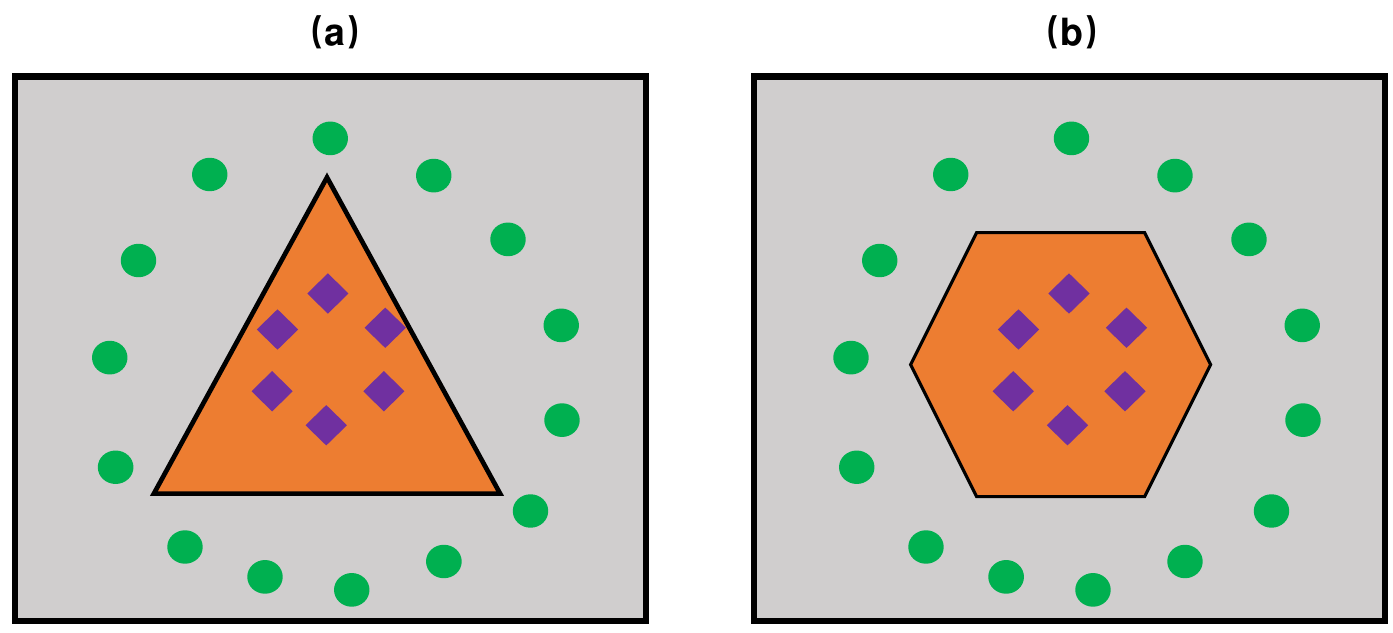}}
\caption{(a) A one-hidden-layer network with three neurons to classify concentric rings. (b) A one-hidden-layer network with six neurons to classify concentric rings. }
\label{Figure_JustifySimplices}
\end{figure*}

\ff{Here, we elaborate why $M$ is a good measure. Previously, because a deep network with piecewise linear activation is a piecewise linear function, the number of linear regions (polytopes) was intensively studied to measure the complexity of a neural network. For example,  \cite{montufar2014number} and \cite{serra2018bounding} estimated the upper and lower bounds of the number of linear regions with respect to the number of neurons at each layer. \cite{xiong2020number} computed the bounds for convolutional neural networks. \cite{park2021unsupervised} proposed neural activation coding to maximize the number of linear regions to improve the model performance. Despite these results, we find that there exists a problem with the number of linear regions as a complexity measure. It may happen that simple and complex networks realize the same number of regions for a given task. As shown in Figure \ref{Figure_JustifySimplices}, two networks divide the space into two linear regions to separate concentric rings. But one network uses three neurons to define a triangle domain, while the other has six neurons to form a hexagon. In this example, according to the number of linear regions, the two networks have the same complexity but the six-neuron network is apparently more complex than the other. To address this problem, the complexity of a linear region should be taken into account as well. We argue that how many simplices a linear region comprises indicates how complex a linear region is. Therefore, we propose to use the number of simplices as a legitimate complexity measure. In Figure \ref{Figure_JustifySimplices}, counting the number of simplices, the complexity of two networks are 2 and 4, respectively, which is a better characterization.}

Let us estimate the lower bound of $M$. To this end, we need to take advantage of the lower bound of the number of polytopes. Empirical bounds of the number of polytopes ($N_p$) in a feedforward ReLU network were estimated in \citep{montufar2014number, serra2018bounding, serra2020empirical}, where one result in \citep{montufar2014number} states that let $n_i,i=1,\cdots,L$, be the number of neurons in the $i$-th layer, and $D$ be the dimension of the input space, $N_p$ is lower bounded by $\Big(\prod_{i=1}^{L-1} [\frac{n_i}{D}]^{D}\Big)\cdot \sum_{j=0}^{D} {n_L \choose j}$. The polytopes constructed therein are hypercubes. Because the minimum number of simplices that fill a $D$-dimensional hypercube is $\frac{2^D \cdot D!}{(D+1)^{(D+1)/2}}$, 
\begin{equation}
    M \geq \frac{2^D \cdot D!}{(D+1)^{(D+1)/2}} \Big(\prod_{i=1}^{L-1} [\frac{n_i}{D}]^{D}\Big)\cdot \sum_{j=0}^{D} {n_L \choose j} .
\label{Mlowerbound}    
\end{equation}

\begin{definition}
We define a wide network and a deep network as follows. Let us assume a function that can be sufficiently complex and yet can be represented by a network. When such a function becomes increasingly complex, the structure of this network must be also increasingly complex, depending on the complexity of the function. We call a network wide if its width is larger than its depth by at least an order of magnitude in $M$, e.g., $\mathcal{O}(M^{\alpha+1})$ vs $\mathcal{O}(M^{\alpha})$, where $M$ is the complexity measure and $\alpha>0$. Similarly, we call a network deep if its depth is larger than its width by at least an order of magnitude in $M$. 
\end{definition} 

It is underscored that we use two different concepts: the complexity of the function class represented by networks and the structural complexity of a network. The former measures the complexity of the function, while the latter measures the topological structure of a network. In our transformation scheme, the structures of constructs/networks are determined by the complexity of the function of interest.

\begin{definition}
We call a wide network $\mathbf{N}_1:\Omega\rightarrow\mathbb{R}$ is equivalent to a deep network $\mathbf{N}_2:\Omega\rightarrow\mathbb{R}$, if $\mathbf{N}_1(\x) = \mathbf{N}_2(\x), \forall x\in\Omega$. We call a wide network $\mathbf{N}_1$ is quasi-equivalent to a deep network $\mathbf{N}_2$, if there is $\delta>0$, $\m(\{\x\in\Omega~|~\mathbf{N}_1(\x)\neq \mathbf{N}_2(\x)\} < \delta$, where $\m$ is a Lebesgue measurement defined on $\Omega$.
\end{definition}

\subsection{Motivating Example}

\begin{figure*}[h]
\center{\includegraphics[width=0.8\linewidth] {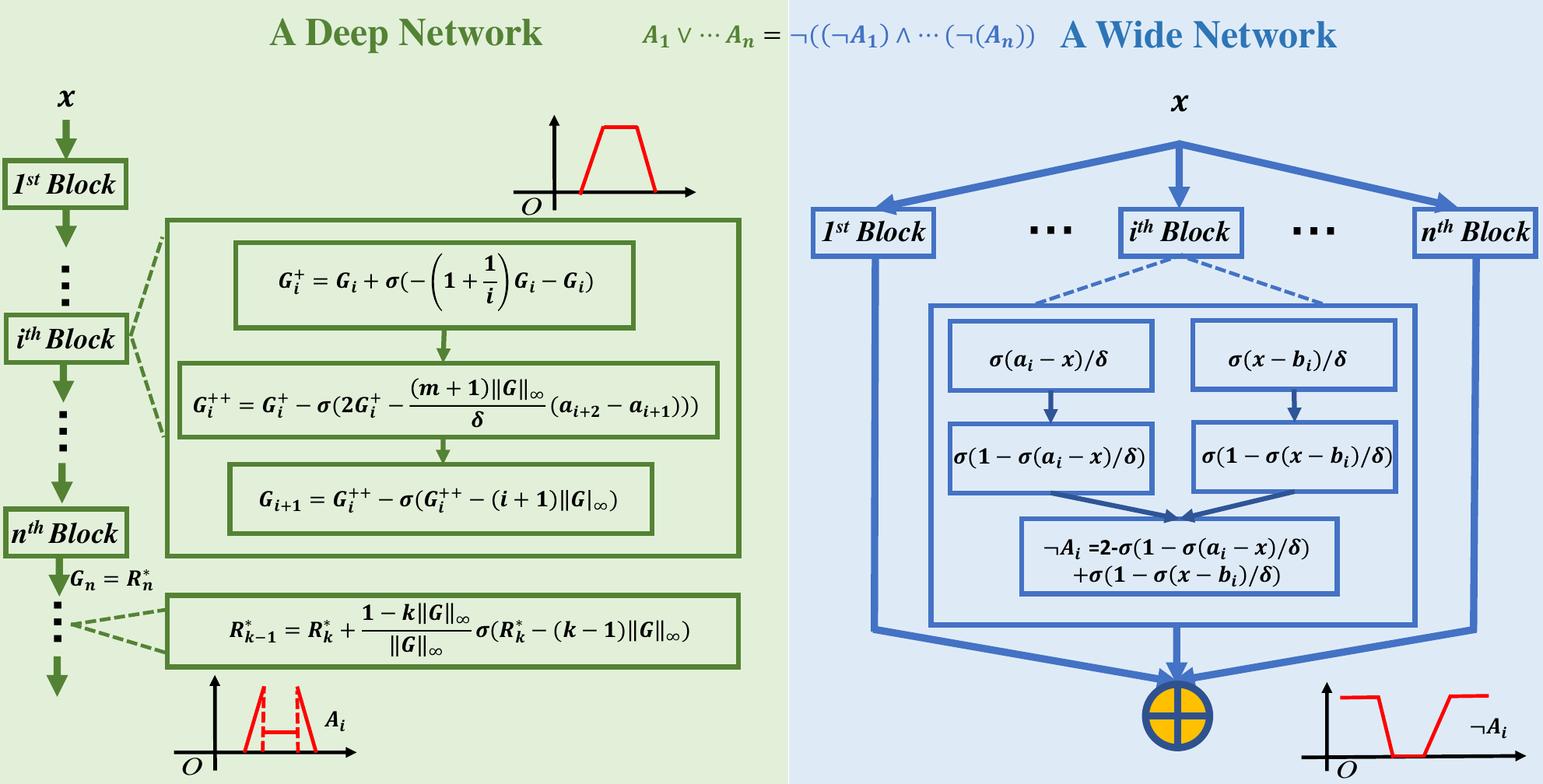}}
\caption{The width and depth equivalence in light of the De Morgan law duality. In this construction, a deep network implements  $A_1 \lor A_2 \cdots \lor A_n$ using a trapezoid function, and a wide network implements $\neg \Big((\neg A_1) \land (\neg A_2) \cdots  \land (\neg A_n) \Big)$ using the trap-like function. $(\cdot)^+$ denotes the ReLU activation.}
\label{DWnetwork}
\end{figure*}

An important school of neural network interpretability research \citep{b30, b31} is to extract interpretable rules from a network \citep{thrun1995extracting, setiono1995understanding, b34} using decompositional or pedagogical methods \citep{thrun1995extracting}. Pedagogical methods decode a set of rules that imitate the input-output relationship of a network, whereas these rules do not necessarily correspond to the parameters of the network. One common type of rules are propositional in the IF-THEN format, where the preconditions are provided as a set of hypercubes with respect to the input:

      IF $input \in [a_i,b_i ]^m$, 
      THEN $input$ belongs to some class.

Since there is a connection between the rule-based inference and the network-based inference, we consider a neural network in terms of propositional rules. 
Furthermore, we know that the De Morgan law holds true for disjoint propositional rules. Mathematically, the De Morgan law is formulated as
\begin{equation}
 A_1 \lor A_2 \cdots \lor A_n = \neg\Big( (\neg A_1) \land (\neg A_2) \cdots  \land  (\neg A_n) \Big),
    \label{eq:DMlaw}
\end{equation}
where $A_i$ is a rule, and $\neg A_i$ is its negation. The De Morgan law gives a duality in the sense of binary logic that the operations $\lor$ and $\land$ are dual, which means that for any propositional rule system described by $A_1 \lor A_2 \cdots \lor A_n$, there exists an equivalent dual propositional rule system $\neg\Big( (\neg A_1) \land (\neg A_2) \cdots  \land  (\neg A_n) \Big)$.

Regarding each rule as an indicator function over a hypercube: 
\[g_i(\x) = \left\{\begin{array}{cc} 1,& \textbf{if}\quad  \x \in \text{a \ hypercube}\\ 0, & \textbf{if} \quad \x \notin \text{a \ hypercube} \end{array}\right.\] 
in Figure \ref{DWnetwork}, we construct a deep network that realizes a logic union of propositional rules (the left hand side of Eq. \eqref{eq:DMlaw}) and a wide network that realizes the negation of the logic intersection of those rules after negation (the right hand side of Eq. \eqref{eq:DMlaw}). As a result, the constructed deep and wide networks are equivalent by the De Morgan law.

\ff{The above motivating example inspires us to consider the width-depth equivalence in a broader domain. First, a ReLU network is a piecewise linear function over polytopes. To generate rules, such a piecewise linear function should be divided into simplices instead of hypercubes. As shown in Figure \ref{Figure_JustifyRules}, we only need two rules if we build rules over simplices, which is much more efficient than building rules over hypercubes. Second, the network can be a regression network rather than a classification network. Thus, representing a linear function rather than an indicator function is demanded. Based on these two considerations, we generalize an indicator function over a hypercube to a linear function over a simplex $S_i$ in a bounded domain:
\[g_i(\x) = \left\{\begin{array}{cc} \mathbf{w}\x + c, & \textbf{if}\quad  \x \in  S_i\\ 0,&\ \textbf{if} \quad \x\in S_i^{c}. \end{array}\right.\]} 

\begin{figure*}[htb!]
\center{\includegraphics[width=0.5\linewidth] {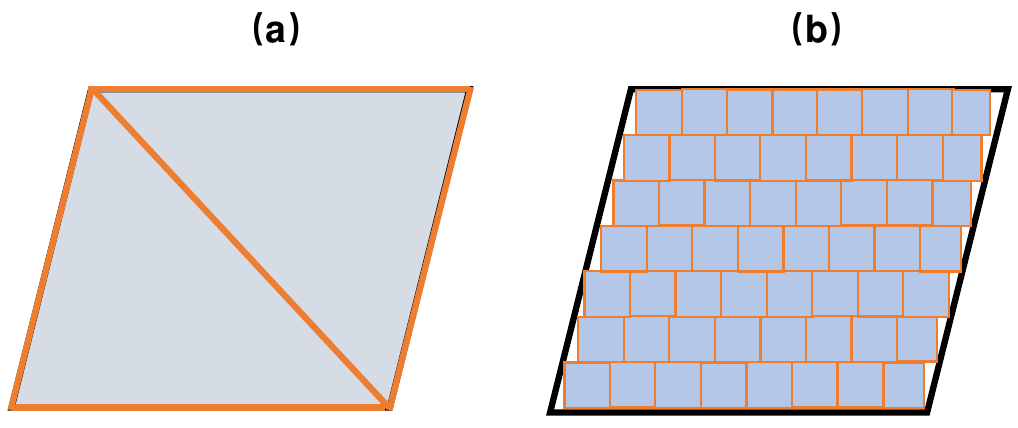}}
\caption{Building rules over simplices is more efficient than over hypercubes for a network.}
\label{Figure_JustifyRules}
\end{figure*}

\subsection{Quasi-Equivalence of Width and Depth of Networks} 

This section describes the first contribution of our paper. We formulate the transformation from an arbitrary ReLU network to a wide network and a deep network, respectively. We use a network-based building block to represent a linear function over a simplex. Integrating such building blocks can represent any piecewise linear function over polytopes, thereby elaborating a general equivalence of the width and depth of networks. Particularly, a regression ReLU network is converted into both a wide and a deep ReLU network (\textbf{Theorem \ref{thm:maind=2}}), while a classification ReLU network is a special case of a regression ReLU network (\textbf{Theorem \ref{thm:networklaw}}). In the regression networks, the transformation of a univariate network is rather different from that of a multivariate network. As a result, the equivalence for the wide and deep networks in the univariate case is precise, whereas the multivariate wide and deep networks are made approximately equivalent up to an arbitrarily small error. What's more, in the multivariate case, the width of the wide network is not the same as the depth of the deep network. This is why we term such an equivalence as a quasi-equivalence.

\subsubsection{Regression Networks}
\label{subsec:regression} 

The sketch of transforming a regression ReLU network is that we first construct either a wide modular network or a deep modular network to represent the corresponding function over each and every simplex, then we aggregate the results into deep or wide networks in series or parallel, respectively, to represent the original network well.

\begin{theorem}[Equivalence of Univariate Regression Networks] Given any ReLU network $f:[-B,~B]\rightarrow\real$, there is a wide ReLU network $\mathbf{H}_1:[-B,~B]\rightarrow\real$ and a deep ReLU network $\mathbf{H}_2:[-B,~B]\rightarrow\real$, such that $f(x) = \mathbf{H}_1(x) = \mathbf{H}_2(x), \forall x\in [-B,~B]$.
\label{thm1}
\end{theorem}
Our main result is formally summarized as the following quasi-equivalence theorem for the multivariate case. 

\begin{theorem} [Quasi-Equivalence of Multivariate Regression Networks] Suppose that the representation of an arbitrary ReLU network is $h: [-B,~B]^D \to \mathbb{R}$, \ff{and $M$ is the minimum number of simplices to cover the polytopes to support $h$}, for any $\delta>0$, there exist a wide ReLU network $\mathbf{H}_1$ of width $\mathcal{O}\left[D(D+1) (2^{D}-1)M \right]$ and depth $D$, and also a deep ReLU network $\mathbf{H}_2$ of width $(D+1)D^2$ and depth $\mathcal{O}\left[(D+1)M\right]$, satisfying that 
\begin{equation}
    \begin{aligned}
    & \m\Big(\x~|~h(\x)\neq \mathbf{H}_1(\x)\}\Big) < \delta \\
    & \m\Big(\x~|~h(\x)\neq \mathbf{H}_2(\x)\}\Big) < \delta, 
    \end{aligned}
\end{equation}
where $\m(\cdot)$ is the standard measure in $[-B,~B]^D$. 
\label{thm:maind=2}
\end{theorem}

We defer the proof of $\textbf{Theorem \ref{thm1}}$ to \textbf{Supplementary Information II}, and split the proof of \textbf{Theorem \ref{thm:maind=2}} into the two-dimensional case (more intuitive) in Appendix and the general case in \textbf{Supplementary Information II} for better readability.

\begin{figure*}[htb!]
\center{\includegraphics[width=0.5\linewidth] {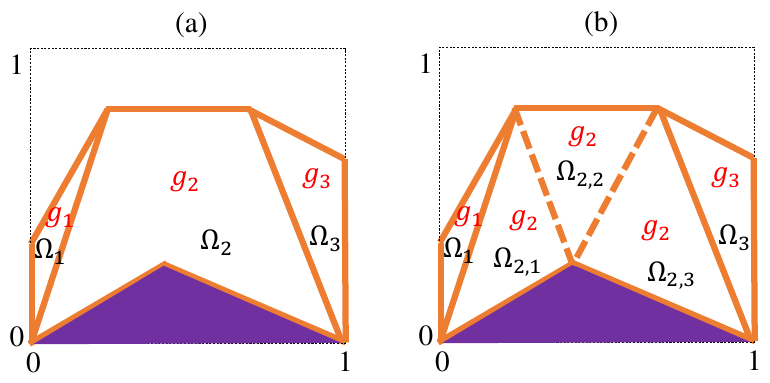}}
\caption{Due to the unboundedness and continuity, the representations in \citep{wang2005generalization,he2018relu} is handicapped in representing a function over polytopes that make a non-convex region. }
\label{Figure_JustifyLocal}
\end{figure*}

The key idea to represent a linear function over a simplex is to construct high-dimensional fan-shaped functions that are supported in fan-shaped domains, and to use these constructs to eliminate non-zero functional values outside the simplex of interest. This is a new and local way to represent a piecewise linear function over polytopes. In contrast, there are two global ways to represent piecewise linear functions \citep{wang2005generalization, he2018relu}. In \citep{wang2005generalization}, for every piecewise linear function $f:\mathbb{R}^n \to \mathbb{R}$, there exists a finite set of linear functions $g_1,\cdots, g_m$ and subsets $T_1,\cdots,T_P \subseteq [m]$ such that $f=\sum_{p=1}^P s_p \underset{i\in T_p}{\max} \{g_i\}$, where $s_p\in \{-1,+1\}, p=1,\cdots,P$. In \citep{he2018relu}, the representation is $f=\underset{1\leq p \leq P}{\max} \underset{i\in T_p}{\min} \{g_i\}$. Nevertheless, \ff{due to the unboundedness and continuity, the global representation of a piecewise linear function is handicapped over polytopes that make a non-convex region. Let us use Figure \ref{Figure_JustifyLocal} to illustrate our point, where the relations of $g_1, g_2, g_3$ are summarized in Table \ref{relation}, and the function value over the purple area is zero. }
\begin{table}[htb]
 \centering
\caption{Regions and relations of functions. }
 \begin{tabular}{ |c|c|   }
      \hline
      Region & Relation  \\
  \hline
     $\Omega_1$  & $g_1\geq g_2\geq g_3$ \\
  \hline
     $\Omega_2$  & $g_2\geq g_1, g_2 \geq g_3$ \\
  \hline
     $\Omega_3$  & $g_3\geq g_2\geq g_1$ \\
     \hline
 \end{tabular}
 \label{relation}
\end{table}

\begin{itemize}
    \item Representation in \cite{wang2005generalization,he2018relu}: $f=\max \{g_1, g_2, g_3\}$
    \item Ours: $f = (g_1)_{\{\x \in \Omega_1\}}+(g_2)_{\{\x \in \Omega_{2,1}\}}+(g_2)_{\{\x \in \Omega_{2,2}\}}+(g_2)_{\{\x \in \Omega_{2,3}\}}+(g_3)_{\{\x \in \Omega_3\}}$.
\end{itemize}
\ff{It can be seen that due to the unboundedness and continuity, $f=\max \{g_1, g_2, g_3\}$ is inaccurate over the purple area. But our representation is accurate over the purple area because it is local.} We highlight the construction of fan-shaped functions, which opens a new door for high-dimensional piecewise function representation. Particularly, the employment of fan-shaped functions will enable a neural network to express a manifold more effectively.


\begin{figure*}[htb!]
\center{\includegraphics[height=1.4in,width=6in,scale=0.4] {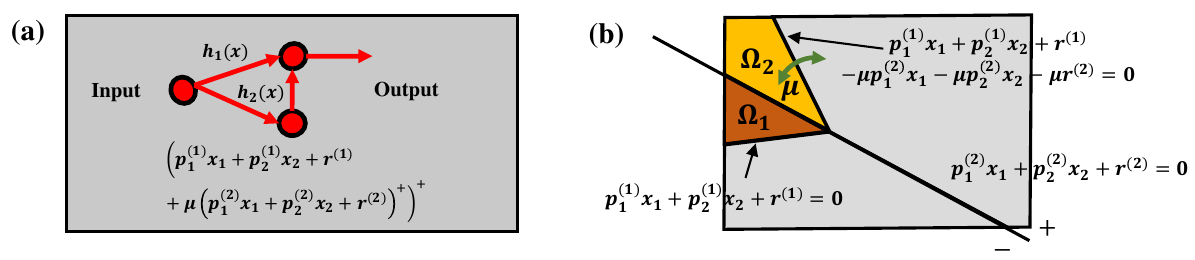}}
\caption{Typical fan-shaped functions constructed by a modularized network to eliminate non-zero functional values outside the simplex of interest.}
\label{fanshape}
\end{figure*}

Since such a fan-shaped function is a basic building block in our construction of wide and deep equivalent networks, let us explain it in a two-dimensional case for easy visualization. An essential building block expressed by a network in Figure \ref{fanshape}(a) is based on the following function:
\begin{equation}
    F(\x) = \sigma\circ(h_1(\x) - \mu \sigma\circ h_2(\x)), 
    \label{eqn:fanshape}
\end{equation}
where $h_1(\x) = p_1^{(1)} x_1 +p_2^{(1)} x_2 +r^{(1)}$, and  $h_2(\x) = p_1^{(2)} x_1 +p_2^{(2)} x_2 +r^{(2)}$ are provided by two linearly independent vectors $\{(p_1^{(1)},p_2^{(1)}), (p_1^{(2)},p_2^{(2)})\}$, and $\mu$ is a positive controlling factor. Eq. \eqref{eqn:fanshape} is a ReLU network of depth=2 and width=2 according to our width-depth definition. As illustrated in Figure \ref{fanshape}(b), the support of $F(\x)$ contains three boundaries and four polytopes (two of which only allow zero value of $F$). For convenience, given a linear function $\ell(\x) = c_1x_1 + c_2x_2 + c_3$, we define $\ell^{-} = \{\x\in\real^2~|~ \ell(\x) <0\}$ and $\ell^{+} = \{\x\in\real^2~|~ \ell(\x) \geq 0\}$. Thus, we can write $\Omega_{1} = h_1^+ \cap h^-_2$ and $\Omega_2 = (h_1 - \mu h_2)^{-}\cap h^+_2$. 
There are three properties of $F(\x)$. First, the common line shared by $\Omega_1$ and $\Omega_2$ is $h_2(\x)=0$.
Second, the size of $\Omega_2$ is adjustable by controlling $\mu$. Note that $h_1(\x) - \mu h_2(\x) = 0$ can move very close to $h_2(\x) = 0$ as $\mu\rightarrow \infty$, which makes $\Omega_2$ negligible. In the limiting case, the support of $F(\x)$ converges to the fan-shaped domain $\Omega_1$. Because $h_1(\x) - \mu h_2(\x) = 0$ is almost parallel to $h_2(\x) = 0 $ when $\mu$ is big enough, we approximate the area of $\Omega_2$  as the product of the length of $h_2(\x)=0$ within $[-B,B]^2$ and the distance between two lines, which yields $|\Omega_2| \leq {2\sqrt{2}B}/\mu$.
Third, the function $F$ over the fan-shaped area $\Omega_1$ is $h_1$. 

\textbf{Remark 1.} 
As a ReLU network of interest partitions the space into more and more polytopes, the number of needed simplices will go increasingly larger. Because the lower bound of $M$ in Eq. \eqref{Mlowerbound} is far larger than $D$ or $(D+1)D^2$, the width of $\mathbf{H}_1(\x)$ and the depth of $\mathbf{H}_2(\x)$ will dominate. Furthermore, the width of $\mathbf{H}_1(\x)$ is higher than its depth by an order of magnitude in terms of $M$, and the depth of $\mathbf{H}_2(\x)$ is higher than its width in a similar way. Therefore, $\mathbf{H}_1(\x)$ is a wide network, and $\mathbf{H}_2(\x)$ is a deep network.

\ff{\textbf{Remark 2.} Compared to universal approximation, our construction, with the use of fan-shaped functions, is valuable in the following aspects: 1) Our construction utilizes the character of ReLU networks. It divides the space into finitely many simplices instead of infinitely many tiny hypercubes; 2) Given a target network, the complexity of our construction does not change with the prescribed error rate. In contrast, the complexity of the construction schemes used in the universal approximation analyses would increase as the preset error decreases. Therefore, our construction is much more efficient; 3) Our construction offers a new and local way to represent a piecewise linear function over polytopes, which is more flexible in representing discontinuous piecewise linear function than the global ways \citep{wang2005generalization, he2018relu}. Furthermore, inspired by the network structure used to construct the proposed fan-shaped function, we find that intra-layer links can enhance the representation capability of a shallow network, closely relevant to the published results on ``depth separation''. Because intra-layer links greatly increase the number of pieces represented by a network, a shallow network with intra-layer links can  express a complicated piecewise linear function as well as a deep network! For more details on this new finding, please see \textbf{Supplementary Information IX} for details.}

\subsubsection{Classification Networks}
\label{subsec:classification}

\ff{A regression network gives a continuous output, while a classification network produces categorical outputs (for example,  $\{0,1,\cdots,9\}$ for digit recognition). The classification network is derived by thresholding the output of the last layer of a ReLU network. In this study, we investigate a classification network with binary labels without loss of generality. We can directly build the equivalence for classification networks in the same way as the regression networks.}

\begin{theorem}[Quasi-Equivalence of Classification Networks]
Without loss of generality for multi-class classification, we assume a binary output. Suppose that the representation of an arbitrary ReLU network is $h: [-B,~B]^D \to \{0,1\}$, \ff{and $M$ is the minimum number of simplices to cover the polytopes to support $h$}, for any $\delta>0$,
there exist a wide ReLU network $\mathbf{H}_1$ of width $\mathcal{O}\left[D(D+1) (2^{D}-1)M \right]$ and depth $D$, and also a deep ReLU network $\mathbf{H}_2$ of width $(D+1)D^2$ and depth $\mathcal{O}\left[(D+1)M\right]$, satisfying that 
\begin{equation}
    \begin{aligned}
    & \m\Big(\x~|~h(\x)\neq \mathbf{H}_1(\x)\}\Big) < \delta \\
    & \m\Big(\x~|~h(\x)\neq \mathbf{H}_2(\x)\}\Big) < \delta, 
    \end{aligned}
\end{equation}
where $\m(\cdot)$ is the standard measure in $[-B,~B]^D$.
\label{thm:networklaw}
\end{theorem}

\begin{proof}
The key is to regard the classification network as a special case of regression network. Then, applying the construction techniques used in the proof of \textbf{Theorem \ref{thm:maind=2}} will lead to that for any $\delta>0$,
\begin{equation}
    \begin{aligned}
    & \m\Big(\x~|~h(\x)\neq \mathbf{H}_1(\x)\}\Big) < \delta \\
    & \m\Big(\x~|~h(\x)\neq \mathbf{H}_2(\x)\}\Big) < \delta. 
    \end{aligned}
\end{equation}
which verifies the correctness of \textbf{Theorem \ref{thm:networklaw}}.
\end{proof}



A classification neural network can be interpreted as a disjoint rule-based system  $A_1 \lor A_2 \cdots \lor A_n$ by splitting the representation of a neural network into many decision polytopes:       IF $(input \in \text{certain polytope})$, THEN ($input$ belongs to some class).  Furthermore, each rule is a local function supported over a decision region.  

\textbf{Remark 3.} The De Morgan equivalence in Figure \ref{DWnetwork} can be summarized as
\begin{equation}
\begin{aligned}
\mathbf{H}_2(\x) & \simeq     A_1 \lor A_2 \cdots \lor A_M  \\
&= \neg\Big( (\neg A_1) \land (\neg A_2) \cdots  \land (\neg A_M) \Big) \simeq \mathbf{H}_1(\x),
\end{aligned}
\end{equation}
when the rules are based on hypercubes. \textbf{Theorem \ref{thm:networklaw}} corresponds to 
\begin{equation}
    \mathbf{H}_2(\x)\simeq     A_1 \lor A_2 \cdots \lor A_M \simeq \mathbf{H}_1(\x),
\end{equation}
when rules are based on simplices. 

\section{Quasi-Equivalence Extended to Networks of ``Quadratic Neurons''}

\subsection{Preliminaries}

\begin{definition}
Quadratic neurons \citep{b2, b25, fan2020universal} integrate the $n$-variable input $\bf x$ with a quadratic function as follows before being nonlinearly processed:
\begin{equation}
\begin{aligned}
h(\x)&=(\sum_{i=1}^{n} w_{ir}x_i +b_r)(\sum_{i=1}^{n} w_{ig}x_i +b_g) + \sum_{i=1}^{n} w_{ib}x_{i}^2+c \\
&=(\textbf{w}_{r}\x^\top+b_{r})(\textbf{w}_{g}\x^\top+b_{g})+\textbf{w}_{b}(\x\odot \x)^\top+c,
\end{aligned}
\end{equation}
where $\bf w_r,\bf w_g, \bf w_b$ are vectors of the same dimensionality as that of $\bf x$, $b_r, b_g, c$ are biases, and $\odot$ is the Hadamard product.
\end{definition}

\ff{The network using quadratic neurons is interesting in several ways: 1) Over the past years the design of neural networks has been focusing on architectures, such as shortcut connections, transformer structure, etc. Almost exclusively, the mainstream deep learning models are constructed with neurons of the same type, which are characterized by the inner product and nonlinear activation. Despite that an artifical neural network was invented via biomimicry, the current artificial networks and a biological neural system are fundamentally different in terms of neuronal diversity and complexity. As we know, a biological neural system coordinates numerous types of neurons to support intellectual behaviors. To fill in this gap, we believe that neuronal diversity should be taken into account in machine learning; 2) Due to the enhanced expressive power at the neuronal level, quadratic neurons have been used in real-world problems and achieved superior performance, such as medical imaging \citep{fan2019quadratic}, civil engineering \citep{ji2021prediction}, applied math \citep{bu2021quadratic}, and so on. Given the utility of quadratic neurons, research on networks with quadratic neurons is attractive.}

Hereafter, networks made of quadratic neurons are referred to as quadratic networks. We refer to the neurons using inner-product as the conventional neurons, and corresponding networks are called conventional networks. Please note that quadratic neurons are distinct from the conventional neurons that use quadratic activation. The decision boundary of the latter is still linear. The key characteristic of quadratic neurons is the employment of the quadratic function replacing the inner-product, and the choice of activation functions is flexible. Hereafter, we refer to networks made of quadratic neurons as quadratic networks.

\begin{prop} Any univariate polynomial of degree $N$ can be perfectly computed by a quadratic ReLU network with the depth of $\log_2(N)$ and the width of $N$ \citep{fan2020universal}. 
\label{prop:AlgebraicStructure}
\end{prop} 

\begin{proof}
Please refer to \citep{fan2020universal} for a detailed proof. Any univariate polynomial of degree $N$ can be factorized as $\prod_{j=1}^{N/2}(r_{j} x^2+s_{j}x+t_{j})$, without the incorporation of complex numbers. Due to the identity: $f(x) = \sigma(f(x)) - \sigma(-f(x))$, every two quadratic neurons $\sigma(r_{j} x^2+s_{j}x+t_{j})$ and $\sigma(-(r_{j} x^2+s_{j}x+t_{j}))$ can be ensembled together to perfectly express $r_{j} x^2+s_{j}x+t_{j}$ in the first layer of a network, followed by the half number of quadratic neurons in the second layer to combine the yields of the first layer, and so on and so forth. Consequently, a quadratic ReLU network with a depth of $\log_2(N)$ and a width of $N$ can express any univariate polynomial of degree $N$. 
\end{proof}

\begin{prop} [Stone-Weierstrass Theorem \citep{de1959stone}] Let $f(x)$ be a continuous real-valued function over $[a,b]$, then for any $\sigma>0$, there exists a polynomial $P(x)$, satisfying
\begin{equation}
    \underset{x\in [a,b]}{\sup}  |f(x)-P(x)| < \sigma.
\end{equation}
\label{StoneWtheorem}
\end{prop}

\begin{prop}[Kolmogorov-Arnold representation theorem \citep{b29}] For any continuous function $f(x_1,\cdots,x_D)$ with $D\geq 2$, there exists a group of continuous functions: 
$\phi_{t,s}$ and $\Phi_t$, where $t = 0,1,\cdots,2D$ and $s =1,2,\cdots,D$, such that
\begin{equation}
    f(x_1,x_2,\cdots,x_D) = \sum_{t = 0}^{2D}\Phi_t\left(\sum_{s=1}^{D}\phi _{t,s}(x_s)\right).
\end{equation}
\label{KATheorem}
\end{prop}

\ff{In the above propositions, \textbf{Proposition} 1 shows a quadratic ReLU network can express a univariate polynomial through factorization, which corresponds to the left side of Eq. \eqref{contiued}. At the same time, this quadratic network is wide. \textbf{Proposition} 2 suggests that a polynomial can represent any continuous function. \textbf{Proposition} 3 informs us that a multi-variate function can be expressed as a combination of univariate functions, which serves as a bridge to generalize our result from approximation to univariate functions to the case of multivariate functions.}

\subsection{Motivating Example}

In contrast to the factorization representation, there is a continued fraction representation for a general univariate polynomial: 
\begin{equation}
\begin{aligned}
    & \sum_{i=0}^{2N} a_{i} x^i = \sum_{k=0}^{N} a_{2k} x^{2k} + \sum_{k=0}^{N} a_{2k+1} x^{2k+1}  \\
  =  &  \cfrac{b_0}{1-\cfrac{b_1 x^2}{1+b_1 x^2-\cdots \cfrac{b_{N-1} + x^2}{1+b_{N-1}x^2- \cfrac{b_N x^2}{1 + b_N x^2} }}} + \cfrac{c_0x}{1-\cfrac{c_1 x^2}{1+c_1 x^2-\cdots \cfrac{c_{N-1} x^2}{1+c_{N-1}x^2- \cfrac{c_N x^2}{1 + c_N x^2} }}}, \\
\end{aligned}     
    \label{contiued}
\end{equation}
where $a_i \neq 0$, $b_0 = a_0$, $b_k = a_{2k}/a_{2k-2}, k\geq 1$ and $c_0 = a_1$, $c_k = a_{2k+1}/a_{2k-1}, k\geq 1$. In the right side of Eq. \eqref{contiued}, the first part contains the terms with even powers of $x$, while the second part with the odd powers of $x$. 
In \textbf{Supplementary Information V}, we prove the correctness of Eq. \eqref{contiued} in detail. 

\begin{figure*}[h]
\center{\includegraphics[width=0.8\linewidth] {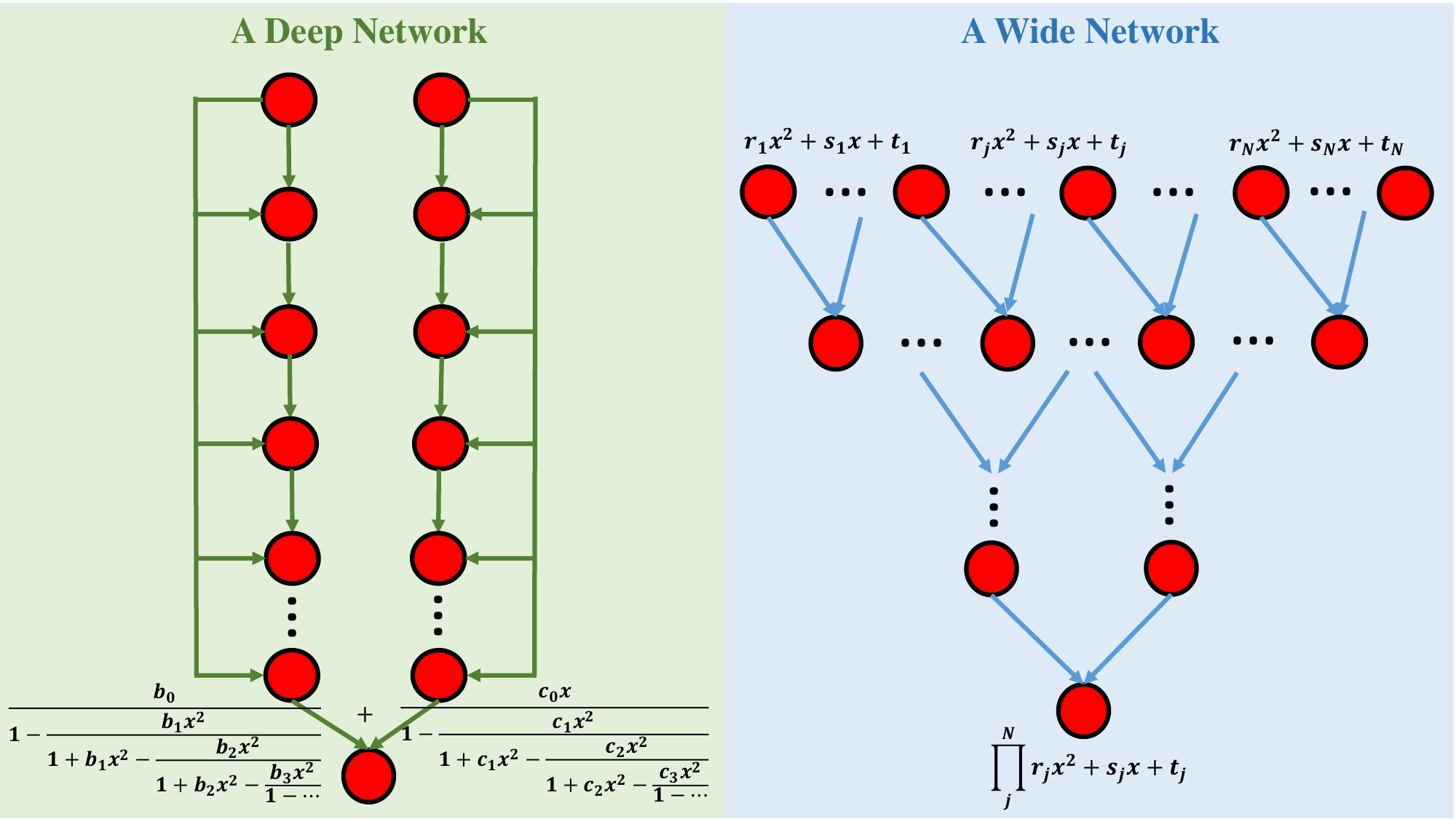}}
\caption{The width and depth equivalence for networks of quadratic neurons. In this construction, a deep network is to implement the continued fraction of a polynomial, and a wide network reflects the factorization of the polynomial.}
\label{DWnetwork_continued}
\end{figure*}

Since both the factorization representation and the continued fraction representation express the same polynomial, we have the following identity:
\begin{equation}
\begin{aligned}
    & \prod_{j}^{N}(r_{j} x^2+s_{j}x+t_{j}) \\
    =& \cfrac{b_0}{1-\cfrac{b_1 x^2}{1+b_1 x^2-\cdots \cfrac{b_{N-1} x^2}{1+b_{N-1}x^2- \cfrac{b_N x^2}{1 + b_N x^2} }}} +  \cfrac{c_0x}{1-\cfrac{c_1 x^2}{1+c_1 x^2-\cdots \cfrac{c_{N-1} x^2}{1+c_{N-1}x^2- \cfrac{c_N x^2}{1 + c_N x^2} }}}, 
\end{aligned}
\label{ExtendedDeMorgan_section4}
\end{equation}
where $r_j, s_j, t_j$ and $b_i, c_i$ are intrinsically connected. Let $\mathbb{Q}_j = r_j x^2 + s_j x + t_j$, $\mathbb{B}_l (x) = \frac{b_l x^2}{1 + b_l x^2- \mathbb{B}_{l+1} (x)}$, where $\mathbb{B}_N (x) = \frac{b_N x^2}{1 + b_N x^2}$ and $\mathbb{B}_0 (x) = \frac{b_0}{1 - \mathbb{B}_1 (x)}$, and $\mathbb{C}_l (x) = \frac{c_l x^2}{1 + c_l x^2- \mathbb{C}_{l+1} (x)}$, $\mathbb{C}_N (x) = \frac{c_N x^2}{1 + c_N x^2}$ and $\mathbb{C}_0 (x) = \frac{c_0 x}{1 - \mathbb{C}_1 (x)}$, we can express Eq. \eqref{ExtendedDeMorgan_section4} concisely as follows:
\begin{equation}
  \mathbb{Q}_1 \times \mathbb{Q}_2 \times \cdots \mathbb{Q}_j \times \cdots \mathbb{Q}_N = \mathbb{B}_{0} \circ \mathbb{B}_{1} \circ \cdots \mathbb{B}_{l} \circ \cdots \mathbb{B}_{N} + \mathbb{C}_{0} \circ \mathbb{C}_{1} \circ \cdots \mathbb{C}_{l} \circ \cdots \mathbb{C}_{N},
\label{condensedDM_section4}
\end{equation}
which could be somehow analogized to the De Morgan law by considering that multiplication and composition operations replace $\lor$ and $\land$ operations, respectively, for either side of Eq. \eqref{condensedDM_section4}. But the objects of those operations are global functions instead of local rules. Eq. \eqref{condensedDM_section4} is inspiring in the sense that the left side is of a parallel computational structure, and the right side is with two nested structures. Clearly, the left and right sides of Eq. \eqref{condensedDM_section4} suggest a wide network and a deep network, respectively. 

In Figure \ref{DWnetwork_continued}, a deep quadratic network is constructed to represent the right side of Eq. \eqref{condensedDM_section4}, while a wide quadratic network is constructed to represent the left side. In light of Eq. \eqref{condensedDM_section4}, these two quadratic networks are functionally equivalent to each other. In the case of the deep network, if the polynomial is of degree $2N$, then the depth of the deep network is $N$. The inter-layer relationship in the two branches are respectively 
$y_{l+1}^{(1)}(x) = \frac{b_{N-l} x^2}{1+b_{N-l} x^2 -y_{l}^{(1)}(x)}$, $y_{N+1}^{(1)} = \frac{b_0}{1-y_{N}^1}$, and  $y_{l+1}^{(2)}(x) = \frac{c_{N-l} x^2}{1+c_{N-l} x^2 - y_{l}^{(2)}(x)}$,
$y_{N}^2 = \frac{c_0 x}{1-y_{N}^2}$, where $y_l^{(1)}$ and $y_l^{(2)}$ are the outputs of the $l^{th}$ layer in either branch. Specially, $y_0^{(1)} = y_0^{(2)} = 0$. In this deep network, the activation function is in a form of reciprocal relation between two inputs: $z(x,y) = \frac{ax}{1+ax-y}$, which can be well approximated by a ReLU function after a proper normalization operation. In \textbf{Supplementary Information V}, we demonstrate that the reciprocal activation works effectively in the deep network after some twists. As for constructing the wide network, we employ Proposition \ref{prop:AlgebraicStructure} directly. Suppose that the polynomial is of degree $2N$, the resultant quadratic network has a width of $2N$ and a depth of only $\log_2(2N)$, where the width dominates. 

Despite the novelty, the above equivalence only fits univariate polynomials. To make a broader impact, we further demonstrate that in terms of representing a multivariate polynomial, the width and the depth of a neural network is also equivalent to each other.

\subsection{Width-Depth Quasi-Equivalence}

This section delineates the second contribution of our paper. Particularly, we first show that a combination of univariate polynomials can approximate any continuous multivariate function. Then, we leverage the constructed wide and deep quadratic networks for univariate polynomials to represent a general continuous multivariate function. As a result, the resultant wide and deep networks offer the same functionality.

\begin{lemma} For any continuous function $f:[0, 1]^D \rightarrow \mathbb{R}$, given $\delta>0$, there exists a function formulated as $\sum_{t=0}^{2D}P_t(\sum_{s=1}^{D}P_{t,s}(x_t))$, where $P_{t,s}$ and $P_t$ are univariate polynomials, such that 
\begin{equation}
    \underset{[x_1,x_2,...,x_D] \in [0,1]^D}{\sup} \left | f(x_1,x_2,...,x_D) - \sum_{t=0}^{2D}P_t(\sum_{s=1}^{D}P_{t,s}(x_t)) \right | < \delta.
\end{equation} 
\label{lemma:polynomials}
\end{lemma}

\begin{proof} 
First, we apply the famous Kolmogorov-Arnold theorem to express $f$:
\begin{equation}
    f(x_1,x_2,\cdots,x_D) = \sum_{t = 0}^{2D}\Phi_t\left(\sum_{s=1}^{D}\phi _{t,s}(x_s)\right),
\end{equation}
where $\phi_{t,s}(x_s)$ and $\Phi_t$ are continuous. 
According to Proposition \ref{StoneWtheorem}, for every function $\phi_{t,s}(x_s)$, given an arbitrarily small quantity $\epsilon_{t,s}>0$ , there exists a polynomial $P_{t,s}(x_s)$ such that
\begin{equation}
    \underset{x_s \in [0,1]}{\sup} \left | \phi_{t,s}(x_s) - P_{t,s}(x_s) \right | < \epsilon_{t,s}.
\end{equation} 
Combining $\phi_{t,1}(x_1), \phi_{t,2}(x_2), ..., \phi_{t,D}(x_D)$ and applying the triangle inequality, we have
\begin{equation}
\begin{aligned}
 &   \underset{[x_1,x_2,...,x_D]^D \in [0,1]^D}{\sup} |\sum_{s=1}^{D}\phi _{t,s}(x_s)-\sum_{s=1}^{D}P_{t,s}(x_s)|  \\
 & \leq \underset{[x_1,x_2,...,x_D]^D \in [0,1]^D}{\sup} \sum_{s=1}^{D}|\phi _{t,s}(x_s)-P_{t,s}(x_s)|  \\
 & < \sum_{s=1}^{D} \epsilon_{t,s}.
\end{aligned}
\end{equation}
Because $\Phi_t$ is a continuous function, we can use the property of continuity. We choose a sufficiently small $\epsilon_{t,s}, s=1,...,D$ such that for every $\Phi_t$, the following inequality holds:
\begin{equation}
        \underset{[x_1,x_2,...,x_D]^D \in [0,1]^D} {\sup}   |\Phi_t(\sum_{s=1}^{D}\phi _{t,s}(x_s))-\Phi_t(\sum_{s=1}^{D}D_{t,s}(x_s))| 
     < \frac{\delta}{4n+2}.
\end{equation}
With respect to the continuous function $\Phi_t$, we can find a polynomial $P_t$ such that
\begin{equation}
\begin{split}
    \underset{x \in \mathbb{R}}{\sup}  \left | \Phi_t(x) - P_t(x) \right | < \frac{\delta}{4n+2}.
\end{split}
\end{equation}
Next, applying the triangle inequality again, we have
\begin{equation}
\begin{aligned}
  & \underset{[x_1,x_2,...,x_D]^D \in [0,1]^D}{\sup} |\Phi_t(\sum_{s=1}^{D}\phi _{t,s}(x_s))-P_t(\sum_{s=1}^{D}P_{t,s}(x_s))| \\
  & \leq \underset{[x_1,x_2,...,x_D]^D \in [0,1]^D}{\sup} |\Phi_t(\sum_{s=1}^{D}\phi _{t,s}(x_s))-\Phi_t(\sum_{s=1}^{D}P_{t,s}(x_s))| \\
  & +  \underset{[x_1,x_2,...,x_D]^D \in [0,1]^D}{\sup} |\Phi_t(\sum_{s=1}^{D}P_{t,s}(x_s))-P_t(\sum_{s=1}^{D}P_{t,s}(x_s))| \\
   & < \frac{\delta}{2n+1}.
\end{aligned}
\end{equation}
Finally, integrating $\Phi_t, t=0,...,2D$, we immediately obtain that 
\begin{equation}
\begin{aligned}
& \underset{[x_1,x_2,...,x_D]^D \in [0,1]^D}{\sup}  |\sum_{t=0}^{2D} \Phi_t(\sum_{s=1}^{D}\phi _{t,s}(x_s))-\sum_{t=0}^{2D}  P_t(\sum_{s=1}^{D}P_{t,s}(x_s))| \\
& \leq \underset{[x_1,x_2,...,x_D]^D \in [0,1]^D}{\sup}  \sum_{t=0}^{2D} |\Phi_t(\sum_{s=1}^{D}\phi _{t,s}(x_s))- P_t(\sum_{s=1}^{D}P_{t,s}(x_s))| \\
& < \delta,
   \end{aligned}
\end{equation}
which concludes the proof.
\end{proof}

\begin{theorem}[Quasi-Equivalence by the Extension of the De Morgan Law]
Given a continuous function $h:[0, 1]^D \rightarrow \mathbb{R}$, for any $\delta>0$, there exists a function expressed as $\sum_{t=0}^{2D}P_t(\sum_{s=1}^{D}P_{t,s}(x_t))$ such that
\begin{equation}
        \underset{\x \in [0,1]^D}{\sup} \left | h(\x) - \sum_{t=0}^{2D}P_t(\sum_{s=1}^{D}P_{t,s}(x_t)) \right | < \delta,
\end{equation}
where $P_t$ and $P_{t,s}$ are polynomials of degrees $\deg(P_t)$ and $\deg(P_{s,t})$. Correspondingly, let $K_1 = \max_{t} [\deg(P_t)]$ and $K_2 = \max_{\{s,t\}}[\deg(P_{s,t})]$, there exist a wide quadratic network $\mathbf{Q}_1$ of a width $\max\{K_1,K_2\}$ and a depth $\log_2 (K_1 K_2)$ and a deep quadratic network $\mathbf{Q}_2$ of a width $2(2D+1)$ and a depth $K_1+K_2$, satisfying 
\begin{equation}
    \begin{aligned}
        & \underset{\x \in [0,1]^D}{\sup} \left | h(\x) - \mathbf{Q}_1 (\x) \right | < \delta \\
        & \underset{\x \in [0,1]^D}{\sup} \left | h(\x) - \mathbf{Q}_2 (\x) \right | < \delta .
    \end{aligned}
\end{equation}
\label{thm:RealDuality}
\end{theorem}

\begin{proof}
To prototype the wide network, we can use the wide quadratic sub-network scheme in Figure \ref{DWnetwork_continued} to express $P_t$ and $P_{t,s}$ whose [width, depth] are $[\deg(P_t), \log_2 (\deg(P_t))]$ and $[\deg(P_{t,s}), \log_2 (\deg(P_{t,s}))]$, respectively. A straightforward combination of these wide sub-networks can express $\sum_{t=0}^{2D}P_t(\sum_{s=1}^{D}P_{t,s}(x_t))$. Thus, the width of the construction is the summation of the widths of all the sub-networks: $\max\{\sum_{t=0}^{2D} \deg(P_t), \sum_{t=0}^{2D} \sum_{s=1}^D \deg(P_{t,s})\} \geq \max \{K_1, K_2\}$, while the depth is $\max_t [\log_2 (\deg(P_t)] + \max_{t,s} [\log_2 (\deg(P_{t,s}))] = \log_2 (K_1) +  \log_2(K_2) = \log_2 (K_1 K_2)$. 

For the deep network, we use the deep quadratic sub-network shown in Figure \ref{DWnetwork_continued} to express $P_t$ and $P_{t,s}$ whose [width, depth] are $[2, \deg(P_t)]$ and $[2,\deg(P_{t,s})]$, respectively.
Similarly, by integrating these deep sub-networks, we can have a deep network that expresses $\sum_{t=0}^{2D}P_t(\sum_{s=1}^{D}P_{t,s}(x_t))$. As a result, the width of the derived network is $2(2D+1)$, and the depth is $\max_t [\deg(P_t)] + \max_{t,s} [\deg(P_{t,s})] = K_1+K_2$.
\end{proof}

\ff{In \textbf{Supplementary Information IV,} we use the quasi-equivalence relationship to construct wide and deep quadratic network variants for the same task, verify their validity on the MNIST dataset, and show empirical hints that a wide network might have an enhanced robustness, since adversarial samples have much less room to perturb latent features and play tricks.}

\section{Discussions}

\ff{\textbf{Gains of Equivalence Notion.} Due to the importance of width and depth, a width-depth equivalence theory should have a major impact on deep learning. Some commonly-used deep learning theories, such as neural tangent kernel (NTK) and neural network Gaussian process (NNGP), were developed assuming an infinite width. Naturally, in the era of deep learning we are more concerned with how depth affects the behavior of a neural network. Therefore, depth-oriented theories would be highly desirable. Our work suggests that depth-oriented NTK and NNGP might be feasible since width and depth are equivalent. 2) Our work suggests a significant potential of wide networks in practical applications. Currently, wide networks are less popular than deep networks due to their sub-optimal performance. However, the reason why wide networks do not work as well might be related to the training strategies instead of the model capacities, since width and depth are equivalent in representation power. If wide networks can be trained to match the performance of deep networks, they suggest an alternative path to develop neural networks.  }

\textbf{Equivalent Networks.} In a broader sense, our quasi-equivalence studies demonstrate the existence of mutually equivalent networks. We argue that the network equivalence is useful in network design. Although deep networks manifest superb power, their applications can be constrained, for example, when the application is time-critical. In that case, we can convert the deep network to a wide counterpart that can be executed at a high speed. A direction for network design optimization is to derive a compact network that maintains a high performance of an original large network through quantization \citep{wu2016quantized}, pruning \citep{li2016pruning}, distillation \citep{polino2018model}, low-rank approximation \citep{zhang2015efficient}, etc. We envision that the equivalence of a deep network and a wide network suggests a new means of network design. Ideally, a wide network can replace the well-trained deep network without compromising the performance. Due to the parallel nature, the wide network can be trained on a computing cluster with many machines for the fast training. At the same time, the inference time of the equivalent wide network is shorter than its deep counterpart.

\textbf{Width-Depth Correlation.} Every continuous $n$-variable function $f$ on $[0,1]^n$ can be in the $L_1$ sense represented by partially separable multivariate functions ~\citep{light2006approximation}: 
\begin{equation}
    \int_{{(x_1,\cdots,x_n)\in [0,1]^n}}  |f(x_1,\cdots,x_n) - \sum_{l=1}^L \prod_{i=1}^n \phi_{li}(x_{i})| < \epsilon,
\end{equation}
where $\epsilon$ is an arbitrarily small positive number, $\phi_{li}$ is a continuous function, and $L$ is the number of products. In the \textbf{Supplementary Information V and VI}, we justify the suitability of this partially separable representation by showing its boundedness, and comparing it with other representations.

Further, we correlate the width and depth of a network to the structure of a function to be approximated. In a nutshell, each continuous function $\phi_{li}$ can be approximated by a polynomial of some degree, which can be appropriately represented by quadratic neurons. As a consequence, via a quadratic representation scheme, the width and depth of a network structure must reflect the complexity of $\sum_{l=1}^L \prod_{i=1}^n \phi_{li}(x_{i})$. In other words, they are controlled by the nature of a specific task. As the task becomes complicated, the width and depth must increase accordingly, and the combination of the width and depth is not unique. For more details, please see the \textbf{Supplementary Information VII}. 

\textbf{Effects of Width on Optimization and Generalization.} In the \textbf{Supplementary Information VIII}, we illustrate the importance of width on optimization in the context of over-paramterization, kernel ridge regression, and NTK, and then report our findings that the existing generalization bounds also shed light on the relationship of the width and depth given a fixed complexity. 

\section{Conclusion}

Inspired by the De Morgan law and through a systematic analysis, we have established the quasi-equivalence between the depth and width of ReLU neural networks from two perspectives. In the first perspective, we have formulated two transforms for mapping an arbitrary regression/classification ReLU network to a wide ReLU network and a deep ReLU network, respectively. In the second perspective, we have extended our quasi-equivalence results from ReLU networks of popular artificial neurons to those of quadratic neurons. This quasi-equivalence represents a step forward in developing a unified deep learning theory. More efforts are needed in the future to refine this quasi-equivalence relationship and find real-world applications.


\acks{F. L. Fan is supported by the Rensselaer-IBM AI Research Collaboration Program (\url{http://airc.rpi.edu}), part of the IBM AI Horizons Network (\url{http://ibm.biz/AIHorizons}), R. Lai is partially supported by an NSF Career Award DMS–1752934 and NSF DMS-2134168. G. Wang is partially supported by an NIH R01 CA237267 grant.}


\newpage

\appendix
\section*{Appendix A. Proof of \textbf{Theorem \ref{thm:maind=2}} (2D) }

Here, we show the correctness of Theorem \ref{thm:maind=2} in the 2D case. Regarding the transformation of an arbitrary multivariate network, the situation is more complicated than in the case of univariate networks. Nevertheless, we are able to establish the $\delta$-equivalence, which is a slightly relaxed result. 

\textbf{The sketch of proof:} A ReLU network is a piecewise linear function over polytopes, which can be decomposed into a summation of linear functions over a simplex. \textbf{Lemma} \ref{lem:wide_deep_modules} shows that a wider network module $\mathbf{N}_1(\x)$ and a deeper network module $\mathbf{N}_2(\x)$ can represent an arbitrary linear function over a simplex. Next, in \textbf{Theorem} \ref{thm:maind=2}, to transform an arbitrary ReLU network into a wide and a deep network, we horizontally aggregate network modules $\mathbf{N}_1(\x)$ to have a wide network, and we use shortcuts to sequentially establish a deep network with $\mathbf{N}_2(\x)$.

A $D$-simplex $S$ is a $D$-dimensional convex hull provided by convex combinations of  $D+1$ affinely independent vectors $\{\v_i\}_{i=0}^D \subset \real^D$. In other words, $\displaystyle S = \left\{ \sum_{i=0}^D \xi_i \v_i ~|~ \xi_i \geq 0, \sum_{i=0}^D \xi_i = 1 \right \}$. In 2D case, if we write $V = (\v_1 - \v_0,\v_2 - \v_0)$, then $V$ is invertible and $S = \left\{\v_0 + V\x ~|~ \x \in \Delta \right\}$, where $\Delta = \left\{\x\in\real^2~|~ \x \geq 0, \mathbf{1}^\top \x \leq 1 \right \}$ is a template simplex in $\real^2$. It is clear that the following one-to-one affine mapping between $S$ and $\Delta$ exists, which is  
\begin{equation}
    T:S\rightarrow \Delta, \p \mapsto T(\p) = V^{-1} (\p - \v_0).
\label{eqn:simplex_transform}    
\end{equation}
Therefore, we only need to make the construction in the special case where $S = \Delta$ to simplify our analysis. The coordinate transform in Eq. \eqref{eqn:simplex_transform} can conveniently map the construction from $\Delta$ to $S$. 

Given a linear function $\ell(\x) = c_1x_1 + c_2x_2 + c_3$, we write $\ell^{-} = \{\x\in\real^2~|~ \ell(\x) <0\}$ and $\ell^{+} = \{\x\in\real^2~|~ \ell(\x) \geq 0\}$. $\Delta$ is enclosed by three lines provided by $\ell_1(\x)=x_1$, 
$\ell_2(\x)= x_2 $, and $\ell_3(\x)=-x_1 -x_2 +1$. We write three vertices of $\Delta$ as $\v_0 = (0,0), \v_1 = (1,0), \v_2 = (0,1)$. Then, $f: [-B,~B]^2\rightarrow\real$ supported on $\Delta$ is expressed as follows:
\begin{equation}
f(\x) = \left\{\begin{array}{cc}  \a^\top \x + b,& \textbf{if}\quad  \x \in \Delta\\
0,& \ \textbf{if} \quad \x\in \Delta^{c}
\end{array}\right.,
\label{eqn:2Df}
\end{equation}
where $\a = (f(\v_1) - f(\v_0),f(\v_2) - f(\v_0)), b = f(\v_0)$.

\begin{lemma}
Suppose that the representation of an arbitrary ReLU network is $f: [-B,~B]^D \to \mathbb{R}$ expressed as Eq. \eqref{eqn:2Df},
for any $\delta>0$, there exist a ReLU network $\mathbf{N}_1$ of width $D(D+1) (2^{D}-1)+2$ and depth $D$, and also a ReLU network $\mathbf{N}_2$ of width $(D+1)D^2$ and depth $D+1$, satisfying that
\begin{equation}
    \begin{aligned}
    & \m\Big(\x~|~f(\x)\neq \mathbf{N}_1(\x)\}\Big) < \delta \\
    & \m\Big(\x~|~f(\x)\neq \mathbf{N}_2(\x)\}\Big) < \delta.
    \end{aligned}
\end{equation}
\label{lem:wide_deep_modules}
\end{lemma}

\begin{proof}(\textbf{D=2})
Our goal is to approximate the given piecewise linear function $f$ over $\Delta$; therefore, we need to cancel $f$ outside its domain. We first index the polytopes separated by three lines $\ell_1(\x) = 0, \ell_2(\x)= 0$, and $\ell_3(\x)=0$ as $\mathcal{A}^{(\chi_1, \chi_2, \chi_3)}=\ell_1^{\chi_1} \cap \ell_2^{\chi_2} \cap \ell_3^{\chi_3}, \chi_i \in \{+,-\}, i=1,2,3$. It is clear that $\Delta = \A^{(+,+,+)}$. In addition, we use $\vee$ to exclude a component. For instance, $\mathcal{A}^{(\chi_1, \vee, \chi_3)} =  \ell_1^{\chi_1}  \cap \ell_3^{\chi_3}$. It can be easily verified that $\mathcal{A}^{(\chi_1, \vee, \chi_3)} = \mathcal{A}^{(\chi_1, +, \chi_3)}\cup \mathcal{A}^{(\chi_1, -, \chi_3)}$. 

\underline{Constructing $\mathbf{N}_1(\x)$:} 
The discontinuity of $f$ in Eq. \eqref{eqn:2Df} is a major challenge of representing the function using a ReLU network. To tackle this issue, we start from a linear function $\tf(\x) = \a^\top \x + b, \forall \x\in\real^2$, which can be represented by two neurons  $\sigma\circ \tf-\sigma\circ (-\tilde{f})$. The key idea is to eliminate $f$ over all polytopes outside $\Delta$. In other words,  $\tf$ over three fan-shaped polytopes  $\A^{(\vee,-,+)},\A^{(-,+,\vee)}$, and $\A^{(+,\vee,-)}$  should be cancelled. 

Let us take the polytope $\A^{(+,\vee,-)}$ as an example. 
Note that $\A^{(+,\vee,-)}$ has two boundaries $\ell_1(\x) = 0$ and $\ell_3(\x) = 0$ as illustrated in Figure \ref{fig:2D_simplex_W}(b). We choose a sufficiently large positive number $\mu$ to construct the three fan-shaped functions:
\begin{equation}
\begin{aligned}
      & F_1^{(+,\vee,-)}(x_1,x_2) = \sigma(x_1 - \mu \sigma(-x_1-x_2+1)) \\
        & F_2^{(+,\vee,-)}(x_1,x_2) = \sigma(x_1 -\eta x_2 - \mu \sigma(-x_1-x_2+1)) \\
                & F_3^{(+,\vee,-)}(x_1,x_2) =\sigma ( x_1 -\eta - \mu \sigma (-x_1-x_2+1)),
\end{aligned}
\end{equation}
where the positive number $\eta$ is chosen to be small enough such that  the lines $ x_1 -\eta x_2 =0$ and $ x_1  -\eta=0$ are very close to $x_1 =0$, then  $\m((x_1)^+\cap (x_1 -\eta x_2)^{-}) < 2\sqrt{2}B\eta$ and $\m((x_1)^+\cap (x_1 -\eta)^{-}) < 2\sqrt{2}B\eta$.  

\begin{figure*}[hbt!]
\centering
\includegraphics[width=0.8\linewidth]{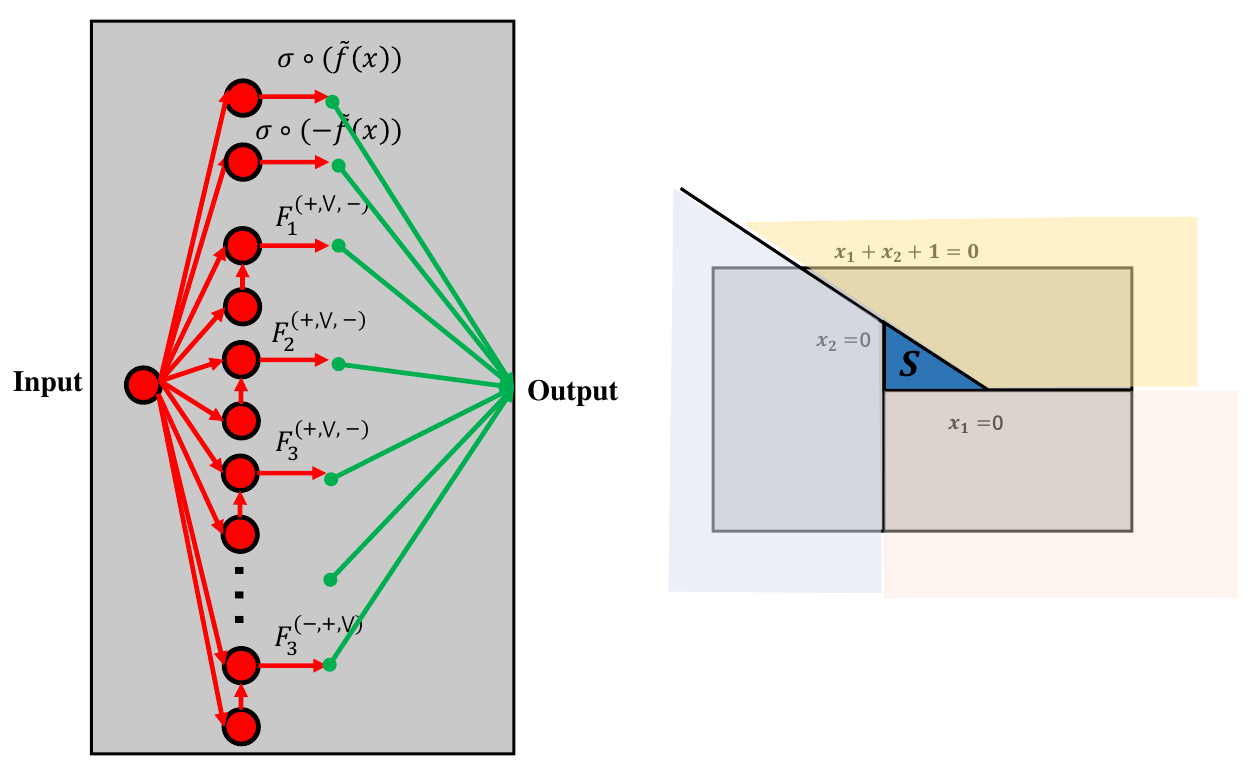}
\caption{Quasi-equivalence analysis in 2D case. Left: The structure of the wide network to represent $f$ over $\Delta$, where two neurons denote $f$ over $[-B,B]^2$ and nine fan-shaped functions handle the polytopes outside $\Delta$. Right: The polytopes outside $\Delta$ comprise of three fan-shaped domains, on which $f$ can be cancelled by three linearly independent fan-shaped functions.}
\label{fig:2D_simplex_W}
\end{figure*}

According to the aforementioned properties of fan-shaped functions, we approximately have 
\begin{equation}
\begin{aligned}
    F_1^{(+,\vee,-)}(\x) &= x_1, \quad\quad\quad\quad   \forall \x\in\A^{(+,\vee,-)}\\
    F_2^{(+,\vee,-)}(\x) &= x_1 - \eta x_2, \quad \forall \x\in\A^{(+,\vee,-)}\backslash \Big((x_1)^+\cap (x_1 -\eta x_2)^{-} \Big)   \\
     F_3^{(+,\vee,-)}(\x) &= x_1 - \eta , \quad\quad \forall \x\in\A^{(+,\vee,-)} \backslash \Big( (x_1)^+\cap (x_1 -\eta)^{-}\Big).
\end{aligned}
\end{equation} 
Let us find $\omega_1^*,\omega_2^*,\omega_3^*$ by solving
\begin{equation}
\begin{bmatrix}
1 & 1 & 1 \\
0 & -\eta & 0  \\
0 & 0 & -\eta
\end{bmatrix}\begin{bmatrix}
\omega_1 \\
\omega_2  \\
\omega_3
\end{bmatrix} = -\begin{bmatrix}
a_1 \\
a_2  \\
b
\end{bmatrix}.
\end{equation}
Based on the property of fan-shaped functions, $F_1^{(+,\vee,-)}(\x), F_2^{(+,\vee,-)}(\x), F_3^{(+,\vee,-)}(\x)$ are approximately constrained into the region $\A^{(+,\vee,-)}$ such that $\tilde{f}(\x)$ is approximately counteracted over $\A^{(+,\vee,-)}$. Mathematically, the new function
$F^{(+,\vee,-)}(\x)=\omega_1^* F_1^{(+,\vee,-)}(\x)+\omega_2^* F_2^{(+,\vee,-)}(\x)+\omega_3^* F_3^{(+,\vee,-)}(\x)$ satisfies
\begin{equation}
\begin{aligned}
    & \m\Big(\{\x\in~\A^{(+,\vee,-)}|~\tf(\x) + F^{(+,\vee,-)}(\x) \neq 0 \} \Big) \\
    & < 2\sqrt{2}B (2\eta + 3/\mu).
\end{aligned}
\end{equation}

Similarly, we can construct $F^{(\vee,-,+)}$ and $F^{(-,+,\vee)}$ to eliminate $\tf$ on $\A^{(\vee,-,+)}$ and $\A^{(-,+,\vee)}$  respectively. 
Finally, these fan-shaped functions are aggregated to form the following ReLU network $\mathbf{N}_1$ (illustrated in Figure \ref{fig:2D_simplex_W}(a)):
\begin{equation}
\begin{aligned}
     \mathbf{N}_1 (\x)=& \sigma\circ (\tf(\x))-\sigma\circ (-\tf(\x)) \\
     & +F^{(\vee,-,+)}(\x) + F^{(+,\vee,-)}(\x)+F^{(-,+,\vee)}(\x),   
\end{aligned}
\end{equation}
where the width and depth of the network are  $2+3\times3\times2=20$ and 2 respectively. In addition, due to the 9 fan-shaped functions being utilized and the effect of the $\eta$, the total area of the regions suffering from errors is no more than 
\begin{equation}
    2\sqrt{2} B (6\eta + 9/\mu).
\end{equation}
Therefore, for any $\delta>0$, as long as we choose $\eta$ and $\mu$ satisfying
\begin{equation}
0 < \eta, 1/\mu < \frac{\delta}{2\sqrt{2}B(6+9)}=\frac{\delta}{30\sqrt{2}B},
\end{equation}
the constructed network $\mathbf{N}$ will have
\begin{equation}
    \m\left(\{\x\in\real^2~|~f(\x)\neq \mathbf{N}_1(\x)\}\right) < \delta.
\end{equation}

\begin{figure*}[hbt!]
\centering
\includegraphics[width=0.8\linewidth]{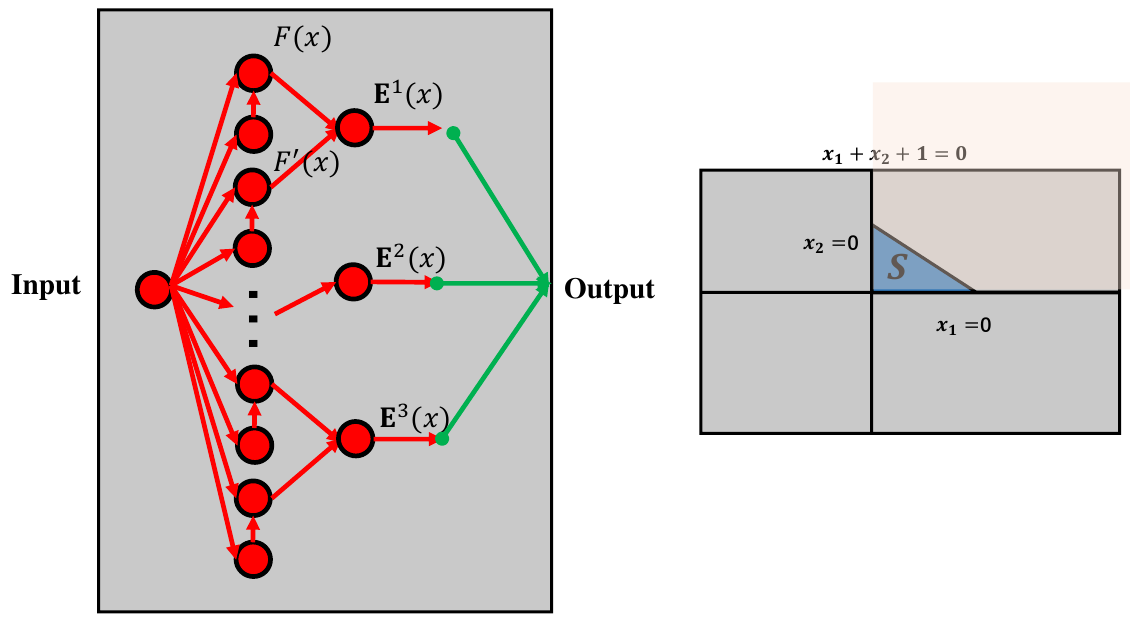}
\caption{Quasi-equivalence analysis in 2D case. Left: The structure of the deep network to represent $f$ over $\Delta$. Right: The polytopes outside $\Delta$ comprise of three fan-shaped domains, on which $f$ can be cancelled by three linearly independent functions over $\Delta$.}
\label{fig:2D_simplex_D}
\end{figure*}

\underline{Constructing $\mathbf{N}_2(\x)$:} Allowing more layers in a network provides an alternate way to represent $f$. Let $F(x_1,x_2) = \sigma \circ (x_1  -\mu \sigma \circ (-x_2))$ and $F^{'}(x_1,x_2) = \sigma \circ ( x_1 -\nu x_2 - \mu \sigma \circ (-x_2))$, both of which are approximately enclosed by boundaries $x_1=0$ and $x_2=0$. Therefore, the fan-shaped regions of $F(x_1,x_2)$ and $F^{'}(x_1,x_2)$ almost overlap as $\nu$ is 
small. The negative sign for $x_2$ is to make sure that the fan-shaped region is $\Delta$.  To obtain the third boundary $\ell_3(\x) =0$ for building the simplex $\Delta$, we stack one more layer with only one neuron to separate the fan-shaped region of $F(x_1,x_2)$ with the boundary $-x_1 -x_2 +1=0$ as follows:
\begin{equation}
    \mathbf{E}^1 (\x) = (\gamma_1^* F(\x) + \gamma_2^* F'(\x)+\gamma_3^*)^+,
\end{equation}
where $(\gamma_1^*, \gamma_2^*, \gamma_3^*)$ are roots of the following system of equations:
\begin{equation}
    \begin{cases}
     \gamma_1 +\gamma_2  = -1& \\
    -\nu \gamma_2  = -1 & \\
     \gamma_3= 1.
    \end{cases}
\end{equation}
Thus, $\mathbf{E}^{1} (x_1,x_2)$ will represent the function $-x_1-x_2 + 1$ over $\Delta$ and zero in the rest area. The depth and width of $\mathbf{E}^1 (x_1,x_2)$ are 3 and 4 respectively. Similarly, due to the employment of the two fan-shaped functions and the effect of $\nu$, the area of the region with errors is estimated as
\begin{equation}
    2\sqrt{2}B\left(\nu+2/\mu \right) .
\end{equation}

To acquire $f$ over $\Delta$, similarly we need three linear independent functions as linear independent bases. We modify $\ell_3$ slightly to get $\ell_3^{'}=\ell_3-\tau' x_1$ and $\ell_3^{''}=\ell_3-\tau'' x_2$. Repeating the procedure described in (1), for $\ell_3^{'}$ we construct the network $\mathbf{E}^{2}(x_1,x_2)$ that is $\ell_3-\tau' x_1$ over $\ell_1^{+}\cap \ell_2^{+} \cap (\ell_3^{'})^{+} $, while for $\ell_3^{''}$ we construct the network $\mathbf{E}^{3}(x_1,x_2)$ that is $\ell_3-\tau' x_1$ over $\ell_1^{+}\cap \ell_2^{+} \cap (\ell_3^{''})^{+} $. We set positive numbers $\tau'$ and $\tau''$  small enough to have two triangular domains $\ell_1^{+}\cap \ell_2^{+} \cap (\ell_3^{'})^{+} $ and $\ell_1^{+}\cap \ell_2^{+} \cap (\ell_3^{''})^{+}$ almost identical with $\Delta$. In addition, let $\tau'$ and $\tau''$ satisfy
\begin{equation}
\begin{bmatrix}
-1 & -1-\tau' & -1\\
-1 & -1 & -1-\tau''  \\
1 & 1& 1
\end{bmatrix} \begin{bmatrix}
\rho_1^*\\
\rho_2^*\\
\rho_3^* 
\end{bmatrix} = \begin{bmatrix}
a_1\\
a_2\\
b
\end{bmatrix},
\end{equation}
where $\rho_1^*,\rho_2^*,\rho_3^*$ are solutions. As a consequence, the deep network (illustrated in Figure \ref{fig:2D_simplex_D} (c)): 
\begin{equation}
  \mathbf{N}_2(\x) = \rho_1^*\mathbf{E}^1(\x)+\rho_2^*\mathbf{E}^{2}(\x)+\rho_3^*\mathbf{E}^{3}(\x)  
\end{equation}
produces $f$ on $\Delta$. The depth and width of the network are 3 and 12. Similarly, the area of the region with errors is bounded above by 
\begin{equation}
    2\sqrt{2}B\left(3\nu+\tau'+\tau''+6/\mu)\right). 
\end{equation}
Therefore, for any $\delta>0$, if we choose
\begin{equation}
   0< \nu, \tau', \tau'', 1/\mu< \frac{\delta}{2\sqrt{2}B(3+2+6)}=\frac{\delta}{22\sqrt{2}B}
\end{equation}
then the constructed network $\mathbf{N}_2$ will satisfy
\begin{equation}
    \m\left(\{\x \in \mathbb{R}^2|f(\x)\neq \mathbf{N}_2(\x)\}\right) < \delta .
\end{equation}
\end{proof}

\begin{proof}
(\textbf{Theorem \ref{thm:maind=2}}, $\mathbf{D=2}$) 
The network $h$ is piecewise linear and splits the space into polytopes. It is feasible to employ a number of simplices to represent the polytopes defined by $h$ \cite{ehrenborg2007perles}. Given that $M$ is the minimum number of required simplices, we have
\begin{equation}
    h(\x) = \sum_{m=1}^M f^{(m)}(\x), 
\end{equation}
where 
\begin{equation}
f^{(m)}(\x) = \left\{\begin{array}{cc}   a_1^{(m)} x_1 + a_2^{(m)} x_2 + b^{(m)},& \textbf{if}\quad  \x \in S^{(m)}\\
0,&  \textbf{if} \quad \x\notin S^{(m)}
\end{array}\right.,
\end{equation}
and $S^{(m)}$ is the $m$-th simplex. For construction of a wide network, we use network modules to represent $f^{(m)}(\x)$ and then horizontally aggregate them into a wide network. In contrast, for construction of a deep network, we sequentially express $f^{(m)}(\x)$ in terms of a network module without linking them to the input. 

\begin{figure*}[hbt!]
\centering
\includegraphics[width=0.95\linewidth]{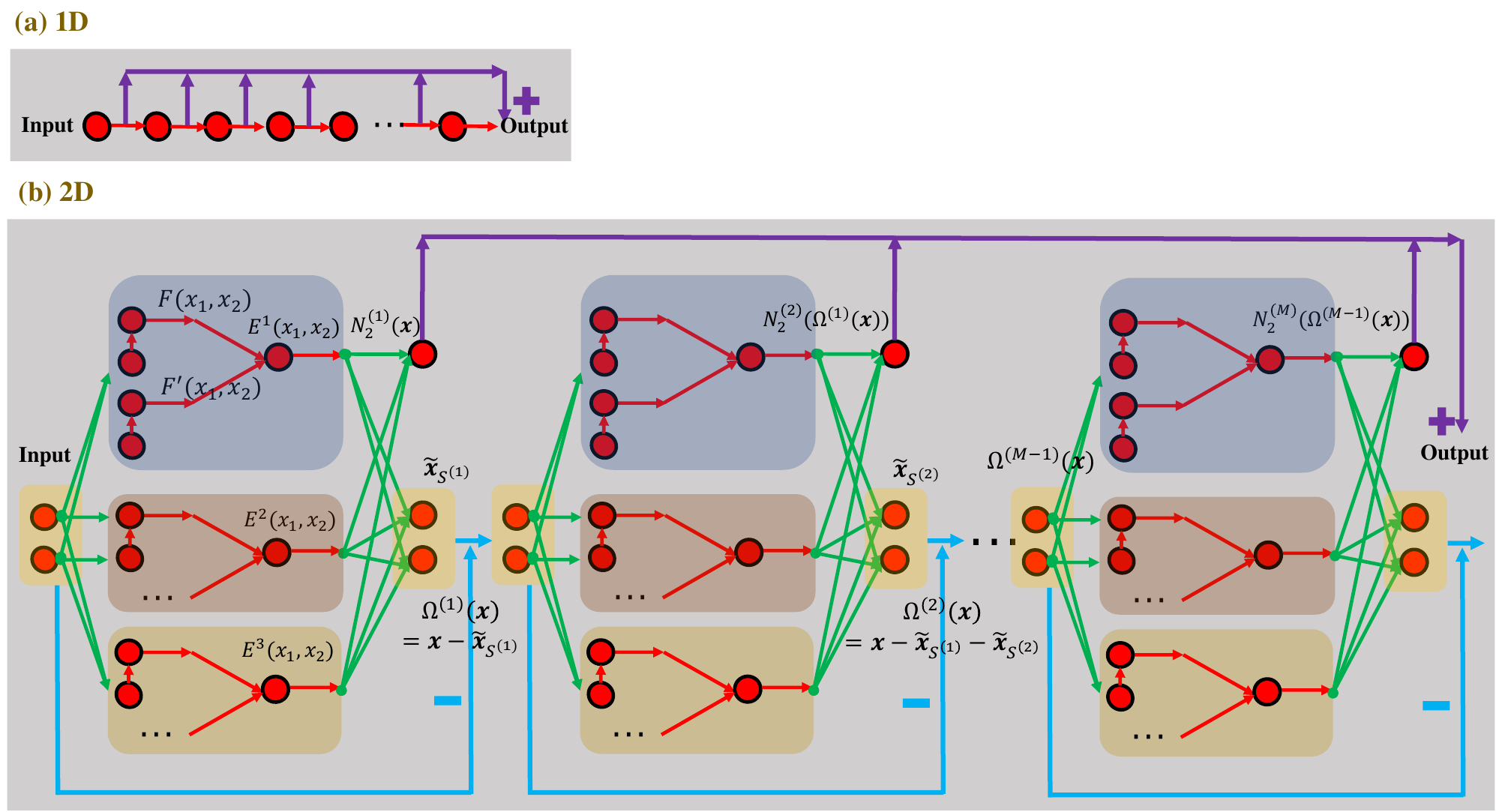}
\caption{Illustration of deep networks in 1D and 2D cases.}
\label{fig:deep_structure_2d}
\end{figure*}

\underline{Representing $f$ with a wide ReLU network}: \textbf{Lemma \ref{lem:wide_deep_modules}} suggests that a wide network module $\mathbf{N}_1$ can generically represent a function over a template simplex. To represent $f^{(m)}$ over $S^{(m)}$, we need to use Eq. \eqref{eqn:simplex_transform} to transform the function from the barycentric coordinate system to the Euclidean coordinate system. Let three vertices of $S^{(m)}$ be $\{\v_0^{(m)},\v_1^{(m)},\v_2^{(m)}\}$ and $V^{(m)} = (\v_1^{(m)} - \v_0^{(m)},\v_2^{(m)} - \v_0^{(m)})$, we have 
\begin{equation}
    \mathbf{N}_1^{(m)}(\x) = \mathbf{N}_1((V^{(m)})^{-1}(\x-\v_0^{(m)})),
\end{equation}
satisfying 
\begin{equation}
    \m\left(\{\x\in\real^2~|~f^{(m)}(\x)\neq \mathbf{N}_1^{(m)}(\x)\}\right) < \delta.
\end{equation}

By aggregating the network $\mathbf{N}_1^{(m)}(\x)$ horizontally, we have the following wide network:
 \begin{equation}
     \mathbf{H}_1(\x) = \sum_{m=1}^M \mathbf{N}_1^{(m)}(\x) .
 \end{equation}
Therefore, the constructed wide network $\mathbf{H}_1(\x)$ is of width $\mathcal{O}(20M)$ and depth 2. It is clear that the width $\mathcal{O}(20M)$ of the wide network $\mathbf{H}_1(\x)$ dominates, as the number of needed simplices goes larger and larger.

\underline{Representing $f$ with a deep ReLU network}: For a deep construction, the fundamental difficulty is how to sequentially express each $f^{(m)}$, \textit{i.e.}, the input of each block can only come from the earlier block instead of the input, like what we did in one-dimensional case (Figure \ref{fig:deep_structure_2d}(a)). Let us derive via induction how to sequentially represent each $f^{(m)}$. We still adopt the idea of modularized networks, but now each network module has two outputs. \textbf{Lemma \ref{lem:wide_deep_modules}} suggests that a deep network module $\mathbf{N}_2$ can generically represent a function over a template simplex. 

\textbf{Step 1.} Assume that the two outputs of the first block are $\mathbf{N}_2^{(1)}$ and $\Omega^{(1)}$. To represent $f^{(1)}$ over $S^{(1)}$, similarly, we need to transform the function from the barycentric coordinate system to the Euclidean coordinate system. Let three vertices of $S^{(1)}$ be $\{\v_0^{(1)},\v_1^{(1)},\v_2^{(1)}\}$ and $V^{(1)} = (\v_1^{(1)} - \v_0^{(1)},\v_2^{(1)} - \v_0^{(1)})$, we have 
\begin{equation}
    \mathbf{N}_2^{(1)} = \mathbf{N}_2((V^{(1)})^{-1}(\x-\v_0^{(1)})),
\end{equation}
which is one output of the first block.

Next, we derive the other output $\Omega^{(1)}(\x)$. Note that we are not allowed to use the input directly. Encouraged by the inversion idea in the univariate case, we invert the function domain into the input domain to get a function that is approximately only supported over $S^{(1)}$:
\begin{equation}
\tilde{\x}_{S^{(1)}} = \left\{\begin{array}{cc}   \x,& \textbf{if}\quad  \x \in S^{(1)}\\
0,& \textbf{if} \quad \x\notin S^{(1)}
\end{array}\right..
\end{equation}
To do so, recall that we have $\mathbf{E}^1, \mathbf{E}^2, \mathbf{E}^3$ in constructing $\mathbf{N}_2$, we will have $(\xi_1^*,\xi_2^*,\xi_3^*)$ 
such that 
\begin{equation}
   \tilde{\x}_{S^{(1)}} = \xi_1^*\mathbf{E}^1((V^{(1)})^{-1}(\x-\v_0^{(1)}))+\xi_2^*\mathbf{E}^{2}((V^{(1)})^{-1}(\x-\v_0^{(1)}))+\xi_3^*\mathbf{E}^{3}((V^{(1)})^{-1}(\x-\v_0^{(1)})) 
\end{equation}

As shown in Figure \ref{fig:deep_structure_2d}(b), use the residual connection, we compute $\Omega^{(1)}(\x)= \x-\tilde{\x}_{S^{(1)}}$, which is zero over $S^{1}$ and $\x$ for other regions. $\Omega^{(1)}(\x)$ will be used to feed the next block to construct $f^{(2)}$.

\textbf{Step 2}. Suppose that the two output functions of the $m$-th block are $\mathbf{N}_2^{m}(\x)$ and $\Omega^{(m)}(\x)=\x-\sum_{i=1}^{m} \tilde{\x}_{S^{(i)}}$. $\Omega^{m}(\x)$ is fed into the $(m+1)$-th block as the input. Because $S^{(m+1)}$ is outside the domain of $S^{(1)}\cup \cdots \cup S^{(m)}$, with the same technique in \textbf{Step 1}, we construct
\begin{equation}
    \mathbf{N}_2^{(m+1)}(\Omega^{(m)}(\x)) = \mathbf{N}_2((V^{(m+1)})^{-1}(\Omega^{(m)}(\x)-\v_0^{(1)})),
\end{equation}
where $\{\v_0^{(m+1)},\v_1^{(m+1)},\v_2^{(m+1)}\}$ are three vertices of $S^{(m+1)}$, and $V^{(m+1)} = (\v_1^{(m+1)} - \v_0^{(m+1)},\v_2^{(m+1)} - \v_0^{(m+1)})$ to express $f^{(m+1)}$ over $S^{(m+1)}$ well. $\Omega^{(m)}(\x)$ is functionally equivalent to $\x$ for two reasons. First, all simplices do not overlap. $\Omega^{(m+1)}(\x)$ is $\x$ outside the domain of $S^{(1)}\cup \cdots \cup S^{(m)}$; therefore zero value over $S^{(1)}\cup \cdots \cup S^{(m)}$ has no effect. Second, w.l.o.g., we assume that $(0,0)$ is outside of all simplices or lies on the boundary of a certain simplex. As such, we have $\mathbf{N}_2^{(m+1)}((0,0))=0, \forall m$, therefore, $\mathbf{N}_2^{(m+1)}(\Omega^{(m)}(\x))$ will not erroneously produce a non-zero constant over $S^{(1)}\cup \cdots \cup S^{(m)}$. Furthermore, we also obtain a function $\tilde{\x}_{S^{(m+1)}}$ inside the $(m+1)$-th block. Apply the residual operation, we have 
\begin{equation}
    \Omega^{(m+1)}(\x)=\Omega^{(m)}(\x) - \tilde{\x}_{S^{(m+1)}} =  \x-\sum_{i=1}^{m+1} \tilde{\x}_{S^{(i)}}.
\end{equation}

Lastly, simlilar to the one-dimensional deep network, we use the shortcut connection to aggregate $\mathbf{N}_2^{(m)}(\Omega^{(m-1)}(\x))$ to obtain the following deep network:
 \begin{equation}
     \mathbf{H}_2(\x) = \sum_{m=1}^M \mathbf{N}_2^{(m)}(\Omega^{(m-1)}(\x)).
 \end{equation}
Therefore, the constructed deep network $\mathbf{H}_2(\x)$ is of depth $\mathcal{O}(3M)$ and width 12. Please note that in the above equation, the summation is made by shortcuts. It is clear that the depth of $\mathbf{H}_2(\x)$ dominates over the width.

\end{proof}

%

\newpage

\addcontentsline{toc}{section}{I. Width as Related to Neural Tangent Kernel (NTK)}
\subsection*{I. Width as Related to Neural Tangent Kernel (NTK)}
NTK sheds light on the power of a network when the width of this network goes to infinity. Previous results have suggested a link between the neural network and the Gaussian distribution when the network width increases infinitely. In \cite{neal1996priors,lee2018deep, matthews2018gaussian, novak2018bayesian, arora2019exact}, the Gaussian process is shown in the two-layer networks, the deep networks, and the convolutional networks. Weakly-trained networks \citep{arora2019exact} refer to the networks where only the top layer is trained after other layers are randomly initialized. Let $f(\bm{\theta}, x) \in \mathbb{R}$ denote the output of the network on input $x$ where $\bm{\theta}$ denotes the parameters of the network, and $\bm{\theta}$ is an initialization over $\bm{\Theta}$. In the above context, training the top layer with the $l_2$ penalty reduces to the kernel regression:
\begin{equation}
    ker^{(F)}(x,x') = \underset{\bm{\theta}\sim \bm{\Theta}}{\mathbb{E}} \langle f(\bm{\theta}, x), f(\bm{\theta}, x') \rangle,
\end{equation}
where $(F)$ denotes the functional space. On the other hand, an infinite wide network is considered as an over-parameterized network, with the motivation to explain the fact that why over-paramterized networks scale. There were extensive studies on this topic \citep{du2018gradient, allen2019learning, cao2019generalization}. Notably, these studies indicate that the training renders relatively small changes when the network is sufficiently wide. Such a concentration behavior is justified by NTK in terms of the Gaussian process. The NTK is defined as the following:
\begin{equation}
    ker^{(\partial)}(x,x') = \underset{\bm{\theta}\sim \bm{\Theta}}{\mathbb{E}} \langle \frac{\partial f(\bm{\theta}, x)}{\partial \bm{\theta}},\frac{\partial f(\bm{\theta}, x')}{\partial \bm{\theta}} \rangle
\end{equation}
where $(\partial)$ denotes the gradient space. Along this line, \cite{huang2020deep} compared the NTK of the ResNet and the deep chain-like network, and found that the NTK of the chain-like network converges to a non-informative kernel as the depth goes to infinity, while the counterpart of ResNet does not. \cite{du2019width} showed that the width provably matters in networks using the piecewise linear activation.

\newpage

\addcontentsline{toc}{section}{II. Transformation of an Arbitrary Network}
\subsection*{II. Transformation of an Arbitrary Network}

Here, we provide the proof for \textbf{Theorem \ref{thm1_SI}} and the experiment that shows the feasibility of the procedures in the proof of \textbf{Theorem \ref{thm:maind=N}}. The similar derivation to the proof for \textbf{Theorem \ref{thm1_SI}} is also seen in \cite{fan2018sparse}.

\begin{theorem}[Equivalence of Univariate Regression Networks] Given any ReLU network $f:[-B,~B]\rightarrow\real$ with one dimensional input and output variables. There is a wide ReLU network $\mathbf{H}_1:[-B,~B]\rightarrow\real$ and a deep ReLU network $\mathbf{H}_2:[-B,~B]\rightarrow\real$, such that $f(x) = \mathbf{H}_1(x) = \mathbf{H}_2(x), \forall x\in [-B,~B]$.
\label{thm1_SI}
\end{theorem}

\begin{proof} (\textbf{Theorem \ref{thm1_SI}}) Since the function represented by a network $f: [-B,~B] \to \mathbb{R}$ is piecewise linear, we express $f(x)$ as
\begin{equation}
    f(x) =
    \begin{cases}
    w^{(0)}(x-x_0)+f(x_0), & x\in[x_0,x_1)\\
    w^{(1)}(x-x_1)+f(x_1), & x\in[x_1,x_2)\\
    w^{(2)}(x-x_2)+f(x_2), & x\in[x_2,x_3)\\
    \ \ \ \ \ \ \ \ \ \ \ \ \vdots \\
    w^{(n)}(x-x_n)+f(x_n), & x\in[x_{n},x_{n+1}],
    \end{cases}
\end{equation}
where $-B= x_0<x_1<x_2<\cdots<x_n<x_{n+1}=B$, $\displaystyle w^{(i)}=\frac{f(x_{i+1})-f(x_i)}{x_{i+1}-x_i}$ to guarantee continuity, and the slopes of neighboring pieces are different; otherwise they can be fused together. 

We first construct a wide ReLU network $\mathbf{H}_1(x)$ to represent $f$. This can be straightforward as follows:
\begin{equation}
    \displaystyle \mathbf{H}_1(x) = f(x_0) + \sum_{i=0}^n (w^{(i)}-w^{(i-1)})\sigma(x-x_i),
    \label{eqn:1Dwide}
\end{equation}
where specially $w^{(-1)}=0$. This network is of width $n+1$.
Now, we verify this claim. Given $x \in [x_j,x_{j+1})$, for some $j\geq 0$, we have the following:
\begin{equation}
\begin{aligned}
       & f(x_0)+ \sum_{i=0}^n  (w^{(i)}-w^{(i-1)})\sigma(x-x_i) \\      
       =&f(x_0)+ \sum_{i=0}^j  (w^{(i)}-w^{(i-1)})(x-x_i) \\
       =& f(x_0)+ w^{(j)} x + \sum_{i=0}^{j-1}  w^{(i)}(x_{i+1}-x_i) - w^{(j)}x_j\\
      =& f(x_0)  + \sum_{i=0}^{j-1} (f(x_{i+1})- f(x_i)) + w^{(j)}(x - x_j) \\ 
       =& f(x_j) + w^{(j)}(x-x_j).
\end{aligned}
\end{equation}

\begin{figure}[h]
\centering
\includegraphics[width=.4\linewidth]{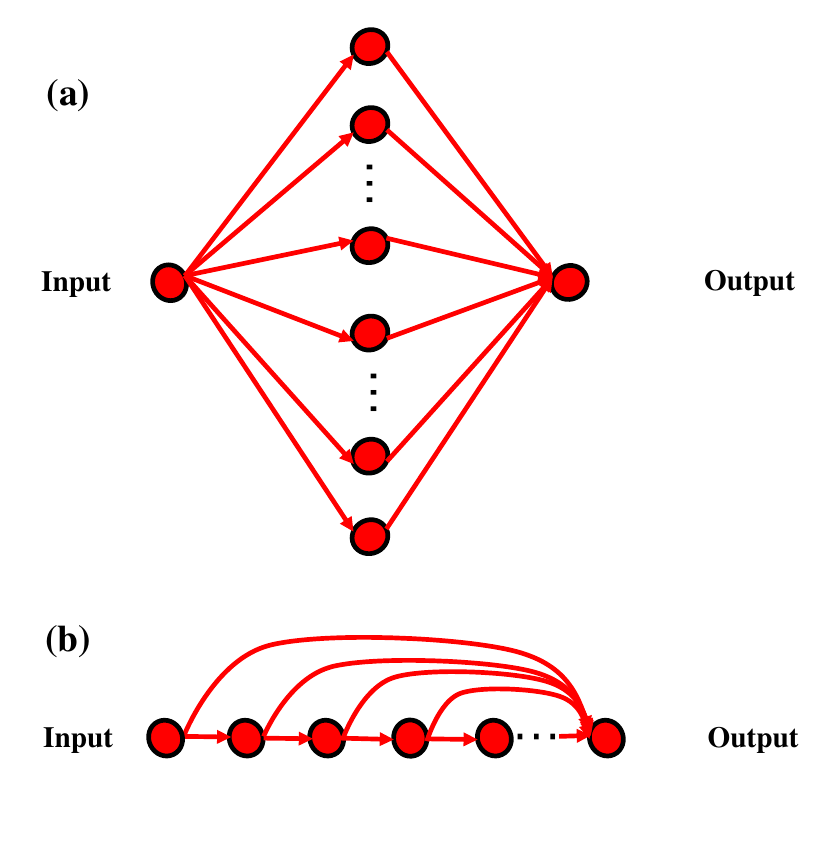}
\caption{Equivalent wide (a) and deep (b) univariate ReLU networks.}
\label{fig:1Ddual}
\end{figure}

Next, we construct an equivalent deep network dual to the above wide network. Note that \eqref{eqn:1Dwide} is the same as
\begin{equation}
    \displaystyle  f(x_0) + \sum_{i=0}^n \sgn(w^{(i)}-w^{(i-1)})\sigma\Big(|w^{(i)}-w^{(i-1)}|(x-x_i)\Big),
\end{equation}
where $\sgn(x<0)=-1$ and $\sgn(x>0)=1$.
If we write $R_i(x) = \sigma(|w^{(i)}-w^{(i-1)}|(x-x_i)), i = 0,\cdots,n-1$, then it is clear that the following recursive relation holds:
\begin{equation}
R_{i+1}(x)  = \sigma\left(|w^{(i+1)}-w^{(i)}| (\frac{1}{|w^{(i)}-w^{(i-1)}|}R_i(x) - x_{i+1} + x_{i})\right).
\label{eqn:1Drecursive}
\end{equation}
Thanks to the recurrent relation, each $R_i$ can accurately represent a small monotonically increasing piece over $[x_{i}, x_{i+1}]$. We aggregate the outputs of those $n+1$ pieces in the output neuron to get the desired deep ReLU network $\mathbf{H}_2(x)$: 
\begin{equation}
\mathbf{H}_2(x) = f(x_0) + \sum_{i=0}^{n} \sgn(w^{(i)}-w^{(i-1)}) R_i(x),
\label{eqn:1Ddeep}
\end{equation}
which is the same as $\mathbf{H}_1(x)$. To construct $R_n(x)$, $n+1$ consecutive neurons are connected in series. Therefore, such a one-neuron-wide network has $n+2$ layers. We plot two types of neural networks $\mathbf{H}_1(x)$ and $\mathbf{H}_2(x)$ in \textbf{Figure \ref{fig:1Ddual}}.
\end{proof}

\begin{theorem} [Quasi-Equivalence of Multivariate Regression Networks] Suppose that the representation of an arbitrary ReLU network is $h: [-B,~B]^D \to \mathbb{R}$, and $M$ is the minimum number of simplices to cover the polytopes to support $h$, for any $\delta>0$, there exist a wide ReLU network $\mathbf{H}_1$ of width $\mathcal{O}\left[D(D+1) (2^{D}-1)M \right]$ and depth $D$, and also a deep ReLU network $\mathbf{H}_2$ of width $(D+1)D^2$ and depth $\mathcal{O}\left[(D+1)M\right]$, satisfying that
\begin{equation}
    \begin{aligned}
    & \m\Big(\x~|~h(\x)\neq \mathbf{H}_1(\x)\}\Big) < \delta \\
    & \m\Big(\x~|~h(\x)\neq \mathbf{H}_2(\x)\}\Big) < \delta, 
    \end{aligned}
\end{equation}
where $\m(\cdot)$ is the standard measure in $[-B,~B]^D$. 
\label{thm:maind=N}
\end{theorem}

\textbf{The sketch of proof:} A ReLU network is a piecewise linear function over polytopes, which can be decomposed into a summation of linear functions over a simplex. \textbf{Lemma} \ref{lem:wide_deep_modules_ND} shows that a wider network module $\mathbf{N}_1(\x)$ and a deeper network module $\mathbf{N}_2(\x)$ can represent an arbitrary linear function over a simplex. Next, in \textbf{Theorem} \ref{thm:maind=N}, to transform an arbitrary ReLU network into a wide and a deep network, we horizontally aggregate network modules $\mathbf{N}_1(\x)$ to have a wide network, and we use shortcuts to sequentially establish a deep network.

A $D$-simplex $S$ is a $D$-dimensional convex hull provided by convex combinations of  $D+1$ affinely independent vectors $\{\v_i\}_{i=0}^D \subset \real^D$. In other words, $\displaystyle S = \left\{ \sum_{i=0}^D \xi_i \v_i ~|~ \xi_i \geq 0, \sum_{i=0}^D \xi_i = 1 \right \}$. 

If we write $V = (\v_1 - \v_0,\v_2 - \v_0,\cdots,\v_D - \v_0)$, then $V$ is invertible, and $S = \left\{\v_0 + V\x ~|~ \x \in \Delta \right\}$ where $\Delta = \left\{\x\in\real^D~|~ \x \geq 0, \mathbf{1}^\top \x \leq 1 \right \}$ is a template simplex in $\real^D$. It is clear that the following one-to-one affine mapping between $S$ and $\Delta$ exists, which is  
\begin{equation}
    T:S\rightarrow \Delta, \p \mapsto T(\p) = V^{-1} (\p - \v_0).
\label{eqn:simplex_transform_SI}    
\end{equation}
Therefore, we only need to prove the statement on the special case that $S = \Delta$.

We write $D+1$ vertices of $S$ as $\v_0 = (0,0,\cdots,0), \v_1 = (1,0,\cdots,0), \cdots, \v_{D+1} = (0,\cdots,0,1)$. Then $f: [-B,~B]^D\rightarrow\real$ supported on $S$ is provided as
\begin{equation}
f(\x) = \Big{\{}\begin{array}{cc}  \a^\top \x + b,& \textbf{if}\quad  \x \in S\\
0,&\ \textbf{if}\quad \x\in S^{c}
\end{array},
\label{eqn:NDf}
\end{equation}
where $\a = (f(\v_1) - f(\v_0),f(\v_2) - f(\v_0),\cdots,f(\v_D) - f(\v_0)), b = f(\v_0)$. We denote the domain of a network as $[-B,B]^D$. Given a linear function $\ell(\x) = c_1x_1 + c_2x_2 +\cdots+c_nx_n+ c_{n+1}$, we write $\ell^{-} = \{\x\in\real^D~|~ \ell(\x) <0\}$ 
and $\ell^{+} = \{\x\in\real^D~|~ \ell(\x) \geq 0\}$. $S$ is enclosed by $D+1$ hyperplanes provided by $\ell_i(\x)=x_i, i = 1,\cdots,D $, and $\ell_{D+1}(\x)=-x_1 -\cdots-x_D +1$.

\begin{lemma}
Suppose that the representation of an arbitrary ReLU network is $f: [-B,~B]^D \to \mathbb{R}$ expressed as Eq. \eqref{eqn:NDf},
for any $\delta>0$, there exist a ReLU network $\mathbf{N}_1$ of width $D(D+1) (2^{D}-1)+2$ and depth $D$, and also a ReLU network $\mathbf{N}_2$ of width $(D+1)D^2$ and depth $D+1$, satisfying that
\begin{equation}
    \begin{aligned}
    & \m\Big(\x~|~f(\x)\neq \mathbf{N}_1(\x)\}\Big) < \delta \\
    & \m\Big(\x~|~f(\x)\neq \mathbf{N}_2(\x)\}\Big) < \delta.
    \end{aligned}
\end{equation}
\label{lem:wide_deep_modules_ND}
\end{lemma}

\begin{proof} (\textbf{Lemma \ref{lem:wide_deep_modules_ND}}, $\mathbf{D\geq2}$) 
Our goal is to approximate the given piecewise linear function $f$ using ReLU networks. We first index the polytopes separated by $D+1$ hyperplanes $\ell_i(\x) = 0, i=1,\cdots,D+1$ as $\mathcal{A}^{(\chi_1, \cdots,\chi_i,\cdots, \chi_{D+1})}=\ell_1^{\chi_1} \cap \cdots \cap \ell_i^{\chi_i} \cap \cdots \cap \ell_{D+1}^{\chi_{D+1}}, \chi_i \in \{+,-\}, i=1,\cdots,D+1$. It is clear to see that $S = \A^{(+,+,\cdots,+)}$. In addition, we use $\vee$ to denote exclusion of certain component. For instance, $\mathcal{A}^{(\chi_1, \vee, \chi_3,\cdots,\chi_{D+1})} =  \ell_1^{\chi_1}  \cap \ell_3^{\chi_3}\cap \cdots \cap \ell_{D+1}^{\chi_{D+1}}$. It can be easily verified that
\begin{equation}
\mathcal{A}^{(\chi_1, \vee, \chi_3,\cdots,\chi_{D+1})} = \mathcal{A}^{(\chi_1, +, \chi_3,\cdots,\chi_{D+1})}\cup \mathcal{A}^{(\chi_1, -, \chi_3,\cdots,\chi_{D+1})}.
\label{eqn:Combine}
\end{equation}
Please note that $\mathcal{A}^{(-,-,\cdots,-)} = \emptyset$. Thus, $D+1$ hyperplanes create in total $2^{D+1}-1$ polytopes in the $[-B,B]^D$. 

Now we recursively define an essential building block, a D-dimensional fan-shaped ReLU network $F_D (\x)$:
\begin{equation}
    \begin{cases}
    F_1(\x) = h_1(\x)\\
    F_{j+1}(\x) = \sigma \circ ( F_j(\x)-\mu^j \sigma \circ h_{j+1}(\x)), \quad j = 1,\cdots, D-1,
    \end{cases}
\end{equation}
where the set of linear functions $\{h_k (\x) =  \p_k^\top \x + r_k\}_{k=1}^{D}$ are provided by $D$ linearly independent vectors $\{\p_k\}_{k=1}^D$, and $\mu$ is a large positive number ($\mu^j$ denotes $\mu$ with the power to $j$). Note that the network $F_D$ is of width $D$ and depth $D$. This network enjoys the following key characteristics: 1) As $\mu \rightarrow \infty$, the hyperplane $h_{1}- \mu h_{2} -\cdots-\mu^j h_{j+1} =0 $ is approximate to the hyperplane $h_{j+1} =0$ as the term $\mu^j h_{j+1}$ dominates. 
Thus, the support of $F_D (\x)$ converges to $h_1^+ \cap h_2^-\cap \cdots \cap h_D^-$ which is a $D$-dimensional fan-shaped function.
2) Let $C$ be the maximum area of hyperplanes in $[-B,~B]^D$. Because the real boundary $h_{1}- \mu h_{2} -\cdots-\mu^j h_{j+1} =0 $ is almost parallel to the ideal boundary $h_{j+1} =0$, the measure of the imprecise domain caused by $\mu^j$ is at most $C/\mu^j$, where $1/\mu^j$ is the approximate distance between the real and ideal boundaries. In total, the measure of the inaccurate region in building $F_D (\x)$ is at most $C \sum_{j=1}^{D-1}1/\mu^j\leq C/(\mu - 1) $.
3) The function over $D$-dimensional fan-shaped domain is $h_1$, since $(h_j)^+ = 0, j\geq 2$ over the $D$-dimensional fan-shaped domain. 

\underline{Constructing $\mathbf{N}_1$}: Discontinuity of $f$ in Eq.\eqref{eqn:NDf} is one of the major challenges of representing it using a ReLU network. To tackle this issue, we start from a linear function $\tf(\x) = \a^\top \x + b, \forall \x\in\real^D$, which can be represented by two neurons  $\sigma\circ \tf-\sigma\circ (-\tilde{f})$. The key idea is to eliminate $f$ over all $2^{D+1}-2$ polytopes outside $S$ using the $D$-dimensional fan-shaped functions.

Let us use $\mathcal{A}^{(+,+,+,-,\cdots,-)}$ and $\mathcal{A}^{(+,+,-,-,\cdots,-)}$ to show how to cancel the function $\tf$ over the polytopes outside $S$. According to \eqref{eqn:Combine}, $\mathcal{A}^{(+,+,+,-,\cdots,-)}$ and $\mathcal{A}^{(+,+,-,-,\cdots,-)}$ satisfy
\begin{equation}
    \mathcal{A}^{(+,+,\vee,-,\cdots,-)} = \mathcal{A}^{(+,+,+,-,\cdots,-)} \cup \mathcal{A}^{(+,+,-,-,\cdots,-)},
\label{eq:neighbors}    
\end{equation}
where $\mathcal{A}^{(+,+,\vee,-,\cdots,-)}$ is a $D$-dimensional fan-shaped domain. Without loss of generality, a number $D+1$ of $D$-dimensional fan-shaped functions over $\mathcal{A}^{(+,+,\vee,-,\cdots,-)}$ are needed as the group of linear independent bases to cancel $\tf$ , where the $k^{th}$ fan-shaped function is constructed as
\begin{equation}
    \begin{cases}
        & F_1^{(k)} = x_1-\eta_k x_k\\
        & F_2^{(k)} = \sigma \circ \big(F_1^{(k)}-\mu \sigma \circ (-x_2) \big)\\        
        & F_{3}^{(k)} =  \sigma \circ\big(F_{2}^{(k)} -\mu^2 \sigma \circ(x_{4}) \big) \\
        & F_{4}^{(k)} = \sigma \circ \big(F_{3}^{(k)} -\mu^3 \sigma \circ(x_{5}) \big) \\
        & \ \ \ \ \vdots \\
        & F_{D}^{(k)} =   \sigma \circ \big(F_{D-1}^{(k)}-\mu^{D-1} \sigma \circ(-x_1-\cdots-x_D+1) \big),
    \end{cases}
\end{equation}
where we let $x_{D+1}=1$ for consistency, the negative sign for $x_2$ is to make sure that the fan-shaped region $\ell_1^+ \cap (-\ell_2)^- \cap \ell_4^- \cap \cdots \cap \ell_{D+1}^-$ of $F_{D}^{(k)}$ is $\mathcal{A}^{(+,+,\vee,-,\cdots,-)}$, $\eta_1=0$, and $\eta_k=\eta, k=2,...,D+1$ represents a small shift for $x_1=0$ such that $\m((x_1)^+ \cap (x_1-\eta_k x_k)^-) < C\eta_k $. The constructed function over $\mathcal{A}^{(+,+,\vee,-,\cdots,-)}$ is
\begin{equation}
F_{D}^{(k)} = x_1-\eta_k x_k, k=1,\cdots,D+1,
\end{equation}
which is approximately over 
\begin{equation}
    \forall \x\in\mathcal{A}^{(+,+,\vee,-,\cdots,-)}\backslash \Big((x_1)^+ \cap (x_1-\eta_k x_k)^-) \Big)  .
\end{equation}
Let us find $\omega_1^*,\cdots,\omega_{D+1}^*$ by solving 
\begin{equation}
\begin{bmatrix}
1 &1 & 1 &\cdots &  1\\
0 &-\eta & 0 &\cdots &  0 \\
0 &0 & -\eta &\cdots &  0 \\
\vdots&  &  & \vdots & \\
0 &0 & 0 &\cdots &  -\eta \\
\end{bmatrix} \begin{bmatrix}
\omega_1 \\
\omega_2  \\
\omega_3 \\
\vdots \\
\omega_{D+1}
\end{bmatrix}=-\begin{bmatrix}
a_1 \\
a_2  \\
a_3 \\
\vdots \\
b
\end{bmatrix},
\end{equation}
and then the new function 
$F^{(+,+,\vee,-,\cdots,-)}(\x)=\sum_{k=1}^{D+1}\omega_k^* F_D^{(k)}(\x)$ satisfies that
\begin{equation}
\begin{aligned}
       &  \m\Big(\{\x\in \mathcal{A}^{(+,+,\vee,-,\cdots,-)}|~\tf(\x) + F^{(+,+,\vee,-,\cdots,-)}(\x) \neq 0 \} \Big) \\
     & \leq C (D\eta+\frac{D+1}{\mu-1}).
\end{aligned}
\end{equation}

Similarly, we can construct other functions $\overbrace{F^{(+,-,\vee,-,\cdots,-)}(\x),F^{(+,\vee,-,-,\cdots,-)}(\x),\cdots}^{~2^{D}-2 \ \ terms} $ 
to cancel $\tf$ over other polytopes. Finally, these $D$-dimenional fan-shaped functions are aggregated to form the following wide ReLU network $\mathbf{N}_1(\x)$:
\begin{equation}
\begin{aligned}
    & \mathbf{N}_1 (\x) = \sigma\circ (\tf(\x))-\sigma\circ (-\tf(\x)) \\
    & +\overbrace{F^{(+,+,\vee,-,\cdots,-)}(\x) + F^{(+,\vee,-,-,\cdots,-)}(\x)+\cdots}^{~{2^{D}-1} \ \ terms},
\end{aligned}    
\end{equation}
where the width and depth of the network are  $D(D+1)(2^{D}-1)+2$ and $D$ respectively. In addition, because there are $2^{D}-1$ polytopes being cancelled, the total area of the regions suffering from errors is no more than  
\begin{equation}
    (2^{D}-1)C (D\eta+\frac{D+1}{\mu-1}).
\end{equation}
Therefore, for any $\delta>0$, as long as we choose appropriate $\mu, \eta $ that fulfill
\begin{equation}
0 < 1/{(\mu-1)} , \eta< \frac{\delta}{(2^{D}-1)C(D+D+1)}=\frac{\delta}{(2^{D}-1)C(2D+1)},
\end{equation}
the constructed network $\mathbf{N}_1(\x)$ will have
\begin{equation}
    \m\left(\{\x\in\real^D~|~f(\x)\neq \mathbf{N}_1(\x)\}\right) < \delta.
\end{equation}

\underline{Constructing $\mathbf{N}_2$}: Allowing more layers in a network provides an alternate way to represent $f$. The fan-shaped functions remain to be used. The whole pipeline can be divided into two steps: (1) build a function over $S$; and (2) represent $f$ over $S$ by slightly moving one boundary of $S$ to create linear independent bases.

(1) We construct a number $D$ of $D$-dimensional fan-shaped functions. Without loss of generality, the $k^{th}$ fan-shaped function is constructed as
\begin{equation}
    \begin{cases}
        & F_1^{(k)} = x_1-\nu_k x_k\\
        & F_2^{(k)} = \sigma \circ \big(F_1^{(k)} - \mu^1 \sigma \circ (-x_2) \big)\\        
        & \ \ \ \ \vdots \\
        & F_{D}^{(k)} =   \sigma \circ \big(F_{D-1}^{(k)} - \mu^{D+1} \sigma \circ(-x_D)  \big),
    \end{cases}
\end{equation}
whose fan-shaped region is approximately $(\ell_1-\nu_k x_k)^+ \cap (-\ell_2)^-\cap \cdots \cap (-\ell_D)^- = (\ell_1-\nu_k x_k)^+ \cap \ell_2^+\cap \cdots \cap \ell_D^+$, which almost overlaps with $\mathcal{A}^{(+,\cdots,+,\vee)}=\ell_1^+ \cap \ell_2^+\cap \cdots \cap \ell_D^+$ as $\nu_k$ becomes sufficiently small. The output of $F_{D}^{(k)}$ is $x_1-\nu_k x_k, k=1,\cdots,D$.  To obtain the last boundary $\ell_{D+1}(\x)=-x_1-\cdots-x_D +1=0$ so as to construct the simplex $S$, we stack one more layer with only one neuron as follows: 
\begin{equation}
    \mathbf{E}^1 (\x) = (\gamma_1^* F_{D}^{(1)}+\cdots \gamma_D^* F_{D}^{(D)}+\gamma_{D+1}^*)^+,
\end{equation}
where $\gamma_1^*,\cdots, \gamma_{D+1}^*$ are the roots of the following equation:
\begin{equation}
    \begin{cases}
    (1-\nu_1) \gamma_1 + \gamma_2+\cdots+\gamma_D = -1& \\
    -\nu_2 \gamma_2  = -1 & \\
    \ \ \ \ \ \vdots & \\
    -\nu_D \gamma_D = -1 & \\
     \gamma_{D+1}= 1.
    \end{cases}
\end{equation}
Thus, $\mathbf{E}^1(\x)$ will approximately represent the linear function $-x_1-\cdots-x_D +1$ over $S$ and zero elsewhere. The depth and width of the network are $D+1$ and $D^2$ respectively. Similarly, due to the employment of a number $D$ of $D$-dimensional fan-shaped functions and the effect of $\nu_k$, the area of the region with errors is estimated as
\begin{equation}
   CD/{(\mu-1)} +C\sum_{k=1}^{D}\nu_{k}. 
\end{equation}

(2) To acquire an arbitrary linear function, similarly we need $D+1$ linear independent functions as linear independent bases. Other than the one obtained in step (1), we further modify $\ell_{D+1}$ a little bit $D$ times to get $\ell_{D+1}^{l} = \ell_{D+1}-\tau_l x_l, l=1,\cdots,D$. Repeating the same procedure described in step (1), for $\ell_{D+1}^{l}$, we can construct the network $\mathbf{E}^{l} (\x)$ that is $\ell_{D+1}^{l}-\tau_l x_l$ approximately over $\ell_1^{+}\cap \ell_2^{+} \cap \cdots \cap (\ell_{D+1}^{l})^{+}$, where $\tau_l, l=1,\cdots,D$ is small to render these domains almost identical to $S$, and $\tau_l, l=1,\cdots,D$ satisfies that
\begin{equation}
\begin{bmatrix}
-1-\tau_1 & -1 &\cdots & -1 & -1\\
-1  & -1 &\cdots & -1 & -1 \\
\vdots&  &  & \vdots & \\
-1  & -1 &\cdots & -1-\tau_D & -1 \\
1  & 1 &\cdots & 1 & 1 \\
\end{bmatrix}\begin{bmatrix}
\rho_1^*\\
\rho_2^*\\
\vdots \\
\rho_D^* \\
\rho_{D+1}^*
\end{bmatrix} = \begin{bmatrix}
a_1\\
a_2\\
\vdots\\
a_D \\
b
\end{bmatrix},
\end{equation}
where $\rho_1^*,\cdots,\rho_{D+1}^*$ are the roots. As a result, the deep network is built as
\begin{equation}
  \mathbf{N}_2(\x) = \rho_1^*\mathbf{E}^1(\x)+\rho_2^*\mathbf{E}^{2}(\x)+\cdots+\rho_{D+1}^*\mathbf{E}^{D+1}(\x),  
\end{equation}
producing $f$ on $S$. The depth and width of $\mathbf{N}_2(\x)$ are $D+1$ and $D^2(D+1)$, respectively. Similarly, the area of the region with errors is bounded above by 
\begin{equation}
 CD(D+1)/{(\mu-1)} +C(D+1)\sum_{k=1}^{D}\nu_{k}+C \sum_{l=1}^{D}\tau_{l}.
\end{equation}
Therefore, for any $\delta>0$, if we choose $\mu, \nu_k, \tau_l$ appropriately such that
\begin{equation}
\begin{aligned}
       0& < 1/{(\mu-1)},\nu_{k},\tau_{l} \\
       & < \frac{\delta}{C(D(D+1)+(D+1)D+D)}=\frac{\delta}{C(2D^2+3D)},
\end{aligned}
\end{equation}
then the constructed network $\mathbf{N}_2(\x)$ will satisfy
\begin{equation}
    \m\left(\{\x \in \mathbb{R}^D~|~f(\x)\neq \mathbf{N}_2(\x)\}\right) < \delta .
\end{equation}

\end{proof}

\begin{proof} (\textbf{Theorem \ref{thm:maind=N}}, $\mathbf{D\geq2}$) 
The network $h$ is piecewise linear and splits the space into polytopes. It is feasible to employ a number of simplices to fill the polytopes defined by $h$ \citep{ehrenborg2007perles}. Given that $M$ is the minimum number of required simplices, we have
\begin{equation}
    h(\x) = \sum_{m=1}^M f^{(m)}(\x), 
\end{equation}
where 
\begin{equation}
f^{(m)}(\x) = \left\{\begin{array}{cc}   (\a_1^{(m)})^\top \x + b^{(m)},& \textbf{if}\quad  \x \in S^{(m)}\\
0,&  \textbf{if} \quad \x\notin S^{(m)}
\end{array}\right.,
\end{equation}
and $S^{(m)}$ is the $m$-th simplex. The core of the wide construction is to horizontally aggregate network modules representing $f^{(m)}(\x)$ into a wide network. In contrast, the core of the deep construction is to sequentially express $f^{(m)}(\x)$ in a network module without linking these modules to the input. 

\underline{Representing $h$ with a wide ReLU network}: \textbf{Lemma \ref{lem:wide_deep_modules_ND}} suggests that a wide network module $\mathbf{N}_1$ can generically represent a function over a template simplex. To represent $f^{(m)}$ over $S^{(m)}$, we need to use Eq. \eqref{eqn:simplex_transform_SI} to transform the function from the barycentric coordinate system to the Euclidean coordinate system. Let $D+1$ vertices of $S^{(m)}$ be $\{\v_0^{(m)},\v_1^{(m)},\cdots,\v_D^{(m)}\}$ and $V^{(m)} = (\v_1^{(m)} - \v_0^{(m)},\v_2^{(m)} - \v_0^{(m)}, \cdots, \v_D^{(m)} - \v_0^{(m)})$, we have 
\begin{equation}
    \mathbf{N}_1^{(m)}(\x) = \mathbf{N}_1((V^{(m)})^{-1}(\x-\v_0^{(m)})),
\end{equation}
satisfying 
\begin{equation}
    \m\left(\{\x\in\real^D~|~f^{(m)}(\x)\neq \mathbf{N}_1^{(m)}(\x)\}\right) < \delta.
\end{equation}

By aggregating the network $\mathbf{N}_1^{(m)}(\x)$ horizontally, we have the following wide network:
 \begin{equation}
     \mathbf{H}_1(\x) = \sum_{m=1}^M \mathbf{N}_1^{(m)}(\x) .
 \end{equation}
Therefore, the constructed wide network $\mathbf{H}_1(\x)$ is of width $\mathcal{O}(D(D+1)(2^D-1)M)$ and depth $D$. It is clear that the width $\mathcal{O}(D(D+1)(2^D-1)M)$ of the wide network $\mathbf{H}_1(\x)$ dominates, as the number of needed simplices goes larger and larger.

\underline{Representing $h$ with a deep ReLU network}: For a deep construction, the fundamental difficulty is how to sequentially express each $f^{(m)}$, \textit{i.e.}, the input of each block can only come from the earlier block instead of the input, like what we did in one-dimensional case (Figure \ref{fig:deep_structure_2d}(a)). Let us use mathematical deduction to derive how to sequentially represent each $f^{(m)}$. We still adopt the idea of modularized networks, but now each network module has two outputs. \textbf{Lemma \ref{lem:wide_deep_modules_ND}} suggests that a deep network module $\mathbf{N}_2$ can generically represent a function over a template simplex. 

\textbf{Step 1.} Assume that the two outputs of the first block are $\mathbf{N}_2^{(1)}$ and $\Omega^{(1)}$. To represent $f^{(1)}$ over $S^{(1)}$, similarly, we need to transform the function from the barycentric coordinate system to the Euclidean coordinate system. Let $D+1$ vertices of $S^{(1)}$ be $\{\v_0^{(1)},\v_1^{(1)},\cdots, \v_D^{(1)}\}$ and $V^{(1)} = (\v_1^{(1)} - \v_0^{(1)},\v_2^{(1)} - \v_0^{(1)}, \cdots, \v_D^{(1)} - \v_0^{(1)})$, we have 

\begin{equation}
    \mathbf{N}_2^{(1)} = \mathbf{N}_2((V^{(1)})^{-1}(\x-\v_0^{(1)})),
\end{equation}
which is one output of the first block.

Next, we derive the other output $\Omega^{(1)}(\x)$. Note that we are not allowed to use the input directly. Encouraged by the inversion idea in the univariate case, we invert the function domain into the input domain to get a function that is approximately only supported over $S^{(1)}$:
\begin{equation}
\tilde{\x}_{S^{(1)}} = \left\{\begin{array}{cc}   \x,& \textbf{if}\quad  \x \in S^{(1)}\\
0,& \textbf{if} \quad \x\notin S^{(1)}
\end{array}\right..
\end{equation}
To do so, recall that we have $\mathbf{E}^1, \mathbf{E}^2, \cdots, \mathbf{E}^{D+1}$ in constructing $\mathbf{N}_2$, we will have a coefficient vector $(\xi_1^*,\xi_2^*,\cdots, \xi_{D+1}^*)$ such that 
\begin{equation}
   \tilde{\x}_{S^{(1)}} = \xi_1^*\mathbf{E}^1((V^{(1)})^{-1}(\x-\v_0^{(1)}))+\xi_2^*\mathbf{E}^{2}((V^{(1)})^{-1}(\x-\v_0^{(1)}))+\cdots+\xi_{D+1}^*\mathbf{E}^{D+1}((V^{(1)})^{-1}(\x-\v_0^{(1)})) 
\end{equation}

As shown in Figure \ref{fig:deep_structure_2d}(b), use the residual connection, we compute $\Omega^{(1)}(\x)= \x-\tilde{\x}_{S^{(1)}}$, which is zero over $S^{1}$ and $\x$ for other regions. $\Omega^{(1)}(\x)$ will be used to feed the next block to construct $f^{(2)}$. 


\textbf{Step 2}. Suppose that the two output functions of the $m$-th block are $\mathbf{N}_2^{m}(\x)$ and $\Omega^{(m)}(\x)=\x-\sum_{i=1}^{m} \tilde{\x}_{S^{(i)}}$. $\Omega^{m}(\x)$ is fed into the $(m+1)$-th block as the input. Because $S^{(m+1)}$ is outside the domain of $S^{(1)}\cup \cdots \cup S^{(m)}$, with the same technique in \textbf{Step 1}, we construct
\begin{equation}
    \mathbf{N}_2^{(m+1)}(\Omega^{(m)}(\x)) = \mathbf{N}_2((V^{(m+1)})^{-1}(\Omega^{(m)}(\x)-\v_0^{(1)})),
\end{equation}
where $\{\v_0^{(m+1)},\v_1^{(m+1)},\cdots, \v_D^{(m+1)}\}$ are three vertices of $S^{(m+1)}$, and $V^{(m+1)} = (\v_1^{(m+1)} - \v_0^{(m+1)},\v_2^{(m+1)} - \v_0^{(m+1)}, \cdots, \v_D^{(m+1)} - \v_0^{(m+1)})$ to express $f^{(m+1)}$ over $S^{(m+1)}$ well. $\Omega^{(m)}(\x)$ is functionally equivalent to $\x$ for two reasons. First, all simplices do not overlap. $\Omega^{(m+1)}(\x)$ is $\x$ outside the domain of $S^{(1)}\cup \cdots \cup S^{(m)}$; therefore zero value over $S^{(1)}\cup \cdots \cup S^{(m)}$ has no effect. Second, w.l.o.g., we assume that $\mathbf{0}$ is outside of all simplices or lies on the boundary of a certain simplex. As such, we have $\mathbf{N}_2^{(m+1)}(\mathbf{0})=0, \forall m$, therefore, $\mathbf{N}_2^{(m+1)}(\Omega^{(m)}(\x))$ will not erroneously produce a non-zero constant over $S^{(1)}\cup \cdots \cup S^{(m)}$. Furthermore, we also obtain a function $\tilde{\x}_{S^{(m+1)}}$ inside the $(m+1)$-th block. Apply the residual operation, we have 
\begin{equation}
    \Omega^{(m+1)}(\x)=\Omega^{(m)}(\x) - \tilde{\x}_{S^{(m+1)}} =  \x-\sum_{i=1}^{m+1} \tilde{\x}_{S^{(i)}}.
\end{equation}

Lastly, simlilar to the one-dimensional deep network, we use the shortcut connection to aggregate $\mathbf{N}_2^{(m)}(\Omega^{(m-1)}(\x))$ to obtain the following deep network:
 \begin{equation}
     \mathbf{H}_2(\x) = \sum_{m=1}^M \mathbf{N}_2^{(m)}(\Omega^{(m-1)}(\x)).
 \end{equation}
Therefore, the constructed deep network $\mathbf{H}_2(\x)$ is of depth $\mathcal{O}((D+1)M)$ and width $(D+1)D^2$. Please note that in the above equation, the summation is made by shortcuts. It is clear that the depth of $\mathbf{H}_2(\x)$ dominates over the width.

\end{proof}

\begin{figure*}[h]
\center{\includegraphics[height=2.0in,width=6.2in,scale=0.4] {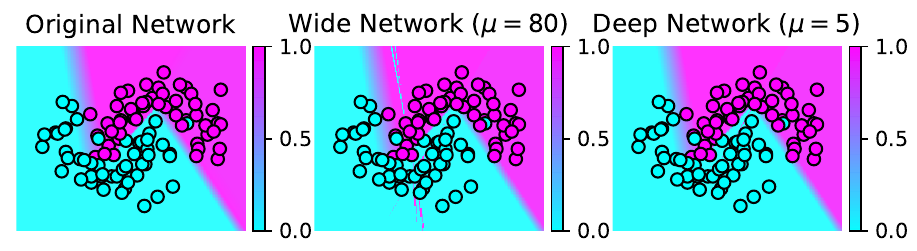}}
\caption{With the procedure in the proof for \textbf{Theorem \ref{thm:maind=N}}, an exemplary 2-6-2-1 network is transformed into a wide network ($\mu=80$) and a deep network ($\mu=5$) respectively.}
\label{fig:experiments}
\end{figure*}

\textbf{Experiment:} A 2-6-2-1 network was trained on the moon dataset. When the training process was finished, the accuracy of the network was 0.94. Then, we transformed the original network into a wide network ($\mu=80$) and a deep network ($\mu=5$) using the procedures described in the proof of \textbf{Theorem \ref{thm:maind=N}}. The patterns of outputs of the constructed wide network and deep network are respectively shown in \textbf{Figure \ref{fig:experiments}}. It is seen that the constructed networks can approximate the original network well, and some expected artifacts can be squeezed by further increasing $\mu$ or using a post-processing network that can identify streaks and remove them.

\clearpage

\newpage
\addcontentsline{toc}{section}{III. Width-Depth equivalence by the De Morgan Law}
\subsection*{III. Width-Depth equivalence by the De Morgan Law}

We have the following formal statement for the quasi-equivalence in light of the De Morgan law. The proof of this theorem is constructive.

\begin{prop}[Proposition C.1 in \citep{lin2018resnet}]
Given a piecewise constant function  $h: \mathbb{R}^{d} \rightarrow \mathbb{R}$, for any small enough $\epsilon>0$, there exists a ResNet $\mathbf{H}_2(\x)$ such that 
\begin{equation}
\m\Big(\{\x~|\mathbf{H}_2(\x)\}\neq h(\x)\} \Big) < \epsilon.
\end{equation}
\label{prop:resnet}
\end{prop}

\textit{The sketch of proof.} The detailed proof can be referred to \cite{lin2018resnet}. We only show the one-dimensional case here for a brief introduction of the idea. 

The operations that are realizable by a residual network block are
\begin{enumerate}
    \item shifting: $G^+=G+c$, for any $c\in \mathbb{R}$. This operation is to shift a function with a constant. 
    \item min or max: $G^+=\min \{G, c\}=G-\sigma(G-c)$ or $G^+=\max \{G, c\}=G+\sigma(c-G)$, for any $c\in \mathbb{R}$. This operation allows us to threshold a function.
    \item min or max with a linear transformation: $G^+=\min \{G, \alpha G+\beta\}=G-\sigma(G-(\alpha G+\beta))$ or $G^+=\max \{G, \alpha G+\beta \}= G+\sigma((\alpha G+\beta)-G)$ for any $\alpha, \beta \in \mathbb{R}$. This operation can adjust the slope of trapezoid functions.
\end{enumerate}

\begin{figure*}[htb!]
\center{\includegraphics[width=\linewidth] {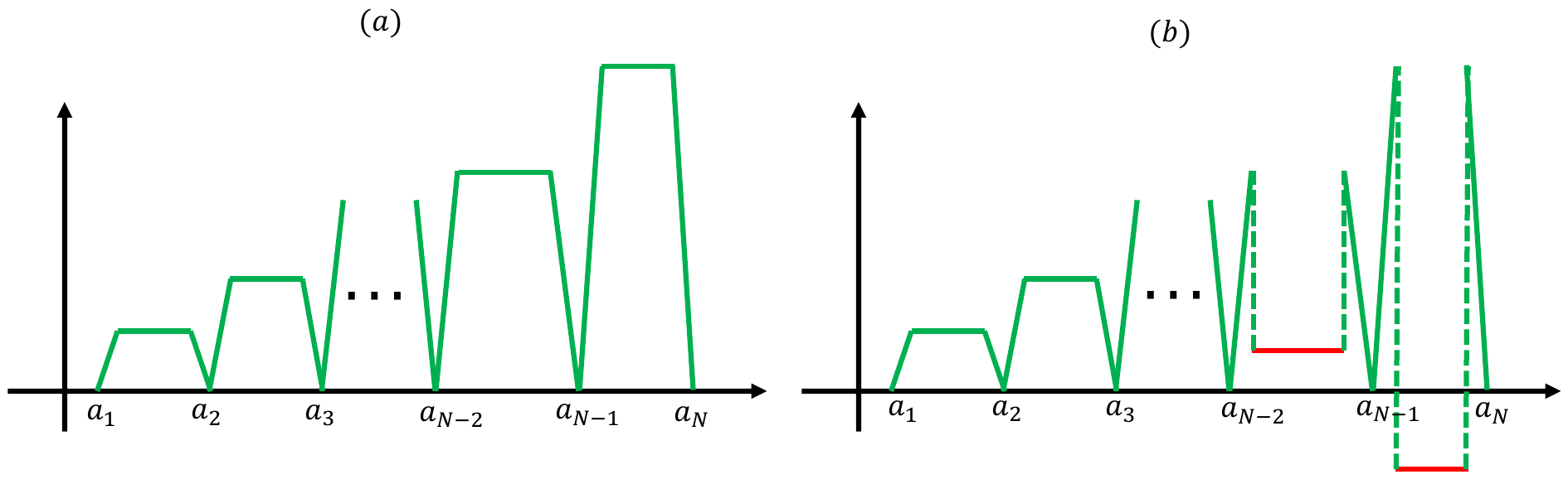}}
\caption{To approximate any univariate function, (a) we construct an increasing trapezoidal function; (b) we adjust the height of each trapezoid function to its corresponding height.}
\label{Figure_theorem3_sketch}
\end{figure*}

As Figure \ref{Figure_theorem3_sketch} shows, we first construct an increasing trapezoidal function with a very high plateau over each trapezoid. Then, we adjust the height of each trapezoid function to its corresponding height. The procedures of constructing an increasing function are shown in Figure \ref{Figure_theorems_trick}, which are based on mathematical induction. Suppose $G_m$ satisfies 
\begin{enumerate}
    \item $G_m=0$, when $x \in [-\infty, a_1]$.
    \item $G_m$ is a trapezoid function over each $[a_i,a_{i+1}]$.
    \item $G_m = k \Vert G \Vert_\infty $ over an interval $[a_k+\delta, a_{k+1}-\delta]$ for any $k=1,2,\cdots,m$.
    \item $0 \leq G_m \leq m\Vert G \Vert_\infty$ over $[-\infty, a_{m+1}]$.
    \item $G_m = -\frac{m\Vert G \Vert_\infty}{\delta}(x-a_{m+1})$ over $ [a_{m+1}, \infty]$.
\end{enumerate}

\begin{figure*}[htb!]
\center{\includegraphics[width=0.5\linewidth] {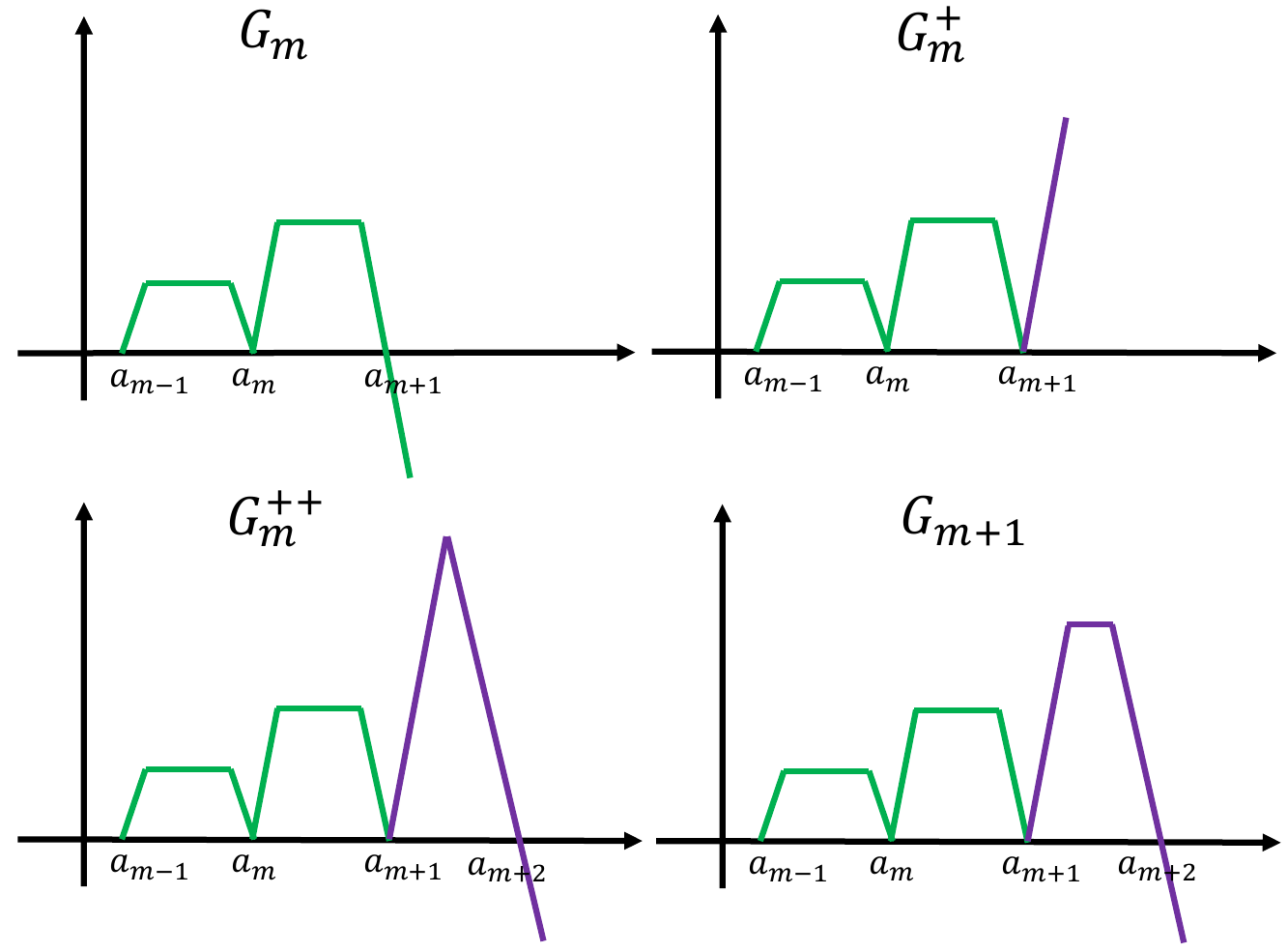}}
\caption{An illustrative plot to show how to derive $G_{m+1}$ from $G_m$.}
\label{Figure_theorems_trick}
\end{figure*}

Given $G_m$, we derive $G_{m+1}$ as follows:
\begin{equation}
    \begin{cases}
    & G_m^+ = \max \{G_m, -(1+\frac{1}{m})G_m \} \\
    & G_m^{++} = \min \{ G_m^+, -G_m^+ + \frac{(m+1)\Vert G\Vert_\infty}{\delta}(a_{m+2}-a_{m+1}) \} \\
    & G_{m+1} = \min \{ G_m^{++}, (m+1)\Vert G\Vert_\infty \}.
    \end{cases}
\end{equation}

The procedure of adjusting the height of each trapezoid function is shown in Figure \ref{Figure_theorem3_sketch}(b). An important consideration is that we need to keep the function on previous subdivisions unchanged
while twisting the current trapezoid function. This is realized by the fact that $||G||_\infty$ is large. Suppose that $\phi_k$ is the target value over the interval $[a_k, a_{k+1}]$, we adjust the values of intervals sequentially by
\begin{equation}
 R_{k-1}^* = R_k^* + \frac{\phi_k-k\Vert G\Vert_\infty}{\Vert G\Vert_\infty}\sigma[R_k^*-(k-1)\Vert G\Vert_\infty].
 \label{eqn:adjust}
\end{equation}
Finally, $R_0^*$ is the constructed function.

$\hfill\square$

\begin{theorem}[Quasi-equivalence in light of the De Morgan law]
Given a disjoint rule system $\{A_i\}_{i=1}^n$, where each rule $A_i$ is characterized by
an indicator function $g_i(\x)$ over a hypercube $\Gamma_i=[a_{1i},b_{1i}]\times\cdots\times[a_{Di},b_{Di}] \in [0,1]^D$: \[g_i(\x) = \left\{\begin{array}{cc} 1,& \textbf{if}\quad  \x \in \Gamma_i\\
0,&\textbf{if} \quad \x\in \Gamma_i^{c},
\end{array}\right.\] fulfilling the De Morgan law
\begin{equation}
 A_1 \lor A_2 \cdots \lor A_n = \neg\Big( (\neg A_1) \land (\neg A_2) \cdots  \land  (\neg A_n) \Big),   
\label{dmlaw}    
\end{equation}
we can construct a wide ReLU network $\mathbf{H}_1(\x)$ to represent the right hand side of the De Morgan law and a deep ReLU network $\mathbf{H}_2(\x)$ to represent the left hand side of the De Morgan law, such that for any $\epsilon>0$,
\begin{equation}
\m\Big(\{\x~|\mathbf{H}_1(\x)\}\neq \mathbf{H}_2(\x)\} \Big) < \epsilon,
\end{equation}
where $\m$ is a measurement in $[0,1]^D$.
\label{thm:dmlaw}
\end{theorem}

\begin{figure*}[htb]
\label{CubeWidth}
\center{\includegraphics[height=3.5in,width=6in,scale=0.4] {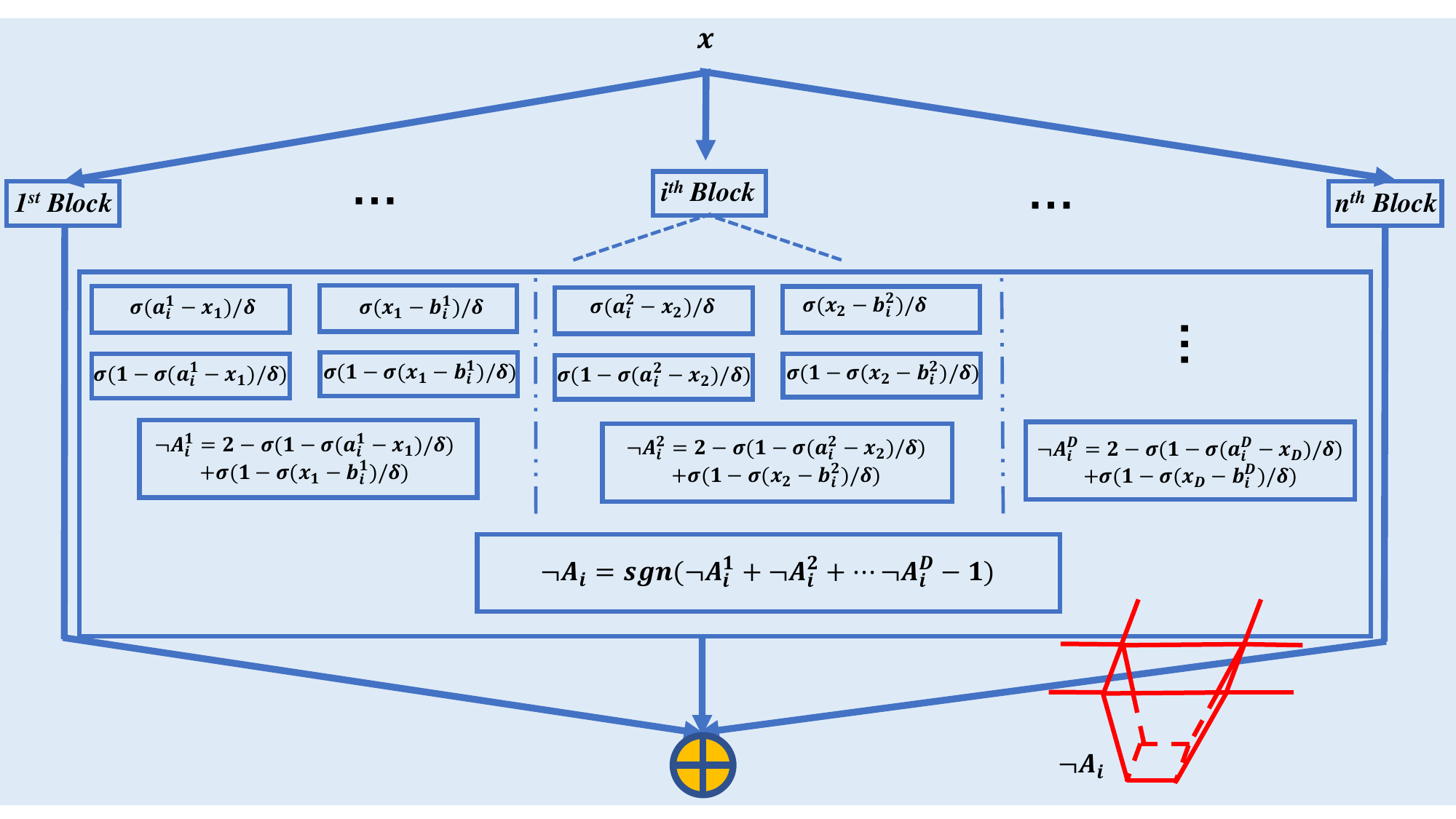}}
\caption{A wide network to realize the negation of the logic intersection of the negation of propositional rules in a high-dimensional space. The negation of each propositional rule is associated with a trap-like indicator function.}
\label{fig:CubeWidth}
\end{figure*}

\begin{proof} The defined rule system is a piecewise constant function over a hypercube. The wide and deep networks are constructed as follows:

\underline{Wide network $\mathbf{H}_1(\x)$}: As illustrated in Figure \ref{fig:CubeWidth}, to represent the rule $A_i$ that is equivalent to $g_i(\x)$, a trap-like function is constructed to represent $\neg A_i$ that is equivalent to $1-g_i(\x)$. In our construction, for each hypercube the measures of $\{\x \in \Gamma_i|\mathbf{H}_1(\x)\neq g_i(\x)\}$ is no more than $vol(\Gamma_i)(1-2\delta)^n$, where $vol(\Gamma_i)$ is the volume of a hypercube $\Gamma_i$. Therefore, the total measure of errors for all the hypercubes is less than $\sum_i^n vol(\Gamma_i)(1-2\delta)^n = (1-2\delta)^n \sum_i^n vol(\Gamma_i) < (1-2\delta)^n$. As a result, let $\delta<(1-(\epsilon/2)^{1/n})/2$, we have
\begin{equation}
\m\Big(\{\x~|\mathbf{H}_1(\x)\}\neq \sum_{i=1}^n g_i(\x)\} \Big) < \epsilon/2 .
\label{eqn:wide_dm}
\end{equation}

\underline{Deep network $\mathbf{H}_2(\x)$}: $A_1 \lor A_2 \cdots \lor A_n$ is a piecewise constant function over a hypercube, based on \textbf{Proposition \ref{prop:resnet}}, we have 
\begin{equation}
\m\Big(\{\x~|\mathbf{H}_2(\x)\}\neq \sum_{i=1}^n g_i(\x)\} \Big) < \epsilon/2 .
\label{eqn:deep_dm}
\end{equation}

Combining Eqs. \eqref{eqn:wide_dm} and \eqref{eqn:deep_dm} leads to
$\m\Big(\{\x~|\mathbf{H}_1(\x)\}\neq \mathbf{H}_2(\x)\} \Big) < \epsilon$, which concludes our proof.

\end{proof}

\clearpage

\newpage
\addcontentsline{toc}{section}{IV. Quasi-equivalence of Quadratic Networks. }
\subsection*{IV. Quasi-equivalence of Quadratic Networks.}
In this section, we first prove the correctness of the continued fraction representation of a polynomial. Then, we report the experimental results on the quasi-equivalency of quadratic networks.
\subsubsection*{1. Continued Fraction Representation}
For a general univariate polynomial, there is a continued fraction representation as follows: 
\begin{equation}
\begin{aligned}
    & \sum_{i=0}^{2N} a_{i} x^i = \sum_{k=0}^{N} a_{2k} x^{2k} + \sum_{k=0}^{N} a_{2k+1} x^{2k+1}  \\
    & = \cfrac{b_0}{1-\cfrac{b_1 x^2}{1+b_1 x^2-\cdots \cfrac{b_{N-1} x^2}{1+b_{N-1}x^2- \cfrac{b_N x^2}{1 + b_N x^2} }}}
     + \cfrac{c_0x}{1-\cfrac{c_1 x^2}{1+c_1 x^2-\cdots \cfrac{c_{N-1} x^2}{1+c_{N-1}x^2- \cfrac{c_N x^2}{1 + c_N x^2} }}}, 
\end{aligned}    
\end{equation}
where $a_i \neq 0$, $b_0 = a_0$, $b_k = a_{2k}/a_{2k-2}, k\geq 1$ and $c_0 = a_1$, $c_k = a_{2k+1}/a_{2k-1}, k\geq 1$. In the right side of \eqref{contiued}, the first part contains the terms with even powers of $x$, while the second part is for the odd powers of $x$. Because both even and odd parts are in the essentially same format, we only show the correctness of the even part:

\begin{equation}
\sum_{k=0}^{N} a_{2k} x^{2k} = \sum_{k=0}^{N} \Big(\prod_{l=0}^k b_l\Big) x^{2k} = \cfrac{b_0}{1-\cfrac{b_1 x^2}{1+b_1 x^2-\cdots \cfrac{b_{N-1} x^2}{1+b_{N-1} x^2- \cfrac{b_N x^2}{1 + b_N x^2} }}}.
    \label{contiued}
\end{equation}

Specially, we can derive the continued fraction as follows:
\begin{equation}
\begin{aligned}
\sum_{k=0}^{N} \Big(\prod_{l=0}^k b_l\Big) x^{2k} & =  b_0(1+b_1 x^2 (1+b_2 x^2(1+b_3 x^2(\cdots ( b_{N-1} x^2 (1+b_N x^2)))))) \\
& =  \cfrac{b_0}{\cfrac{1}{1+b_1 x^2(1+b_2 x^2(1+b_3 x^2(\cdots ( b_{N-1} x^2 (1+b_N x^2)))))}} \\
& = \cfrac{b_0}{1-\cfrac{b_1 x^2(1+b_2 x^2(1+b_3 x^2(\cdots ( b_{N-1} x^2 (1+b_N x^2)))))}{1+b_1 x^2 (1+b_2 x^2(1+b_3 x^2(\cdots ( b_{N-1} x^2 (1+b_N x^2)))))}} \\
& = \cfrac{b_0}{1-\cfrac{b_1 x^2}{\cfrac{1}{1+b_2 x^2(1+b_3 x^2(\cdots ( b_{N-1} x^2 (1+b_N x^2))))}+b_1 x^2 }} \\
& = \cfrac{b_0}{1-\cfrac{b_1 x^2}{1+b_1 x^2-\cfrac{b_2 x^2(1+b_3 x^2(\cdots ( b_{N-1} x^2 (1+b_N x^2))))}{1+b_2 x^2(1+b_3 x^2(\cdots ( b_{N-1} x^2 (1+b_N x^2))))} }} \\
& = \cfrac{b_0}{1-\cfrac{b_1 x^2}{1+b_1 x^2-\cdots \cfrac{b_{N-1} x^2}{1+b_{N-1}x^2- \cfrac{b_N x^2}{1 + b_N x^2} }}}.
\end{aligned}
\end{equation}

The following example helps illustrate the correctness of \eqref{contiued}:
\begin{equation}
\cfrac{b_0}{1-\cfrac{b_1 x^2}{1+b_1 x^2- \cfrac{b_2 x^2}{1+b_2 x^2}}} = \cfrac{b_0}{1-\frac{b_1 x^2 (1+b_2 x^2)}{(1+b_1 x^2)(1+b_2 x^2)-b_2 x^2}} = \cfrac{b_0}{1-\frac{b_1 x^2 + b_1 b_2 x^4}{1+b_1 x^2 + b_1 b_2 x^4}} = b_0 + b_0 b_1 x^2 + b_0 b_1 b_2 x^4.
\end{equation}

If $a_i = 0$ for certain $i$, we can derive an approximate continued fraction representation by complementing $\sum_{i=0}^{2N} a_{i} x^i$ with a term $\delta_i x^i$, where $\delta_i$ is a small constant. For a given approximation requirement, we can always choose a sufficient small, nonzero $\delta_i$ to accommodate the accuracy requirement since $x$ is bounded.

\subsubsection*{2. Experiments}
This subsection reports the experimental results on quasi-equivalence of quadratic networks in two steps. In the first step, we show the feasibility of the reciprocal activation function in the network according to the continued fraction representation. In the second step, we compare the accuracy and robustness between wide and deep networks.

\textbf{Experimental Design:} We first preprocessed the MNIST dataset using image deskewing and dimension deduction techniques. Image deskewing (\url{https://fsix.github.io/mnist/Deskewing.html}) straightens the digits that are written in a crooked manner. Mathematically, skewing is modeled as an affine transformation: $Image^{'} = A(Image)+b$, in which the center of mass of the image is computed to estimate how much offset is needed, and the covariance matrix is estimated to approximate by how much an image is skewed. Furthermore, the center and covariance matrix are employed for the inverse affine transformation, which is referred to as deskewing. Then, we used t-SNE \citep{van2008visualizing} to reduce the dimension of the MNIST from $28\times 28$ to $2$, as the two-dimensional embedding space. Figure \ref{Effect_Comparison} presents the effect of deskewing on t-SNE. 

\begin{figure}[htb]
\center{\includegraphics[height=3.0in,width=6.0in,scale=0.4] {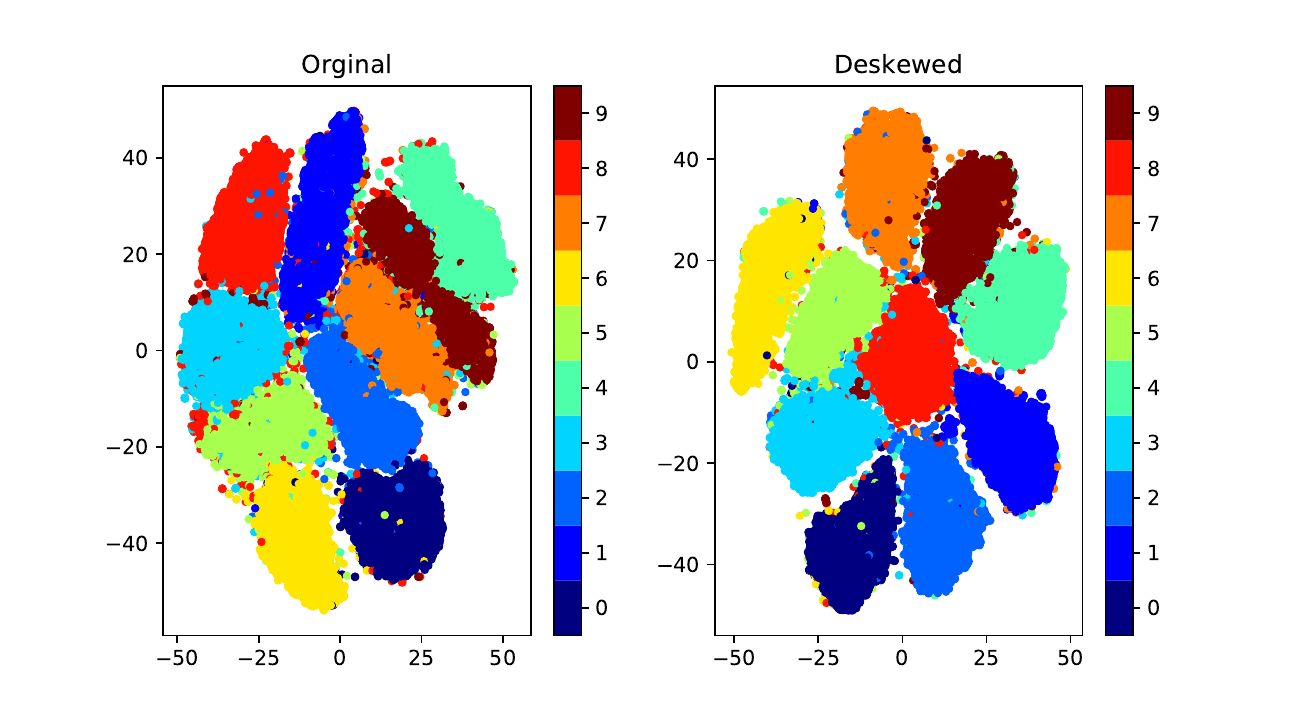}}
\caption{The effect of the deskewing operation on the t-SNE result of the MNIST data. Deskewing significantly improves the embedding quality.}
\label{Effect_Comparison}
\end{figure}

The embedding data can be naturally divided into the training and testing sets based on the partition in the original space. To illustrate the aforementioned quasi-equivalence between quadratic networks, We built three representative quadratic deep learning models, using
two wide sub-networks (by the factorization representation),
four deep sub-networks with the reciprocal activation (by the continued fraction), and 
four deep sub-networks using the ReLU activation (representing the first-order approximation to the continued fraction representation).
Each wide sub-network takes one element of the input, and 
two paired deep sub-networks take the same element of the input.
Each wide sub-network has four layers with $8,4,2,1$ neurons sequentially, while each deep sub-network also has four layers with one neuron in each layer. These sub-networks are followed by a fully connected ReLU layer of the same structure. The fully connected network has three layers with $300$, $200$, and $10$ neurons respectively. 

We performed the batch training with $1,000$ instances per batch in each iteration. All the parameters were initialized with truncated Gaussian distribution of a mean $0$ and a variance $0.1$. The learning rate was set to $0.002$, and the whole network was optimized with Adam \citep{kingma2015adam}. The number of epochs was $200$.

\textbf{Feasibility of the Reciprocal Activation:} At the first glance, the training of a deep model using the reciprocal activation function is subject to instability. For example, when the value of the input is large, $\frac{b_N x^2}{1+b_N x^2}$ gets closer to $1$, which will lead to zero in the denominator in the next fraction, and undermine the training convergence. However, we argue that when data are appropriately normalized, the training process can be made stable. We did the following experiment to study how the normalization can help.
Specifically, the input was normalized by the formula: $x_{new} = \frac{x_{max}}{\zeta} \Big(\frac{x-x_{min}}{x_{max}-x_{min}}\Big)$, where $\zeta$ is a scaling factor. We evaluated four scaling factors $\{1,4,8,16\}$ by repeating $20$ times the training process for each factor. The criterion is whether the training process converges or not. In this experiment, as long as 'nan' does not appear, the training will converge because we used a sufficiently large number of epochs. Table \ref{tab:feasibility} shows the success rate of training for different $\zeta$ scales. It can be seen that without scaling the training process always fails, while the success rate reaches $80\%$ when $\zeta=16$. The point is that the training process can be stablized with an appropriate scaling operation. 

\begin{table}[htb]
 \centering
\caption{The success rates of the training process with reciprocal activation after normalization.}
 \begin{tabular}{ |c|c|c|c| c|c|  }
      \hline
     $\zeta$ & 1  & 4 & 8 & 16 \\
   \hline
     Success Rate & 0 &  60\% & 55.5\% & 80\% \\
   \hline   
 \end{tabular}
 \label{tab:feasibility}
\end{table}

\textbf{Accuracy and Robustness of the Three Models:} We repeated the training process $10$ times and computed the accuracy of the three models on the test dataset. As shown in Table \ref{tab:ModelPerformance}, all the three models achieved the state-of-the-art results. The performance of the deep model using ReLU was only slightly lower than that of the other two models.

\begin{table}[htb]
 \centering
\caption{Performance of the three models on the test dataset after $10$ repetitions. }
 \begin{tabular}{ |c|c|c|c|  }
      \hline
     Model & Wide Model  & \makecell{Deep Model \\ (Reciprocal Activation)} & \makecell{Deep Model \\ (ReLU Activation)}  \\
   \hline
     Accuracy & $0.9813 \pm 0.000582$ &  $0.9810 \pm 0.000249$ & $0.9804 \pm 0.000513$ \\
   \hline   
 \end{tabular}
 \label{tab:ModelPerformance}
\end{table}

Furthermore, we used the following four popular adversarial attack methods to evaluate the robustness of the deep learning models: (1) fast gradient method (FGM); (2) fast sign gradient method (FSGM \citep{goodfellow2014explaining}); (3) iterative fast sign gradient method (I-FSGM \citep{kurakin2016adversarial}); and (4) DeepFool \citep{moosavi2016deepfool}. Let $\mathbf{\theta}$ denote the model parameters, $\x$ be the input, $y$ be the target pertaining to $\x$ and $L(\mathbf{\theta},\x,y)$ the loss, FGM generates an adversary as
\begin{equation}
    \x^{adv} = \x + \epsilon \cdot (\nabla_\x L(\mathbf{\theta},\x,y)),
\end{equation}
where $\epsilon$ is an amplitude factor. FGM is plausible to find the attack along the direction of the gradient. FSGM computes an adversary based on 
\begin{equation}
    \x^{adv} = \x + \epsilon \cdot sign(\nabla_\x L(\mathbf{\theta},\x,y)),
\end{equation}
where $\epsilon$ is also a factor, and $sign(\cdot)$ outputs $1$ for a positive argument and $-1$ otherwise. I-FSGM iteratively derives an adversary with the FSGM formula:
\begin{equation}
\begin{cases}
    & \x_0^{adv} = \x \\
    & \x_{k+1}^{adv} =  Clip_{\x, \alpha} \Big(\x_{k}^{adv} + \epsilon \cdot sign(\nabla_\x L(\mathbf{\theta},\x_{k}^{adv},y)) \Big),
\end{cases}
\end{equation}
where $Clip_{\x, \alpha}(\cdot)$ is a thresholding function such that the maximum value of $\x$ is no more than a preset limit $\alpha$. In the experiment, we set this limit $\alpha$ to the maximum value of the input. DeepFool calculates an adversary by
\begin{equation}
\begin{cases}
    & \x_0^{adv} = \x \\
    & \x_{k+1}^{adv} = \x_{k}^{adv} + \frac{L(\mathbf{\theta},\x_{k}^{adv},y)}{||\nabla_\x L(\mathbf{\theta},\x_{k}^{adv},y))||_2^2} \nabla_\x L(\mathbf{\theta},\x_{k}^{adv},y)).
\end{cases}
\end{equation}

Table \ref{tab:AdvPerformance} shows the performance of the three models on the adversarial samples generated based on the test dataset using the aforementioned attacking methods. We bold-faced the best scores under attack. It can be seen that the wide model shows the highest accuracy under most attacks. The only exceptions are the cases of FSGM($\epsilon=3$) and FSGM($\epsilon=5$), where the wide model is the second best.

\begin{table}[htb]
 \centering
\caption{Accuracy of the three models under attacks, where $\epsilon$ is the amplitude factor and $N$ is the number of iterations.}
 \begin{tabular}{ |c|c|c|c| c|  }
      \hline
     Attacks & Wide Network & Deep Network & Deep Network (ReLU) \\
     \hline
     FGM ($\epsilon=200$) & \textbf{0.9562 $\pm$ 0.0025} & 0.9526 $\pm$ 0.0043 & 0.9524 $\pm$ 0.0061 \\
     \hline
     FGM ($\epsilon=500$) & \textbf{0.9346 $\pm$ 0.0049} & 0.9263 $\pm$ 0.0054 & 0.9226 $\pm$ 0.0061 \\     
   \hline
     FGM ($\epsilon=1000$) & \textbf{0.9114 $\pm$ 0.0076} & 0.8963 $\pm$ 0.0077 & 0.8909 $\pm$ 0.0094 \\ 
   \hline   
     FSGM ($\epsilon=2$) & \textbf{0.4869 $\pm$ 0.0252} & 0.4642 $\pm$ 0.0205 & 0.4835 $\pm$ 0.0132 \\
   \hline   
     FSGM ($\epsilon=3$) & 0.2811 $\pm$ 0.0303 & 0.2349 $\pm$ 0.0245 & \textbf{0.2840 $\pm$ 0.0173} \\
   \hline  
     FSGM ($\epsilon=5$) & 0.1005 $\pm$ 0.0244 & 0.0809 $\pm$ 0.0194 & \textbf{0.1306 $\pm$ 0.0184} \\
    \hline   
     I-FSGM ($\epsilon=0.01, N=20$) & \textbf{0.4750 $\pm$ 0.0152} & 0.4614 $\pm$ 0.0167 & 0.4711 $\pm$ 0.0127 \\
   \hline   
     I-FSGM ($\epsilon=0.01, N=30$) & \textbf{0.2820 $\pm$ 0.0312} & 0.2490 $\pm$ 0.0304 & 0.2646 $\pm$0.0224\\
   \hline  
     I-FSGM ($\epsilon=0.01, N=50$) & \textbf{0.1416 $\pm$ 0.0516} & 0.1138 $\pm$ 0.0292 & 0.1131 $\pm$ 0.0403\\
   \hline   
      DeepFool ($N=30$) & \textbf{0.9225 $\pm$ 0.0069} & 0.9125 $\pm$ 0.0111 & 0.9064 $\pm$ 0.0102 \\
   \hline    
      DeepFool ($N=50$) & \textbf{0.8504 $\pm$ 0.0184} & 0.7516 $\pm$ 0.2305 & 0.8080 $\pm$ 0.0236 \\
   \hline  
      DeepFool ($N=60$) & \textbf{0.8213 $\pm$ 0.0255} & 0.6458 $\pm$ 0.2895 & 0.7612 $\pm$ 0.0337 \\
   \hline  
   
 \end{tabular}
 \label{tab:AdvPerformance}
\end{table}

Next, we probed the robustness of each model by examining the needed strength of the adversarial attack such that the network performance will drop by a pre-specified percentage. Using the absolute performance drop as a reference is better than using the performance compromised by adversarial attacks, as it eliminates the bias induced by the original network performance. Suppose that the original performance (without attack) is $O$, with each attacking method, we reduced $O$ to $O'=O-5\%$ by gradually increasing the strength of attack. The reason to select a $5\%$ drop is that if the drop is too high the comparison is not sensitive, while if the drop is too low the attack is too weak. Specifically, we gradually increased the strength by a fixed step for each attacking method until the performance drop is over $5\%$. For FGM, FSGM, I-FSGM and DeepFool, the steps are 50, 1, 1 and 1 respectively. Table \ref{tab:ChangingEpsilon} shows the needed strength values. The higher the needed strength, the more robust the model is. We bold-faced the best scores. Overall, the wide model shows the highest robustness. The only exception is the case of I-FSGM, where the wide model is the second best. These data suggest that the wide model seems more robust than the deep model. Further studies are needed to understand the underlying mechanisms and derive practical guidelines.

\begin{table}[htb]
 \centering
\caption{Needed strength values for each of the four attacking methods to reduce the unattacked network performance by $5\%$}
 \begin{tabular}{ |c|c|c|c| c|  }
      \hline
     Attacks & Wide Network & Deep Network & Deep Network (ReLU) \\
     \hline
     FGM ($\epsilon$) & \textbf{547.5 $\pm$ 65.8} & 447.4 $\pm$ 71.6 & 512.5 $\pm$ 95.8 \\
   \hline   
     FSGM ($\epsilon$) & \textbf{60.174 $\pm$ 1.267}  & 59.464 $\pm$ 1.895 & 59.783 $\pm$ 1.594\\

    \hline   
     I-FSGM ($\epsilon, N=20$) & 6.20 $\pm$ 0.41  & 6.15 $\pm$ 0.58 & \textbf{6.50 $\pm$ 0.51}\\

   \hline   
      DeepFool ($N$) & \textbf{26.95 $\pm$ 1.60}  & 25.05 $\pm$ 1.76 & 26.85 $\pm$ 2.25\\
   \hline
   
 \end{tabular}
 \label{tab:ChangingEpsilon}
\end{table}

\clearpage

\newpage
\addcontentsline{toc}{section}{V. Bounds for the Partially Separable Representation}
\subsection*{V. Bounds for the Partially Separable Representation}
The purpose of this section is to show that the partially separable function can be realized a quadratic network whose structure is bounded. Let us first introduce necessary preliminaries.

\textbf{Partially Separable Representation:} Approximating a multivariate function $f(x_1,x_2,\cdots,x_n)$ by a set of functions of fewer variables is a basic problem in approximation theory \citep{light2006approximation}. Despite that some function $f$ is directly separable in the form of
\begin{equation}
    f(x_1,x_2,\cdots,x_n) = \phi_1(x_1)\phi_2(x_2)\cdots \phi_n(x_n),
\end{equation}
a more general formulation to express a multivariate function $f$ is 
\begin{equation}
    f(x_1,x_2,\cdots,x_n) = \sum_{l=1}^L \prod_i^n \phi_{li}(x_{i}).
\end{equation}
If $L$ is permitted to be sufficiently large, then such a model is rather universal. For bivariate functions, an inspiring theorem has been proved \citep{light2006approximation}: Let $\{u_n\}_{n\in \mathcal{N}}$ and $\{v_n\}_{n\in \mathcal{N}}$ are orthornomal bases of $L^2 (\textbf{X})$ and $L^2 (\textbf{Y})$ respectively, then $\{u_m v_n\}_{(m,n) \in \mathcal{N}^2}$ is an orthornomal basis of $L^2 (\textbf{X} \times \textbf{Y})$.

To approximate a general multivariate function, we relax the restrictive equality to the approximation in the $L_1$ sense and assume that, for every continuous $n$-variable function $f(x_1,\cdots,x_n)$ on $[0,1]^n$, given any positive number $\epsilon$, there exists a group of $\phi_{li}$, satisfying:
\begin{equation}
    \int_{{(x_1,\cdots,x_n)\in [0,1]^n}}  |f(x_1,\cdots,x_n) - \sum_{l=1}^L \prod_{i=1}^n \phi_{li}(x_{i})| < \epsilon.
\end{equation}

The key to demonstrate boundedness of the partially separable representation with a quadratic network is to use the Taylor's expansion, and estimate the remainder, leading to a realization of the partially separable representation. Based on such a realization, we obtain an upper bound of the needed width and depth for the partially separable representation. Let $\partial^{\bm{\alpha}} f(\x) = \frac{\partial^{\alpha_1}}{\partial x_1} \frac{\partial^{\alpha_2}}{\partial x_2}\cdots \frac{\partial^{\alpha_n}}{\partial x_n} f(\x)$, $\bm{\alpha}! = \prod_{i=1}^n \alpha_i !$, and $\x^{\bm{\alpha}}=x_1^{\alpha_1}\cdots x_n^{\alpha_n}$, for the function class $\mathcal{F}^{n,k}=\{f\in C^{k+1}([0,~1]^n) ~|~ \|\partial^{\bm{\alpha}} f \|_{\infty} \leq M, \quad \forall  |\bm{\alpha}|\leq k \}$,
the following theorem holds. 
\begin{theorem}
For any $f \in \mathcal{F}^{n,k}$, if $\x \in [0,1]^n$, there exists a quadratic network $\mathcal{Q}(\x)$ with the width no more than $k{n+k \choose n}$ and the depth no more than $\log_2 {(kn)}$ to represent $f$ in the partially separable representation, satisfying that 
\begin{equation}
    \sup_{\x \in [0,1]^n} |f(\x)-\mathcal{Q}(\x)| \leq \frac{M}{(k+1)!}n^{k+1}.
\end{equation}
\label{thm:bound}
\end{theorem}

\begin{proof}Let $f \in \mathcal{F}^{n,k}$, if $\x \in [0,1]^n$, then based on classical multivariate calculus \citep{widder1989advanced},
\begin{equation}
\label{taylor1}
    f(\x) = \sum_{|\bm{\alpha}| \leq k} \frac{\partial^{\bm{\alpha}}}{\bm{\alpha}!} f(\bm{0}) \x^{\bm{\alpha}}+R_{\x,k}(\x),
\end{equation}
where $R_{\x,k}(\x)$ is the remainder term. For any $\x \in [0,1]^n$, we have
\begin{equation}
\label{taylor2}
    R_{\x,k}(\x) \leq \frac{M}{(k+1)!}(|x_1| + |x_2| + \cdots + |x_n|)^{k+1} \leq \frac{M}{(k+1)!}n^{k+1}.
\end{equation}
It is observed from Eqs. \ref{taylor1} and \ref{taylor2} that, given $f \in \mathcal{F}^{n,k}$, the Taylor expansion forms a partially separable representation $\sum_{|\bm{\alpha}| \leq k} \frac{\partial^{\bm{\alpha}}}{\bm{\alpha}!} f(\bm{0}) \x^{\bm{\alpha}}$ with ${n+k \choose n}$ products, and the error is no more than $\frac{M}{(k+1)!}n^{k+1}$. Along this line, we apply the quadratic network to express $\sum_{|\bm{\alpha}| \leq k} \frac{\partial^{\bm{\alpha}}}{\bm{\alpha}!} f(\bm{0}) \bm{x}^{\bm{\alpha}}$. Any univariate polynomial of degree $N$ can be computed by a quadratic network using the ReLU function with the depth of $\log_2(N)$ and the width of $N$. To approximate a general term $\frac{\partial^{\bm{\alpha}}}{\bm{\alpha}!} f(\bm{0}) \bm{x}^{\bm{\alpha}}=\frac{\partial^{\bm{\alpha}}}{\bm{\alpha}!} f(\bm{0}) x_1^{\alpha_1}\cdots x_n^{\alpha_n}$, the depth of the quadratic network is no more than $\log_2({\max\{\alpha_1,\cdots,\alpha_n\}})\leq \log_2 {k}$,  and the width is no more than $\sum_{i=1}^n \alpha_i = k$. In addition, the depth of a network that is used to express the product of $n$ terms is  $\log_2 {n}$. Therefore, the depth is no more than $\log_2 {k} + \log_2 {n}=\log_2 {(kn)}$. In contrast, the width is no more than $k{n+k \choose n}$, considering that there are ${n+k \choose n}$ terms in $\sum_{|\bm{\alpha}| \leq k} \frac{\partial^{\bm{\alpha}}}{\bm{\alpha}!} f(\bm{x}) \bm{x}^{\bm{\alpha}}$. 
\end{proof}

\newpage
\addcontentsline{toc}{section}{VI. Properties of the Partially Separable Representation}
\subsection*{VI. Properties of the Partially Separable Representation}
Let us compare the properties of the partially separable representation with those of the piecewise polynomial representation and the Kolmogorov–Arnold representation in terms of universality, finiteness, decomposability and smoothness. As a result, the partially separable representation is well justified.

\textbf{1. Universality:} The representation should have a sufficient ability to express the functions. A piecewise polynomial representation is capable of representing any continuous function in a local way. The Kolmogorov–Arnold representation theorem states that every multivariate continuous function can be represented as a combination of continuous univariate functions, which solves a generalized Hilbert thirteen problem. Therefore, the Kolmogorov–Arnold representation also possesses universality. As far as the partially separable representation is concerned, its universality in the $L_1$ distance is analyzed in the above Taylor expansion analysis. In summary, these three representation classes are all powerful enough to express continuous functions. 

\textbf{2. Finiteness:}  The multivariate Taylor expansion can approximate a general continuous function $f \in \mathcal{F}^{n,k}$ with a finite number of terms. Although the proof in the manuscript devises a specific strategy of obtaining a partially separable representation, the partially separable representation actually allows a variety of constructions in addition to the Taylor expansion. There is no fundamental reason that the partially separable representation will necessarily have exponentially many terms in a majority of practical tasks after reasonable assumptions are incorporated.

\textbf{3. Decomposability:} Decomposability is to decode the functionality of the model in terms of its building units, i.e., cells, layers, blocks, and so on. A plethora of engineering examples, such as software development and optical system design, have shown that modularized analysis is effective. Particularly, modularizing a neural network is advantageous for the model interpretability. For example, \cite{chen2016infogan} developed InfoGAN to enhance decomposability of features learned by GAN \citep{goodfellow2014generative}, where InfoGAN maximizes the mutual information between the latent code and the observations, encouraging the use of noise to encode the semantic concept. 
Since the partially separable representation is more decomposable than its counterparts, the network constructed in reference to the partially separable representation is easily interpretable. For instance, such a network has simpler partial derivatives that can simplify interpretability analysis \citep{lipton2018mythos}. 

\textbf{4. Smoothness:} A useful representation should be as smooth as possible such that the approximation can suppress high-frequency oscillations and be robust to noise. In the last eighties, Girosi and Poggio claimed that the use of the Kolmogorov-Arnold representation is doomed because the inner functions and the outer functions are highly non-smooth \citep{girosi1989representation}. In the piecewise polynomial representation, the situation gets better since at least over each interval the representation is smooth. As far as the partially separable representation is concerned, because the expression is greatly relaxed (from the exact computation to approximate representation), it is feasible to make each element smooth as $L$ goes large, and extract meaningful structures underlying big data. 
\begin{table}[htb]
 \centering
\caption{Comparison of the three functional representations, where "\checkmark" means "yes", "\text{\sffamily X}" means "no", and "--" means "mediocre".}
 \begin{tabular}{ |c|c|c|c| c|  }
      \hline
   \textbf{Representation} & \textbf{Universality} & \textbf {Finiteness}& \textbf{Decomposability} & \textbf{Smoothness}\\
   \hline
   \textbf{Piecewise Polynomial}  &  \checkmark & \text{\sffamily X} &   \text{\sffamily X} & --\\
   \hline
   \textbf{Kolmogorov-Arnold} & \checkmark & \checkmark& -- & \text{\sffamily X}\\
   \hline
   \textbf{Partially Separable} & \checkmark & -- & \checkmark & \checkmark  \\
   \hline
 \end{tabular}
 \label{tab:ps}
\end{table}

In summary, we compare the three representation schemes in terms of universality, finiteness, decomposability, and smoothness in \textbf{Table \ref{tab:ps}}. Clearly, the partially separable representation is suitable for machine learning tasks.

\newpage
\addcontentsline{toc}{section}{VII. Width-Depth Correlation Through Partially Separable Presentation}
\subsection*{VII. Width-Depth Correlation Through Partially Separable Presentation}


Now, let us analyze the complexity of the quadratic network in light of the aforementioned partially separate representation scheme as shown in \textbf{Figure \ref{fig:complexity}}. Suppose that the polynomial $P_{li}$ is of degree $N_{li}$, then the representation of each $P_{li}$ can be done with a network of width $\sum_{l=1}^{L} \sum_{i=1}^{n} N_{li}$ and depth $\max_{l,i}\{\log_2(N_{li})\}$. Next, the multiplication demands an additional network of width $Ln$ and depth $\log_2(n)$. Therefore, the overall quadratic network architecture will be of width $ \max \{\sum_{l=1}^{L} \sum_{i=1}^{n} N_{li}, Ln\}$ and depth $\max_{l,i}\{\log_2(N_{li})\}+\log_2(n)$. Because the depth scales with a log function, which changes slowly when the dimensionality of the input is large. For simplicity, we take an approximation for depth $\max_{l,i}\{\log_2(N_{li})\}+\log_2(n) = \log_2(\max_{l,i}\{N_{li}\})+\log_2(n) \approx \alpha \log_2 \sum_{l=1}^{L} \sum_{i=1}^{n} N_{li}+\log_2(n)$, where $\alpha$ is a positive constant. Let $\sum_{l=1}^{L} \sum_{i=1}^{n} N_{li}= N^{\Sigma}$, which describes the overall complexity of the function to be expressed, then the formulas to compute the width and depth are simplified as follows:

\begin{equation}
\begin{aligned}
    & Width = \max\{N^{\Sigma},Ln\}  \\
    & Depth = \alpha log_{2}(N^{\Sigma}) + log_{2}(n).
\end{aligned}
\label{eq:WDrelation}
\end{equation}

\begin{figure}[htb]
\center{\includegraphics[height=3in,width=3.5in,scale=0.4] {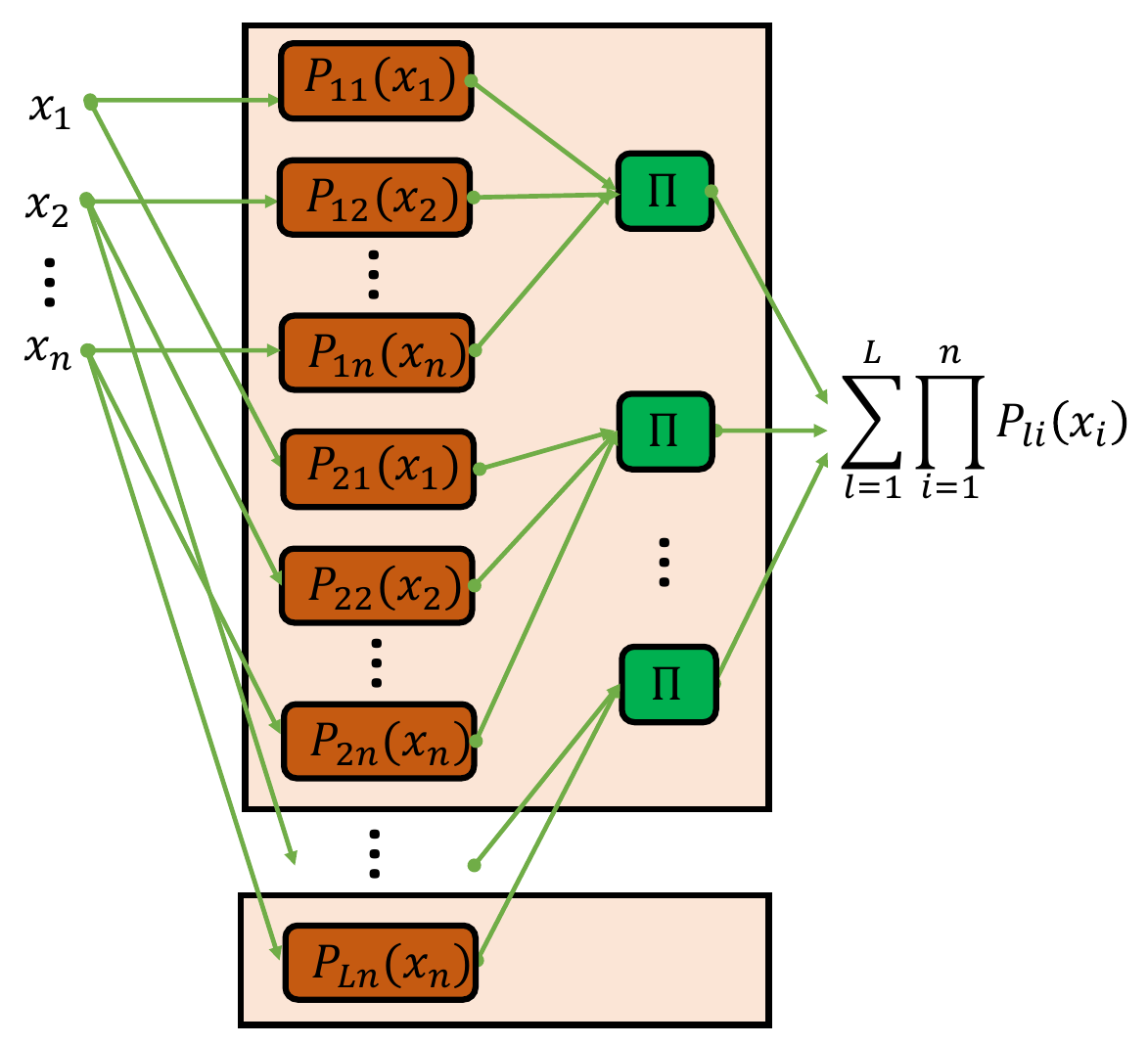}}
\caption{Use of a quadratic network to represent a partially separable representation. Suppose that the polynomial $P_{li}$ is of degree $N_{li}$, the width and depth of the quadratic network to approximate $P_{li}$ are $N_{li}$ and $\log_2(N_{li})$ respectively.}
\label{fig:complexity}
\end{figure}

\begin{figure}[htb]
\center{\includegraphics[height=2.2in,width=3.0in,scale=0.4] {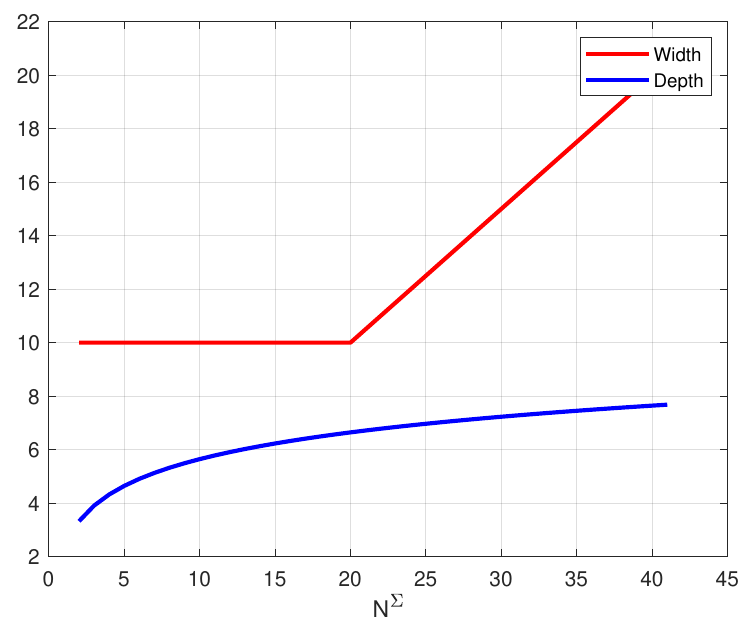}}
\caption{Width and depth versus $N^{\Sigma}$ changes ($L=4$, $n=5$, and $\alpha=1$ without loss of generality) assuming the partially separable representation.}
\label{fig:curve}
\end{figure}

One interesting point from \eqref{eq:WDrelation} is that the lower bounds for depth and width to realize a partially separable representation are also suggested. As shown in \textbf{Figure \ref{fig:curve}}, we plot the width and depth as $N^{\Sigma}$ changes. There are two highlights in \textbf{Figure \ref{fig:curve}}. The first is that the width is generally larger than the depth, which is different from the superficial impression on deep learning. The second is that, as the $N^{\Sigma}$ goes up, the width/depth ratio is accordingly increased. 

\begin{table}[htb]
\centering
\caption{Descriptions of different building blocks.} 
 \begin{tabular}{||c c c c c||} 
 \hline
 Modules & Degree & Operation & Width & Depth \\ [0.5ex] 
 \hline\hline
 $P_{li}$ & $N_{li}$ & Express $\phi_{li}$ & $N_{li}$ & $\log_2(N_{li})$ \\
 $\Pi$ & $-$ & Express $\Phi_i$ & $n$ & $\log_2(n)$ \\ [1ex] 
 \hline
 \end{tabular}
\end{table}

\textbf{Remark:} Through the complexity analysis, we realize that the width and depth of a network depend on the structure or complexity of the function to be approximated. In other words, they are controlled by the nature of a specific task. As the task becomes complicated, the width and depth must increase accordingly, and the combination of the width and depth is not unique. 
\clearpage

\newpage
\addcontentsline{toc}{section}{VIII. Effects of Width on Optimization and Generalization}
\subsection*{VIII. Effects of Width on Optimization and Generalization}
We first illustrate the importance of width on optimization in the context of over-paramterization, kernel ridge regression, and NTK, and then report our findings that the existing generalization bounds and VC dimension results somehow suggest the width and depth equivalence for a given complexity of networks. 

\textbf{1. Optimization:} The optimization mechanism is the key to understand the training process of neural networks \citep{ge2015escaping}. From the view of optimization, the fact that randomly-initialized first-order methods can find well-generalizable minima on the training data is quite intriguing. Given the theme of this paper, we put our endeavor into the importance of width on the optimization of neural networks in hope to provide insights for practitioners. We divide the relevant literature into the three categories: 

(1) Increase width for over-parameterization: Brutzkus et al. showed that a wide two-layer network using the hinge loss can generalize well on linearly separable data with stochastic gradient descent (SGD) \citep{brutzkus2017sgd}. Li and Liang showed that when data is normalized and fulfills a separability condition, a two-layer over-parameterized network can learn these data in a polynomial time \citep{li2018learning}. \cite{allen2019learning} showed that if the width of hidden layers is $\mathcal{O}\Big(n^{30} D^{30} \log^{30} (1/\epsilon) \Big)$, where $n$ is the number of samples, $D$ is the network depth, and $\epsilon$ is an expected error, then the gradient descent search converges with the error $\epsilon$. This bound was further reduced to $n^4 2^{D}$ for the network using non-linear smooth activation functions such as soft-plus \citep{du2018gradient}. These paper established that a pre-specified small training error can be achieved by gradient descent when the network is very wide. The secret therein is that the weight matrix is very close to the initialization due to NTK \citep{jacot2018neural}.

(2) Kernel ridge regression: A neural network tends to give a Gaussian process when the width goes infinitely large. In this situation, only training the top layer of a network (weak training) will reduce to Kernel ridge regression. In a classical way, the weak training is to minimize the quadratic loss:
\begin{equation}
    C(\mathbf{w}) = (\mathbf{y-Xw})^T (\mathbf{y-Xw}) + \lambda ||\mathbf{w}||^2,
\end{equation}
where $\mathbf{X}=[\mathbf{x}_1;...;\mathbf{x}_n]$ is a collection of data $\mathbf{x}_i \in \mathbb{R}^{1\times d}, i=1,...,n$. 
Taking derivatives with respect to $\bm w$ and equating them to zero gives
\begin{equation}
    \mathbf{w} = (\mathbf{X}^T \mathbf{X}+\lambda \mathbf{I}_d)^{-1} \mathbf{X}^T \mathbf{y} = \mathbf{X}^T (\lambda \mathbf{I}_n+\mathbf{X} \mathbf{X}^T)^{-1}  \mathbf{y}. 
\end{equation}
Given a new example $x^*$, the prediction is
\begin{equation}
   y^* = x^* \mathbf{X}^T (\lambda \mathbf{I}_n+\mathbf{X} \mathbf{X}^T)^{-1}  \mathbf{y} .
\end{equation}
We replace the data samples with the high-dimensional representations of the neural network: $\mathbf{x}\to g_\mathbf{\theta}(\mathbf{x}) \in \mathbb{R}^{1\times k}$ and let $\Phi_g (\mathbf{X}) = [g_\mathbf{\theta}(\mathbf{x}_1);...;g_\mathbf{\theta}(\mathbf{x}_n)]$. By a similar derivation, we have
\begin{equation}
     y^* = g_\theta (x^*) \Phi_g (\mathbf{X})^T (\lambda \mathbf{I}_n+\Phi_g (\mathbf{X}) \Phi_g (\mathbf{X})^T)^{-1}  \mathbf{y} .  
\end{equation}
When the width is very large, due to the Gaussian process, the weak training is reasonable, because $\Phi_g (\mathbf{X})$ converges to $\underset{\bm{\theta}\sim \bm{\Theta}}{\mathbb{E}} [\Phi_g (\mathbf{X})]$, and $\Phi_g (\mathbf{X}) \Phi_g (\mathbf{X})^T$ converges to $\underset{\bm{\theta}\sim \bm{\Theta}}{\mathbb{E}} [\Phi_g (\mathbf{X}) \Phi_g (\mathbf{X})^T]$, which can be approximated by the random initialization. 

(3) Neural Tangent Kernel: In the above argument, we have justified the legitimacy of training the top layer by the large network width. How about training the entire network? Given the dataset $\{(\mathbf{x}_i, y_i) \}_{i=1}^n$, we also consider training the network with the quadratic loss for regression: $l(\bm{\theta}) = \frac{1}{2}\sum_{i=1}^n (f_{\bm{\theta}}(\mathbf{x}_i))-y_i)^2$. Consider the gradient descent in an infinitesimally small learning rate, then 
\begin{equation}
    \frac{d \bm{\mathbf{\theta}}(t)}{dt} = -\nabla_{\bm{\mathbf{\theta}}}l(\bm{\theta}(t)) = -\sum_i^n (f_{\bm{\theta}(t)}(\mathbf{x}_i))-y_i) \nabla_{\bm{\mathbf{\theta}}}f_{\bm{\theta}(t)}(\mathbf{x}_i).
\end{equation}
Next, we can describe the dynamics of the model output $y({\bm{\theta}})$ given the input $\mathbf{x}_j$ as follows:
\begin{equation}
\begin{aligned}
    \frac{d f_{{\bm{\theta}}(t)}(\mathbf{x}_j)}{dt} & = \nabla_{\bm{\mathbf{\theta}}}f_{{\bm{\theta}}(t)}(\mathbf{x}_j)^T \frac{d \bm{\mathbf{\theta}}}{dt} \\
    & =-\sum_i^n (f_{\bm{\theta}(t)}(\mathbf{x}_i))-y_i) \Big \langle \nabla_{\bm{\mathbf{\theta}}}  f_{{\bm{\theta}}(t)}(\mathbf{x}_j), \nabla_{\bm{\mathbf{\theta}}}f_{\bm{\theta}(t)}(\mathbf{x}_i) \Big \rangle.
\end{aligned}  
\label{ntkopi}
\end{equation}
Let us consider $\bm{u}(t) = (f_{\bm{\theta}(t)}(\mathbf{x}_i))_{i=1,...,n}$ for all samples at time $t$ and $\bm{y}=(y_i)_{i=1,...,n}$ is the output, \eqref{ntkopi} can be written in a compact way:
\begin{equation}
    \frac{d \bm{u}(t)}{dt} = - \mathbf{H}(t)\cdot (\bm{u}(t)-\bm{y}),
\end{equation}
where the $pq$-entry of $\mathbf{H}(t)$ is
\begin{equation}
    [\mathbf{H}(t)]_{pq} = \Big \langle \nabla_{\bm{\mathbf{\theta}}}  f_{{\bm{\theta}}(t)}(\mathbf{x}_p), \nabla_{\bm{\mathbf{\theta}}}f_{\bm{\theta}(t)}(\mathbf{x}_q) \Big \rangle.
\end{equation}
In the width limit, the Gaussian process shows that $\mathbf{H}(t)$ becomes a constant $\mathbf{H}^*$. Therefore, 
\begin{equation}
     \frac{d \bm{u}(t)}{dt} = - \mathbf{H}^*\cdot (\bm{u}(t)-\bm{y}),
\end{equation}
which characterizes the trajectory of the training in the functional space instead of the parameter space. 

\begin{table}[htb]
 \centering
\caption{Representative bounds for chain-like neural networks, where $B_{l,2}$, $B_{l,F}$ and $B_{l,2\to 1}$ to denote the upper bounds of the spectral norm $||\mathbf{W}_d||_2$, Frobenius norm $||\mathbf{W}_d||_F$, and $||\mathbf{W}_d||_{2,1}$ of the rank-$r$ matrix in $l^{th}$ layer. Generally, $||\mathbf{W}_d||_F$ and $||\mathbf{W}_d||_{2,1}$ is $\sqrt{r}$-times larger than $||\mathbf{W}_d||_2$, $\gamma$ means the ramp loss function is $1/\gamma$-Lipschitz function, $\Gamma$ is the lower bound for the product of the spectal norm of all the layers and the $m$ is the size of data.  }
 \begin{tabular}{ p{2cm}|p{7cm}| p{3cm} }
   \hline
      \hline
   \bf{Reference} &  \bf{Bound}  & $||\mathbf{W}_d||_2 = 1$\\

   \hline
   \cite{neyshabur2015norm}    &   $\mathcal{O}\Big(\frac{2^L\cdot \prod_{l=1}^{L}B_{l,F}}{\gamma \sqrt{m}}\Big)$ & $\mathcal{O}\Big( \frac{2^L \cdot r^{L/2}}{\gamma \sqrt{m}}\Big)$ \\
   \hline
   \cite{bartlett2017spectrally} & $\mathcal{O}\Big(\frac{\prod_{l=1}^{L}B_{l,F}\cdot \log(Lp)}{\gamma \sqrt{m}} \big(\sum_{l=1}^{L}\frac{B_{l,2\to 1}^{2/3}}{B_{l,2}^{2/3}} \big)^{3/2}\Big)$ & \bm{$\mathcal{O}\Big(\frac{\log(p) \sqrt{L^3 r}}{\gamma \sqrt{m}}\Big)$} \\
   \hline
   \cite{neyshabur2018pac} &  $\mathcal{O}\Big(\frac{\prod_{l=1}^{L}B_{l,F}\cdot \log(Lp)}{\gamma \sqrt{m}} \big(L^2 p \sum_{l=1}^{L}\frac{B_{l,F}^2}{B_{l,2}^2} \big)^{1/2}\Big)$ & \bm{$\mathcal{O}\Big(\frac{\log(Lp) \sqrt{L^3 pr}}{\gamma \sqrt{m}}\Big)$}\\
   \hline
   \cite{golowich2017size} & $\mathcal{O}\Big(\frac{\prod_{l=1}^{L}B_{l,F} }{\gamma} \cdot \min\{ \sqrt{\frac{\log{\frac{1}{\Gamma}}\prod_{l=1}^{L}B_{l,F}}{m^{1/2}}}, \sqrt{\frac{D}{m}}\} \Big)$ & \bm {$\mathcal{O}\Big(\frac{\sqrt{Lpr}}{\gamma \sqrt{m}}\Big)$} \\
   \hline
   \hline

 \end{tabular}
 \label{tab:genebounds}
\end{table}

\textbf{2. Generalization:} Analysis of generalization bounds is a powerful tool to explain the excellent performance of neural networks. Traditional wisdom suggests that the increased model complexity will cause over-fitting to  training data, which contradicts the fact that deep networks can easily fit random labels to the data and yet practically generalize well \citep{zhang2016understanding}. Recently, the generalization theory has gained increasingly more traction. In reference to \cite{li2018tighter}, we have summarized the the state-of-the-art generalization bounds in \textbf{Table \ref{tab:genebounds}} and provided their complexities. In \textbf{Table \ref{tab:genebounds}}, we bold-face the bounds of interest which the width and depth dominate. As far as \cite{neyshabur2015norm} is concerned, due to the exponential dependence on the depth, we will not focus on it. Instead, we argue that these bold-faced bounds somehow suggest the width and depth equivalence under a given complexity. Now, we use $B_1 (L,p) = \log(p) \sqrt{L^3}$, $B_2 (L,p) = \log(Lp) \sqrt{L^3 p}$, and $B_3 (L,p)= \sqrt{Lp}$ to denote the bounds of interest. \textbf{Figure \ref{fig:contour}} shows the contours plots of $B_1$, $B_2$ and $B_3$. Along a contour, the width and depth changes for the fixed bound complexity. In other words, the depth and width of a neural network are mutually transformable without impacting the overall generalization ability theoretically.

\begin{figure*}[h]
\center{\includegraphics[height=2.5in,width=6.0in,scale=0.4] {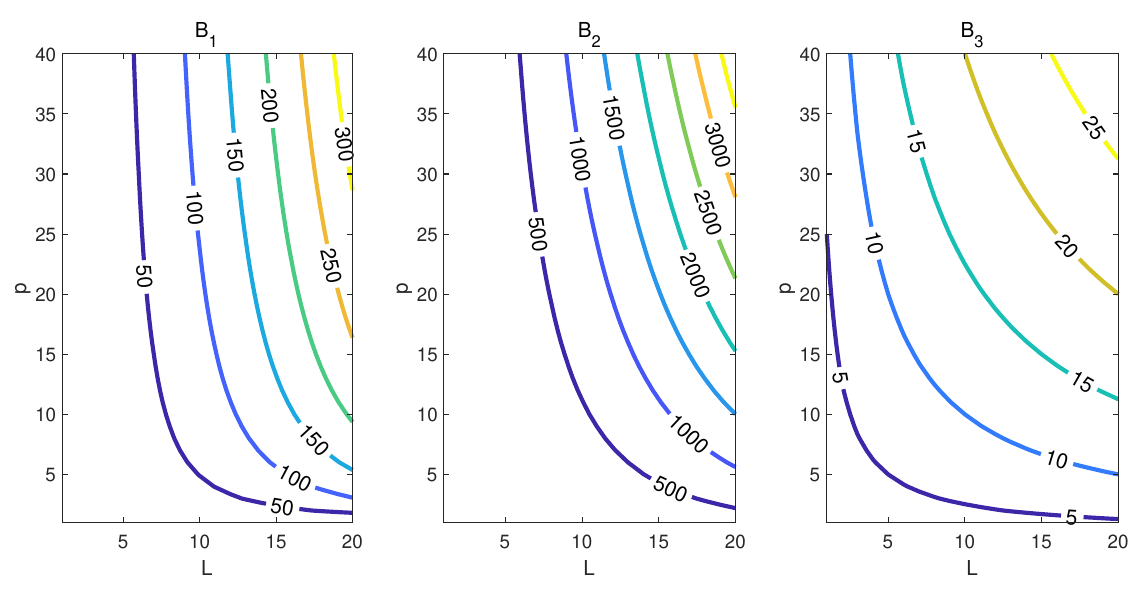}}
\caption{Contours of $B_1$, $B_2$ and $B_3$, where numbers on the curves of each sub-figure represent $\log(p) \sqrt{L^3}$, $\log(p) \sqrt{L^3}$ and $\sqrt{Lp}$ respectively. Along a contour, the width and depth changes to give the same bound.}
\label{fig:contour}
\end{figure*}

Therefore, increasing either width or depth can boost the hypothesis space of a neural network. In other words, when it comes to promoting the expressive power of a network, increasing the width is essentially equivalent to increasing the depth in the sense of VC dimension, which also implies the equivalence of the width and depth of neural networks.

\newpage
\addcontentsline{toc}{section}{IX. Rethinking the Depth Separation with Intra-layer Links }
\subsection*{IX. Rethinking the Depth Separation with Intra-layer Links }

Depth separation highlights the representation ability of a deep network, which has been intensively investigated over the past years. The idea is to show that a deep network can more efficiently approximate a complicated function than a shallow network. One notable depth separation result is from \cite{arora2016understanding}, showing that \textit{"for every pair of natural numbers  $k \geq 1, w \geq 2$, there exists a family of hard functions representable by a  $\mathbb{R} \rightarrow \mathbb{R}$ $(k+1)$-layer feedforward ReLU DNN of width  $w$  such that if it is also representable by a $\left(k^{\prime}+1\right)$-layer feedforward ReLU DNN for any $k^{\prime} \leq k$, then this $\left(k^{\prime}+1\right)$-layer feedforward ReLU DNN has size at least  $\frac{1}{2} k^{\prime} w^{\frac{k}{k^{\prime}}}-1$".} Suppose that the number of neurons, \textit{i.e.}, the size of a ReLU network is $s$, and the piecewise linear function of this ReLU network has $p$ pieces, the core of the depth separation theorem established in \cite{arora2016understanding} is summarized as the size-piece relationship: $s\geq \frac{1}{2} k(2p)^{1/k}-1$. 

In the proposed network structure for the width-depth conversion, intra-layer links are employed for a network to flexibly represent arbitrary piecewise linear functions over polytopes using fan-shaped functions. Here, we find that adding intra-layer links can greatly increase the maximum number of pieces represented by a shallow network such that it can express as a complicated function as a deep network could. As shown in Table \ref{tab:size-piece}, let $n$ be the number of intra-linked neurons, the size-piece inequality $s\geq \frac{n}{2(2^n-1)} k(2p)^{1/k}-1$ trivially holds true. As such, it cannot be used to demonstrate depth separation any more. Instead, we need to rethink the depth separation and re-characterize the power of a network when intra-layer links are used. Our detailed proof is as follows.

\begin{table}[htb]
 \centering
\caption{Network structures and their associated size-piece relationships. The number of neurons, i.e., the size of a ReLU network is $s$. $k$ is the number of hidden layers. The piecewise linear function of this ReLU network has $p$ pieces. $n$ is the number of intra-linked neurons.}
 \begin{tabular}{ |l|c|  }
      \hline
      Network & size-piece relationship \\
   \hline
 feedforward (Figure \ref{Figure_reduce_the_bound})(a) \citep{arora2016understanding} & $s\geq \frac{1}{2} k(2p)^{1/k}-1$ \\
 \hline
ours (Figure \ref{Figure_reduce_the_bound})(b) & $s\geq \frac{1}{3} k(2p)^{1/k}-\frac{2}{3}$ \\
   \hline
   extension (Figure \ref{Figure_reduce_the_bound})(c) & $s\geq \frac{n}{2(2^n-1)} k(2p)^{1/k}-\frac{n}{2^n-1}$ \\
   \hline
 \end{tabular}
 \label{tab:size-piece}
 \vspace{2mm}
 \\
 Note: In the first row, we re-calculate the inequality and obtain a slightly different inequality from the original. The previously-established inequality in \citep{arora2016understanding} is $s\geq \frac{1}{2} kp^{1/k}$.
\end{table}

\begin{figure*}[htb!]
\center{\includegraphics[width=0.8\linewidth] {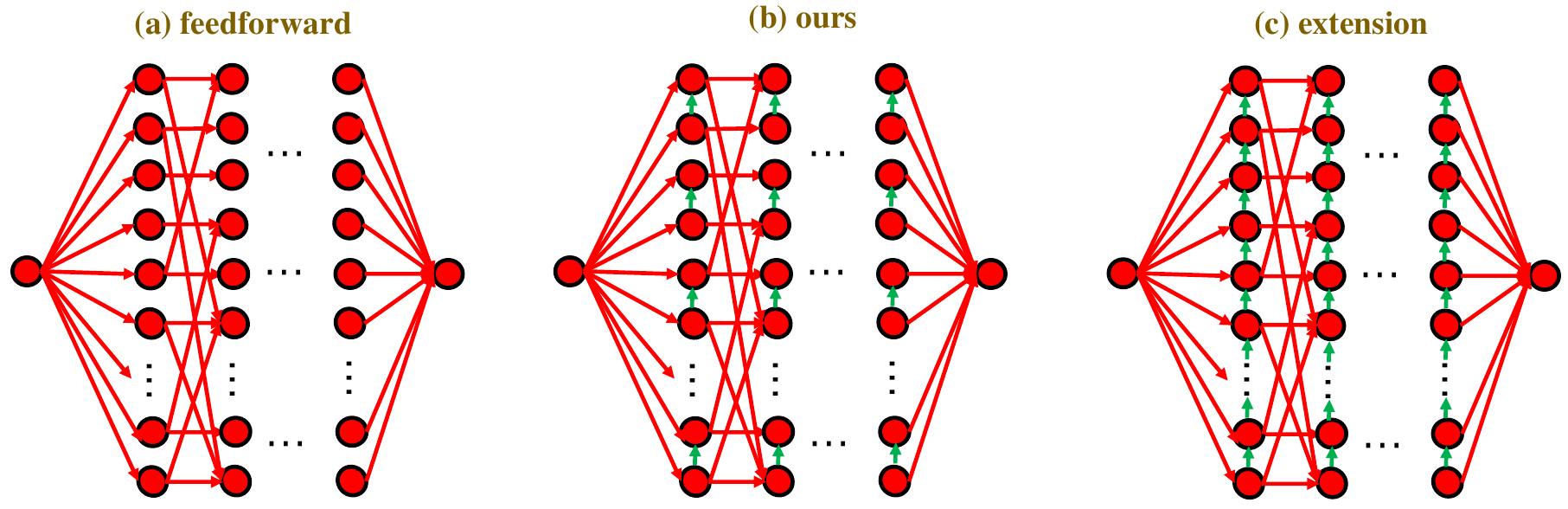}}
\caption{(a) the feedforward architecture; (b)-(c) the structures with intra-layer links.}
\label{Figure_reduce_the_bound}
\end{figure*}

The result listed in the first row of Table \ref{tab:size-piece} is obtained from the following two lemmas and proofs in \citep{arora2016understanding}. 
\begin{lemma}[Lemma D.5 in \citep{arora2016understanding}]
Let  $f: \mathbb{R}\rightarrow\mathbb{R}$ be a function represented by a $\mathbb{R}\rightarrow\mathbb{R}$  ReLU DNN, as shown in Figure \ref{Figure_reduce_the_bound}(a), with depth $k+1$ and widths $w_{1}, \ldots, w_{k} $ of $k$ hidden layers. Then $f$ is a PWL function with at most  $2^{k-1} \cdot\left(w_{1}+1\right) \cdot w_{2} \cdot \ldots \cdot w_{k}$  pieces.
\label{lem:D.5}
\end{lemma}

\begin{lemma}[Lemma D.6 in \citep{arora2016understanding}]
Let  $f: \mathbb{R}\rightarrow \mathbb{R}$  be a piecewise linear function with $p$ pieces. If $f$ is represented by a ReLU DNN, as shown in Figure \ref{Figure_reduce_the_bound}(a), with depth $k+1$, then it must have size at least $\frac{1}{2} k p^{1/k}-1$. Conversely, any piecewise linear function $f$ represented by a ReLU DNN of depth $k+1$ and size at most $s$, can have at most $\left(\frac{2 s}{k}\right)^{k}$ pieces.
\end{lemma}


Now, we prove the result listed in the second row of Table \ref{tab:size-piece} for ReLU DNN with intra-links. 
\begin{lemma}[Our 1st New Result]
Let  $f: \mathbb{R}\rightarrow\mathbb{R}$ be a function represented by a $\mathbb{R}\rightarrow\mathbb{R}$ ReLU DNN, as shown in Figure \ref{Figure_reduce_the_bound}(b), with depth  $k+1$  and widths  $w_{1}, \ldots, w_{k} $ of $k$ hidden layers. Then $f$ is a PWL function with at most $3^{k-1} \cdot\left(\frac{3w_{1}}{2}+1\right) \cdot w_{2} \cdot \ldots \cdot w_{k}$ pieces.
\label{lemma:piece_size_relation_ours}
\end{lemma} 

\begin{proof}
As shown in Figure \ref{Figure_reduce_the_bound}(b), let us connect every two neurons with an intra-layer link. Suppose that one neuron is $\sigma(wx+b)$ which can create 1 breakpoint at $-b/w$, while the other neuron is $\sigma(w'x+b'+\sigma(wx+b))$ creating at most 3 breakpoints $-b/w, -b'/w', -(b+b')/(w+w')$. Thus, we can get at most $\frac{w_{1}}{2}\cdot3=\frac{3w_{1}}{2}$ distinct breakpoints, i.e., $\frac{3w_{1}}{2}+1$ pieces.

Let us estimate the number of pieces via induction. Assume that for some $k\geq 1$, any $\mathbb{R} \rightarrow \mathbb{R}$  ReLU DNN, as shown in Figure \ref{Figure_reduce_the_bound}(b), with depth $k+1$ and widths  $w_{1}, \ldots, w_{k}$ of the $k$ hidden layers produces at most $A_k = 3^{k-1} \cdot\left(\frac{3w_{1}}{2}+1\right) \cdot w_{2} \cdot \ldots \cdot w_{k}$ pieces.
We add one more layer of $w_{k+1}$ neurons to this network such that its pre-activation is actually the output of a $\mathbb{R} \rightarrow \mathbb{R}$ ReLU DNN with depth $k+1$  and widths  $w_{1}, \ldots, w_{k}$. Suppose that the preactivation of the $i$-th neuron is $f_i, i\in[w_{k+1}]$. By the induction hypothesis, $f_i$ is a piecewise linear function  with at most $A_k$ pieces. Without loss of generality, let the $i$-th and $(i+1)$-th neurons be connected by an intra-layer link. Their outputs are $\sigma(f_i)$ and $\sigma(f_{i+1}+\sigma(f_i))$ whose total number of pieces equals to that of three functions $\sigma(f_i), \sigma(f_{i+1}), \sigma(f_i+f_{i+1})$, which is at most $6A_k$. Because we have $w_{k+1}$ neurons in the last hidden layer, we can get at most $6A_k\cdot (w_{k+1}/2)=3^{k} \cdot\left(\frac{3w_{1}}{2}+1\right) \cdot w_{2} \cdot \ldots w_k \cdot w_{k+1}$ pieces, which concludes the induction step.

\end{proof}

\begin{lemma}[Our 2nd New Result]
Let  $f: \mathbb{R} \rightarrow \mathbb{R}$  be a piecewise linear function with  $p$  pieces. If $f$ is represented by a ReLU DNN, as shown in Figure \ref{Figure_reduce_the_bound}(b), with depth $k+1$, then it must have size at least  $\frac{1}{3} k (2p)^{1/k}-\frac{2}{3}$. Conversely, any piecewise linear function $f$  represented by a ReLU DNN of depth  $k+1$  and size at most $s$, can have at most $\frac{1}{2}\left(\frac{3s}{k}\right)^{k}$ pieces.
\label{sp_ineqaulity}
\end{lemma}

\begin{proof}
Let widths of the  $k$  hidden layers be  $w_{1}, \ldots, w_{k}$. By \textbf{Lemma \ref{lemma:piece_size_relation_ours}}, we must have

\begin{equation}
   3^{k-1} \cdot\left(\frac{3w_{1}}{2}+1\right) \cdot w_{2} \cdot \ldots \cdot w_{k} = 2^{k-1} \cdot\left(\frac{3w_{1}}{2}+1\right) \cdot \frac{3w_{2}}{2} \cdot \ldots \cdot \frac{3w_{k}}{2} \geq p.
\label{piece_size_relationship}   
\end{equation}

By the AM-GM inequality, we minimize the size subject to \eqref{piece_size_relationship}, 
\begin{equation}
\begin{aligned}
     s & = w_{1}+w_{2}+\ldots+w_{k} \\
     & =\frac{2}{3} (\frac{3w_{1}}{2}+ \frac{3w_{2}}{2} + \ldots + \frac{3w_{k}}{2}) \\
     & = \frac{2}{3} (\frac{3w_{1}}{2}+1+ \frac{3w_{2}}{2} + \ldots + \frac{3w_{k}}{2}-1) \\
     & \geq \frac{2}{3} \Big(k(\frac{p}{2^{k-1}})^{1/k}-1\Big) \\
     & = \frac{1}{3} k(2p)^{1/k}-\frac{2}{3}, 
\end{aligned}
\end{equation}
where the inequality is achieved at $\frac{3w_{1}}{2}+1=\frac{3w_{2}}{2}=\ldots=\frac{3w_{k}}{2}\geq \frac{1}{2} p^{1/k}$. This leads to the first statement. The second statement follows by reversing the above equation. 
\begin{equation}
    \frac{1}{2}\left(\frac{3s}{k}\right)^{k} \geq p.
\end{equation}
\end{proof}

\begin{lemma}[The 1st New Result Extended to $n$ Intra-Linked Neurons]
Let $f: \mathbb{R}\rightarrow\mathbb{R}$ be a function represented by a $\mathbb{R}\rightarrow\mathbb{R}$ ReLU DNN, as shown in Figure \ref{Figure_reduce_the_bound}(c), with depth $k+1$, widths $w_{1}, \ldots, w_{k} $ of $k$ hidden layers and $n$ intra-linked neurons in each layer. Then $f$ is a PWL function with at most $2^{k-1} \cdot\left(\frac{(2^n-1)w_{1}}{n}+1\right) \cdot \frac{(2^n-1)w_{2}}{n} \cdot \ldots \cdot \frac{(2^n-1)w_{k}}{n} $ pieces.
\label{lemma:piece_size_relation_extended}
\end{lemma} 

\begin{proof}
For the first layer, the $n$ intra-linked neuron can create $2^n-1$ breaking points. Let us derive this fact via induction. Suppose that $G_{n-1}$ is the piecewise linear function of $n-1$ intra-linked neurons, then $G_n = \sigma(wx+b+G_{n-1})$. When $G_{n-1}\geq 0$, $G_n=\sigma(wx+b+G_{n-1})$, generating at most $2^{n-1}-1$ breaking points (same as $G_{n-1}$). When $G_{n-1}\leq 0$, $G_n=\sigma(wx+b)$ has one breaking point. Without repetition, the number of breaking points of $G_n$ is the summation of these newly-generated breaking points and breaking points of $G_{n-1}$, which is $2(2^{n-1}-1)+1=2^n-1$. Because the first layer has $w_1$ neurons, at most $w_1/n$ groups of $n$ intra-linked neurons are available. Thus, we can get at most $\frac{(2^n-1)w_{1}}{n}+1$ pieces.

Computing the number of pieces for the hidden layers shares the same spirit with the first layer, combining the proof of Lemma \ref{lemma:piece_size_relation_ours}, we conclude that a $\mathbb{R}\rightarrow\mathbb{R}$ ReLU DNN, as shown in Figure \ref{Figure_reduce_the_bound}(c), with depth $k+1$, widths $w_{1}, \ldots, w_{k} $ of $k$ hidden layers, and $n$ intra-linked neurons in each layer, has at most $2^{k-1} \cdot\left(\frac{(2^n-1)w_{1}}{n}+1\right) \cdot \frac{(2^n-1)w_{2}}{n} \cdot \ldots \cdot \frac{(2^n-1)w_{k}}{n} $ pieces.
\end{proof}

\begin{lemma}[The 2nd New Result Extended to $n$ Intra-Linked Neurons]
Let  $f: \mathbb{R} \rightarrow \mathbb{R}$  be a piecewise linear function with  $p$  pieces. If $f$ is represented by a ReLU DNN, as shown in Figure \ref{Figure_reduce_the_bound}(c), with depth $k+1$ and $n$ neurons intra-linked in each hidden layer, then it must have size at least  $\frac{n}{2(2^n-1)} k (2p)^{1/k}-\frac{n}{2^n-1}$. Conversely, any piecewise linear function $f$  represented by a ReLU DNN of depth  $k+1$  and size at most $s$, can have at most $\frac{1}{2}\left(\frac{2(2^n-1)s}{nk}\right)^{k}$ pieces.
\end{lemma}

\begin{proof}
Please refer to our proof for Lemma \ref{sp_ineqaulity}.
\end{proof}

\clearpage


\renewcommand\refname{Reference}

\vskip 0.2in
\bibliography{sample}

\begin{thebibliography}{82}
\providecommand{\natexlab}[1]{#1}
\providecommand{\url}[1]{\texttt{#1}}
\expandafter\ifx\csname urlstyle\endcsname\relax
  \providecommand{\doi}[1]{doi: #1}\else
  \providecommand{\doi}{doi: \begingroup \urlstyle{rm}\Url}\fi

\bibitem[Adadi and Berrada(2018)]{b31}
A.~Adadi and M.~Berrada.
\newblock Peeking inside the black-box: A survey on explainable artificial
  intelligence (xai).
\newblock \emph{IEEE Access}, 6:\penalty0 52138--60, 2018.

\bibitem[Allen-Zhu et~al.(2019)Allen-Zhu, Li, and Liang]{allen2019learning}
Zeyuan Allen-Zhu, Yuanzhi Li, and Yingyu Liang.
\newblock Learning and generalization in overparameterized neural networks,
  going beyond two layers.
\newblock In \emph{Advances in neural information processing systems}, pages
  6155--6166, 2019.

\bibitem[Arora et~al.(2016)Arora, Basu, Mianjy, and
  Mukherjee]{arora2016understanding}
Raman Arora, Amitabh Basu, Poorya Mianjy, and Anirbit Mukherjee.
\newblock Understanding deep neural networks with rectified linear units.
\newblock \emph{arXiv preprint arXiv:1611.01491}, 2016.

\bibitem[Arora et~al.(2019)Arora, Du, Hu, Li, Salakhutdinov, and
  Wang]{arora2019exact}
Sanjeev Arora, Simon~S Du, Wei Hu, Zhiyuan Li, Russ~R Salakhutdinov, and
  Ruosong Wang.
\newblock On exact computation with an infinitely wide neural net.
\newblock In \emph{Advances in Neural Information Processing Systems}, pages
  8139--8148, 2019.

\bibitem[Bartlett et~al.(2017)Bartlett, Foster, and
  Telgarsky]{bartlett2017spectrally}
Peter~L Bartlett, Dylan~J Foster, and Matus~J Telgarsky.
\newblock Spectrally-normalized margin bounds for neural networks.
\newblock In \emph{Advances in Neural Information Processing Systems}, pages
  6240--6249, 2017.

\bibitem[Bianchini and Scarselli(2014)]{b14}
Monica Bianchini and Franco Scarselli.
\newblock On the complexity of neural network classifiers: A comparison between
  shallow and deep architectures.
\newblock \emph{IEEE transactions on neural networks and learning systems},
  25\penalty0 (8):\penalty0 1553--1565, 2014.

\bibitem[Brutzkus et~al.(2017)Brutzkus, Globerson, Malach, and
  Shalev-Shwartz]{brutzkus2017sgd}
Alon Brutzkus, Amir Globerson, Eran Malach, and Shai Shalev-Shwartz.
\newblock Sgd learns over-parameterized networks that provably generalize on
  linearly separable data, 2017.

\bibitem[Bu and Karpatne(2021)]{bu2021quadratic}
Jie Bu and Anuj Karpatne.
\newblock Quadratic residual networks: A new class of neural networks for
  solving forward and inverse problems in physics involving pdes.
\newblock In \emph{Proceedings of the 2021 SIAM International Conference on
  Data Mining (SDM)}, pages 675--683. SIAM, 2021.

\bibitem[Cao and Gu(2019)]{cao2019generalization}
Yuan Cao and Quanquan Gu.
\newblock A generalization theory of gradient descent for learning
  over-parameterized deep relu networks, 2019.

\bibitem[Chen and Liu(2017)]{b18}
CL~Philip Chen and Zhulin Liu.
\newblock Broad learning system: An effective and efficient incremental
  learning system without the need for deep architecture.
\newblock \emph{IEEE transactions on neural networks and learning systems},
  29\penalty0 (1):\penalty0 10--24, 2017.

\bibitem[Chen et~al.(2017)Chen, Zhang, Kalra, Lin, Chen, Liao, Zhou, and
  Wang]{b5}
Hu~Chen, Yi~Zhang, Mannudeep~K Kalra, Feng Lin, Yang Chen, Peixi Liao, Jiliu
  Zhou, and Ge~Wang.
\newblock Low-dose ct with a residual encoder-decoder convolutional neural
  network.
\newblock \emph{IEEE transactions on medical imaging}, 36\penalty0
  (12):\penalty0 2524--2535, 2017.

\bibitem[Chen et~al.(2016)Chen, Duan, Houthooft, Schulman, Sutskever, and
  Abbeel]{chen2016infogan}
Xi~Chen, Yan Duan, Rein Houthooft, John Schulman, Ilya Sutskever, and Pieter
  Abbeel.
\newblock Infogan: Interpretable representation learning by information
  maximizing generative adversarial nets.
\newblock In \emph{Advances in neural information processing systems}, pages
  2172--2180, 2016.

\bibitem[Cheng et~al.(2016)Cheng, Koc, Harmsen, Shaked, Chandra, Aradhye,
  Anderson, Corrado, Chai, Ispir, et~al.]{cheng2016wide}
Heng-Tze Cheng, Levent Koc, Jeremiah Harmsen, Tal Shaked, Tushar Chandra,
  Hrishi Aradhye, Glen Anderson, Greg Corrado, Wei Chai, Mustafa Ispir, et~al.
\newblock Wide \& deep learning for recommender systems.
\newblock In \emph{Proceedings of the 1st workshop on deep learning for
  recommender systems}, pages 7--10, 2016.

\bibitem[Chu et~al.(2018)Chu, Hu, Hu, Wang, and Pei]{chu2018exact}
Lingyang Chu, Xia Hu, Juhua Hu, Lanjun Wang, and Jian Pei.
\newblock Exact and consistent interpretation for piecewise linear neural
  networks: A closed form solution.
\newblock In \emph{Proceedings of the 24th ACM SIGKDD International Conference
  on Knowledge Discovery \& Data Mining}, pages 1244--1253, 2018.

\bibitem[Cohen et~al.(2016)Cohen, Sharir, and Shashua]{cohen2016expressive}
Nadav Cohen, Or~Sharir, and Amnon Shashua.
\newblock On the expressive power of deep learning: A tensor analysis.
\newblock In \emph{Conference on learning theory}, pages 698--728. PMLR, 2016.

\bibitem[Dahl et~al.(2011)Dahl, Yu, Deng, and Acero]{b3}
George~E Dahl, Dong Yu, Li~Deng, and Alex Acero.
\newblock Context-dependent pre-trained deep neural networks for
  large-vocabulary speech recognition.
\newblock \emph{IEEE Transactions on audio, speech, and language processing},
  20\penalty0 (1):\penalty0 30--42, 2011.

\bibitem[De~Branges(1959)]{de1959stone}
Louis De~Branges.
\newblock The stone-weierstrass theorem.
\newblock \emph{Proceedings of the American Mathematical Society}, 10\penalty0
  (5):\penalty0 822--824, 1959.

\bibitem[Deng et~al.(2020)Deng, Dong, Zhang, and Zhu]{deng2020understanding}
Zhijie Deng, Yinpeng Dong, Shifeng Zhang, and Jun Zhu.
\newblock Understanding and exploring the network with stochastic
  architectures.
\newblock \emph{Advances in Neural Information Processing Systems}, 33, 2020.

\bibitem[Du and Hu(2019)]{du2019width}
Simon Du and Wei Hu.
\newblock Width provably matters in optimization for deep linear neural
  networks.
\newblock In \emph{International Conference on Machine Learning}, pages
  1655--1664. PMLR, 2019.

\bibitem[Du et~al.(2018)Du, Zhai, Poczos, and Singh]{du2018gradient}
Simon~S Du, Xiyu Zhai, Barnabas Poczos, and Aarti Singh.
\newblock Gradient descent provably optimizes over-parameterized neural
  networks.
\newblock In \emph{International Conference on Learning Representations}, 2018.

\bibitem[Ehrenborg(2007)]{ehrenborg2007perles}
Richard Ehrenborg.
\newblock The perles-shephard identity for non-convex polytopes.
\newblock Technical report, Technical Report, University of Kentucky, 2007.

\bibitem[Eldan and Shamir(2016)]{eldan2016power}
Ronen Eldan and Ohad Shamir.
\newblock The power of depth for feedforward neural networks.
\newblock In \emph{Conference on learning theory}, pages 907--940. PMLR, 2016.

\bibitem[Fan and Wang(2020)]{b30}
F.~Fan and G.~Wang.
\newblock Fuzzy logic interpretation of quadratic networks.
\newblock \emph{Neurocomputing}, 374:\penalty0 10--21, 2020.

\bibitem[Fan et~al.(2018{\natexlab{a}})Fan, Cong, and Wang]{b2}
Fenglei Fan, Wenxiang Cong, and Ge~Wang.
\newblock A new type of neurons for machine learning.
\newblock \emph{International journal for numerical methods in biomedical
  engineering}, 34\penalty0 (2):\penalty0 e2920, 2018{\natexlab{a}}.

\bibitem[Fan et~al.(2018{\natexlab{b}})Fan, Cong, and Wang]{b25}
Fenglei Fan, Wenxiang Cong, and Ge~Wang.
\newblock Generalized backpropagation algorithm for training second-order
  neural networks.
\newblock \emph{International journal for numerical methods in biomedical
  engineering}, 34\penalty0 (5):\penalty0 e2956, 2018{\natexlab{b}}.

\bibitem[Fan et~al.(2018{\natexlab{c}})Fan, Wang, Guo, Zhu, Yan, Wang, and
  Yu]{fan2018sparse}
Fenglei Fan, Dayang Wang, Hengtao Guo, Qikui Zhu, Pingkun Yan, Ge~Wang, and
  Hengyong Yu.
\newblock On a sparse shortcut topology of artificial neural networks,
  2018{\natexlab{c}}.

\bibitem[Fan et~al.(2019)Fan, Shan, Kalra, Singh, Qian, Getzin, Teng, Hahn, and
  Wang]{fan2019quadratic}
Fenglei Fan, Hongming Shan, Mannudeep~K Kalra, Ramandeep Singh, Guhan Qian,
  Matthew Getzin, Yueyang Teng, Juergen Hahn, and Ge~Wang.
\newblock Quadratic autoencoder (q-ae) for low-dose ct denoising.
\newblock \emph{IEEE transactions on medical imaging}, 39\penalty0
  (6):\penalty0 2035--2050, 2019.

\bibitem[Fan et~al.(2020)Fan, Xiong, and Wang]{fan2020universal}
Fenglei Fan, Jinjun Xiong, and Ge~Wang.
\newblock Universal approximation with quadratic deep networks.
\newblock \emph{Neural Networks}, 124:\penalty0 383--392, 2020.

\bibitem[Funahashi(1989)]{b15}
Ken-Ichi Funahashi.
\newblock On the approximate realization of continuous mappings by neural
  networks.
\newblock \emph{Neural networks}, 2\penalty0 (3):\penalty0 183--192, 1989.

\bibitem[Ge et~al.(2015)Ge, Huang, Jin, and Yuan]{ge2015escaping}
Rong Ge, Furong Huang, Chi Jin, and Yang Yuan.
\newblock Escaping from saddle points—online stochastic gradient for tensor
  decomposition.
\newblock In \emph{Conference on Learning Theory}, pages 797--842, 2015.

\bibitem[Girosi and Poggio(1989)]{girosi1989representation}
Federico Girosi and Tomaso Poggio.
\newblock Representation properties of networks: Kolmogorov's theorem is
  irrelevant.
\newblock \emph{Neural Computation}, 1\penalty0 (4):\penalty0 465--469, 1989.

\bibitem[Golowich et~al.(2017)Golowich, Rakhlin, and Shamir]{golowich2017size}
Noah Golowich, Alexander Rakhlin, and Ohad Shamir.
\newblock Size-independent sample complexity of neural networks, 2017.

\bibitem[Goodfellow et~al.(2014{\natexlab{a}})Goodfellow, Pouget-Abadie, Mirza,
  Xu, Warde-Farley, Ozair, Courville, and Bengio]{goodfellow2014generative}
Ian Goodfellow, Jean Pouget-Abadie, Mehdi Mirza, Bing Xu, David Warde-Farley,
  Sherjil Ozair, Aaron Courville, and Yoshua Bengio.
\newblock Generative adversarial nets.
\newblock In \emph{Advances in neural information processing systems}, pages
  2672--2680, 2014{\natexlab{a}}.

\bibitem[Goodfellow et~al.(2016)Goodfellow, Bengio, Courville, and Bengio]{b1}
Ian Goodfellow, Yoshua Bengio, Aaron Courville, and Yoshua Bengio.
\newblock \emph{Deep learning}, volume~1.
\newblock MIT press Cambridge, 2016.

\bibitem[Goodfellow et~al.(2014{\natexlab{b}})Goodfellow, Shlens, and
  Szegedy]{goodfellow2014explaining}
Ian~J Goodfellow, Jonathon Shlens, and Christian Szegedy.
\newblock Explaining and harnessing adversarial examples, 2014{\natexlab{b}}.

\bibitem[He et~al.(2018)He, Li, Xu, and Zheng]{he2018relu}
Juncai He, Lin Li, Jinchao Xu, and Chunyue Zheng.
\newblock Relu deep neural networks and linear finite elements, 2018.

\bibitem[Hornik et~al.(1989)Hornik, Stinchcombe, and White]{b16}
Kurt Hornik, Maxwell Stinchcombe, and Halbert White.
\newblock Multilayer feedforward networks are universal approximators.
\newblock \emph{Neural networks}, 2\penalty0 (5):\penalty0 359--366, 1989.

\bibitem[Huang et~al.(2020)Huang, Wang, Tao, and Zhao]{huang2020deep}
Kaixuan Huang, Yuqing Wang, Molei Tao, and Tuo Zhao.
\newblock Why do deep residual networks generalize better than deep feedforward
  networks?--a neural tangent kernel perspective, 2020.

\bibitem[Jacot et~al.(2018)Jacot, Gabriel, and Hongler]{jacot2018neural}
Arthur Jacot, Franck Gabriel, and Cl{\'e}ment Hongler.
\newblock Neural tangent kernel: Convergence and generalization in neural
  networks.
\newblock In \emph{Advances in neural information processing systems}, pages
  8571--8580, 2018.

\bibitem[Ji et~al.(2021)Ji, Liu, Wen, Zhai, Wang, and
  Katsanos]{ji2021prediction}
Duofa Ji, Jin Liu, Weiping Wen, Changhai Zhai, Wei Wang, and Evangelos~I
  Katsanos.
\newblock Prediction of cumulative absolute velocity based on refined
  second-order deep neural network.
\newblock \emph{Journal of Earthquake Engineering}, pages 1--20, 2021.

\bibitem[Kawaguchi et~al.(2019)Kawaguchi, Huang, and Kaelbling]{b22}
Kenji Kawaguchi, Jiaoyang Huang, and Leslie~Pack Kaelbling.
\newblock Effect of depth and width on local minima in deep learning.
\newblock \emph{Neural computation}, 31\penalty0 (7):\penalty0 1462--1498,
  2019.

\bibitem[Kingma and Ba(2015)]{kingma2015adam}
Diederik~P Kingma and Jimmy Ba.
\newblock Adam: A method for stochastic optimization.
\newblock In \emph{ICLR (Poster)}, 2015.

\bibitem[Kolmogorov(1956)]{b29}
A.~N. Kolmogorov.
\newblock On the representation of continuous functions of several variables by
  superpositions of continuous functions of a smaller number of variables.
\newblock \emph{Proceedings of the USSR Academy of Sciences}, 108:\penalty0
  179--182, 1956.

\bibitem[Kumar et~al.(2016)Kumar, Irsoy, Ondruska, Iyyer, Bradbury, Gulrajani,
  Zhong, Paulus, and Socher]{kumar2016ask}
Ankit Kumar, Ozan Irsoy, Peter Ondruska, Mohit Iyyer, James Bradbury, Ishaan
  Gulrajani, Victor Zhong, Romain Paulus, and Richard Socher.
\newblock Ask me anything: Dynamic memory networks for natural language
  processing.
\newblock In \emph{International conference on machine learning}, pages
  1378--1387. PMLR, 2016.

\bibitem[Kurakin et~al.(2016)Kurakin, Goodfellow, Bengio,
  et~al.]{kurakin2016adversarial}
Alexey Kurakin, Ian Goodfellow, Samy Bengio, et~al.
\newblock Adversarial examples in the physical world, 2016.

\bibitem[Lee et~al.(2018)Lee, Bahri, Novak, Schoenholz, Pennington, and
  Sohl-Dickstein]{lee2018deep}
Jaehoon Lee, Yasaman Bahri, Roman Novak, Samuel~S Schoenholz, Jeffrey
  Pennington, and Jascha Sohl-Dickstein.
\newblock Deep neural networks as gaussian processes.
\newblock In \emph{International Conference on Learning Representations}, 2018.

\bibitem[Levine et~al.(2020)Levine, Wies, Sharir, Bata, and
  Shashua]{levine2020limits}
Yoav Levine, Noam Wies, Or~Sharir, Hofit Bata, and Amnon Shashua.
\newblock Limits to depth efficiencies of self-attention, 2020.

\bibitem[Li et~al.(2016)Li, Kadav, Durdanovic, Samet, and Graf]{li2016pruning}
Hao Li, Asim Kadav, Igor Durdanovic, Hanan Samet, and Hans~Peter Graf.
\newblock Pruning filters for efficient convnets, 2016.

\bibitem[Li et~al.(2018)Li, Lu, Wang, Haupt, and Zhao]{li2018tighter}
Xingguo Li, Junwei Lu, Zhaoran Wang, Jarvis Haupt, and Tuo Zhao.
\newblock On tighter generalization bound for deep neural networks: Cnns,
  resnets, and beyond, 2018.

\bibitem[Li and Liang(2018)]{li2018learning}
Yuanzhi Li and Yingyu Liang.
\newblock Learning overparameterized neural networks via stochastic gradient
  descent on structured data.
\newblock In \emph{Advances in Neural Information Processing Systems}, pages
  8157--8166, 2018.

\bibitem[Light and Cheney(2006)]{light2006approximation}
William~A Light and Elliot~W Cheney.
\newblock \emph{Approximation theory in tensor product spaces}, volume 1169.
\newblock Springer, 2006.

\bibitem[Lin and Jegelka(2018)]{lin2018resnet}
Hongzhou Lin and Stefanie Jegelka.
\newblock Resnet with one-neuron hidden layers is a universal approximator.
\newblock \emph{Advances in Neural Information Processing Systems},
  31:\penalty0 6169--6178, 2018.

\bibitem[Lipton(2018)]{lipton2018mythos}
Zachary~C Lipton.
\newblock The mythos of model interpretability.
\newblock \emph{Queue}, 16\penalty0 (3):\penalty0 31--57, 2018.

\bibitem[Lu et~al.(2017)Lu, Pu, Wang, Hu, and Wang]{lu2017expressive}
Zhou Lu, Hongming Pu, Feicheng Wang, Zhiqiang Hu, and Liwei Wang.
\newblock The expressive power of neural networks: a view from the width.
\newblock In \emph{Proceedings of the 31st International Conference on Neural
  Information Processing Systems}, pages 6232--6240, 2017.

\bibitem[Mantini and Shah(2021)]{mantini2021cqnn}
Pranav Mantini and Shishr~K Shah.
\newblock Cqnn: Convolutional quadratic neural networks.
\newblock In \emph{2020 25th International Conference on Pattern Recognition
  (ICPR)}, pages 9819--9826. IEEE, 2021.

\bibitem[Matthews et~al.(2018)Matthews, Hron, Rowland, Turner, and
  Ghahramani]{matthews2018gaussian}
Alexander G de~G Matthews, Jiri Hron, Mark Rowland, Richard~E Turner, and
  Zoubin Ghahramani.
\newblock Gaussian process behaviour in wide deep neural networks.
\newblock In \emph{International Conference on Learning Representations}, 2018.

\bibitem[Mhaskar and Poggio(2016)]{b10}
Hrushikesh~N Mhaskar and Tomaso Poggio.
\newblock Deep vs. shallow networks: An approximation theory perspective.
\newblock \emph{Analysis and Applications}, 14\penalty0 (06):\penalty0
  829--848, 2016.

\bibitem[Montufar et~al.(2014)Montufar, Pascanu, Cho, and
  Bengio]{montufar2014number}
Guido~F Montufar, Razvan Pascanu, Kyunghyun Cho, and Yoshua Bengio.
\newblock On the number of linear regions of deep neural networks.
\newblock In \emph{Advances in neural information processing systems}, pages
  2924--2932, 2014.

\bibitem[Moosavi-Dezfooli et~al.(2016)Moosavi-Dezfooli, Fawzi, and
  Frossard]{moosavi2016deepfool}
Seyed-Mohsen Moosavi-Dezfooli, Alhussein Fawzi, and Pascal Frossard.
\newblock Deepfool: a simple and accurate method to fool deep neural networks.
\newblock In \emph{Proceedings of the IEEE conference on computer vision and
  pattern recognition}, pages 2574--2582, 2016.

\bibitem[Neal(1996)]{neal1996priors}
Radford~M Neal.
\newblock Priors for infinite networks.
\newblock In \emph{Bayesian Learning for Neural Networks}, pages 29--53.
  Springer, 1996.

\bibitem[Neyshabur et~al.(2015)Neyshabur, Tomioka, and
  Srebro]{neyshabur2015norm}
Behnam Neyshabur, Ryota Tomioka, and Nathan Srebro.
\newblock Norm-based capacity control in neural networks.
\newblock In \emph{Conference on Learning Theory}, pages 1376--1401, 2015.

\bibitem[Neyshabur et~al.(2018)Neyshabur, Bhojanapalli, and
  Srebro]{neyshabur2018pac}
Behnam Neyshabur, Srinadh Bhojanapalli, and Nathan Srebro.
\newblock A pac-bayesian approach to spectrally-normalized margin bounds for
  neural networks.
\newblock In \emph{International Conference on Learning Representations}, 2018.

\bibitem[Novak et~al.(2018)Novak, Xiao, Bahri, Lee, Yang, Hron, Abolafia,
  Pennington, and Sohl-dickstein]{novak2018bayesian}
Roman Novak, Lechao Xiao, Yasaman Bahri, Jaehoon Lee, Greg Yang, Jiri Hron,
  Daniel~A Abolafia, Jeffrey Pennington, and Jascha Sohl-dickstein.
\newblock Bayesian deep convolutional networks with many channels are gaussian
  processes.
\newblock In \emph{International Conference on Learning Representations}, 2018.

\bibitem[Park et~al.(2021)Park, Lee, Kim, and Blei]{park2021unsupervised}
Yookoon Park, Sangho Lee, Gunhee Kim, and David Blei.
\newblock Unsupervised representation learning via neural activation coding.
\newblock In \emph{International Conference on Machine Learning}, pages
  8391--8400. PMLR, 2021.

\bibitem[Polino et~al.(2018)Polino, Pascanu, and Alistarh]{polino2018model}
Antonio Polino, Razvan Pascanu, and Dan Alistarh.
\newblock Model compression via distillation and quantization.
\newblock In \emph{International Conference on Learning Representations}, 2018.

\bibitem[Saad and II(2007)]{b34}
E.~W. Saad and D.~C.~Wunsch II.
\newblock Neural network explanation using inversion.
\newblock \emph{Neural networks}, 20\penalty0 (1):\penalty0 78--93, 2007.

\bibitem[Serra and Ramalingam(2020)]{serra2020empirical}
Thiago Serra and Srikumar Ramalingam.
\newblock Empirical bounds on linear regions of deep rectifier networks.
\newblock In \emph{AAAI}, pages 5628--5635, 2020.

\bibitem[Serra et~al.(2018)Serra, Tjandraatmadja, and
  Ramalingam]{serra2018bounding}
Thiago Serra, Christian Tjandraatmadja, and Srikumar Ramalingam.
\newblock Bounding and counting linear regions of deep neural networks.
\newblock In \emph{International Conference on Machine Learning}, pages
  4558--4566. PMLR, 2018.

\bibitem[Setiono and Liu(1995)]{setiono1995understanding}
Rudy Setiono and Huan Liu.
\newblock Understanding neural networks via rule extraction.
\newblock In \emph{IJCAI}, volume~1, pages 480--485. Citeseer, 1995.

\bibitem[Szymanski and McCane(2014)]{b7}
Lech Szymanski and Brendan McCane.
\newblock Deep networks are effective encoders of periodicity.
\newblock \emph{IEEE transactions on neural networks and learning systems},
  25\penalty0 (10):\penalty0 1816--1827, 2014.

\bibitem[Thrun(1995)]{thrun1995extracting}
Sebastian Thrun.
\newblock Extracting rules from artificial neural networks with distributed
  representations.
\newblock pages 505--512. MORGAN KAUFMANN PUBLISHERS, 1995.

\bibitem[Van~der Maaten and Hinton(2008)]{van2008visualizing}
Laurens Van~der Maaten and Geoffrey Hinton.
\newblock Visualizing data using t-sne.
\newblock \emph{Journal of machine learning research}, 9\penalty0 (11), 2008.

\bibitem[Wang(2016)]{b6}
Ge~Wang.
\newblock A perspective on deep imaging.
\newblock \emph{Ieee Access}, 4:\penalty0 8914--8924, 2016.

\bibitem[Wang and Sun(2005)]{wang2005generalization}
Shuning Wang and Xusheng Sun.
\newblock Generalization of hinging hyperplanes.
\newblock \emph{IEEE Transactions on Information Theory}, 51\penalty0
  (12):\penalty0 4425--4431, 2005.

\bibitem[Widder(1989)]{widder1989advanced}
David~Vernon Widder.
\newblock \emph{Advanced calculus}.
\newblock Courier Corporation, 1989.

\bibitem[Wu et~al.(2016)Wu, Leng, Wang, Hu, and Cheng]{wu2016quantized}
Jiaxiang Wu, Cong Leng, Yuhang Wang, Qinghao Hu, and Jian Cheng.
\newblock Quantized convolutional neural networks for mobile devices.
\newblock In \emph{Proceedings of the IEEE Conference on Computer Vision and
  Pattern Recognition}, pages 4820--4828, 2016.

\bibitem[Xie et~al.(2019)Xie, Kirillov, Girshick, and He]{xie2019exploring}
Saining Xie, Alexander Kirillov, Ross Girshick, and Kaiming He.
\newblock Exploring randomly wired neural networks for image recognition.
\newblock In \emph{Proceedings of the IEEE International Conference on Computer
  Vision}, pages 1284--1293, 2019.

\bibitem[Xiong et~al.(2020)Xiong, Huang, Yu, Liu, Zhu, and
  Shao]{xiong2020number}
Huan Xiong, Lei Huang, Mengyang Yu, Li~Liu, Fan Zhu, and Ling Shao.
\newblock On the number of linear regions of convolutional neural networks.
\newblock In \emph{International Conference on Machine Learning}, pages
  10514--10523. PMLR, 2020.

\bibitem[Xu et~al.(2022)Xu, Yu, Xiong, and Chen]{xu2022quadralib}
Zirui Xu, Fuxun Yu, Jinjun Xiong, and Xiang Chen.
\newblock Quadralib: A performant quadratic neural network library for
  architecture optimization and design exploration.
\newblock \emph{arXiv preprint arXiv:2204.01701}, 2022.

\bibitem[Zagoruyko and Komodakis(2016)]{b19}
Sergey Zagoruyko and Nikos Komodakis.
\newblock Wide residual networks.
\newblock In \emph{BMVC}, 2016.

\bibitem[Zhang et~al.(2016)Zhang, Bengio, Hardt, Recht, and
  Vinyals]{zhang2016understanding}
Chiyuan Zhang, Samy Bengio, Moritz Hardt, Benjamin Recht, and Oriol Vinyals.
\newblock Understanding deep learning requires rethinking generalization, 2016.

\bibitem[Zhang et~al.(2015)Zhang, Zou, Ming, He, and Sun]{zhang2015efficient}
Xiangyu Zhang, Jianhua Zou, Xiang Ming, Kaiming He, and Jian Sun.
\newblock Efficient and accurate approximations of nonlinear convolutional
  networks.
\newblock In \emph{Proceedings of the IEEE Conference on Computer Vision and
  pattern Recognition}, pages 1984--1992, 2015.

\end{thebibliography}

\end{document}